\documentclass[english]{article}
\usepackage{style}

\title{\textbf{Regret-Optimal Model-Free Reinforcement Learning\\ for Discounted MDPs with Short Burn-In Time}}
\author{Xiang Ji\footnote{Department of Electrical and Computer Engineering, School of Engineering and Applied Science, Princeton University, Princeton, NJ 08544, USA.} \hspace{2.05in} Gen Li\footnote{Department of Statistics, The Chinese University of Hong Kong, Hong Kong, China.}\\
Princeton \hspace{2.0in} CUHK
}

\date{}

\begin{document}

\maketitle

\allowdisplaybreaks

\begin{abstract}
A crucial problem in reinforcement learning is learning the optimal policy. We study this in tabular infinite-horizon discounted Markov decision processes under the online setting. The existing algorithms either fail to achieve regret optimality or have to incur a high memory and computational cost. In addition, existing optimal algorithms all require a long burn-in time in order to achieve optimal sample efficiency, i.e., their optimality is not guaranteed unless sample size surpasses a high threshold. We address both open problems by introducing a model-free algorithm that employs variance reduction and a novel technique that switches the execution policy in a slow-yet-adaptive manner. This is the first regret-optimal model-free algorithm in the discounted setting, with the additional benefit of a low burn-in time.
\end{abstract}


\section{Introduction}\label{sec:intro}

In reinforcement learning (RL), a crucial task is to find the optimal policy that maximizes its expected cumulative reward in any given environment with unknown dynamics. An immense body of literature is dedicated to finding algorithms that solve this task with as few samples as possible, which is the prime goal under this task. Ideally, one hopes to find an algorithm with a theoretical guarantee of optimal sample efficiency. At the same time, this task might be accompanied with additional requirements such as low space complexity and computational cost, as it is common that the state and action spaces exhibit high dimensions in modern applications. The combination of these various goals and requirements presents an important yet challenging problem in algorithm design.

The task of searching for optimal policy has been well-studied by existing work in the generative setting \citep{sidford2018near,sidford2018variance,li2020breaking,agarwal2020model}. This fundamental setting allows the freedom of querying samples at any state-action pair. In contrast, it is more realistic but difficult to consider the same task in the online setting, in which samples can only be collected along trajectories generated from executing a policy in the unknown Markov decision process (MDP). Solving this task with optimal sample efficiency requires a careful balance between exploration and exploitation, especially when coupled with other goals such as memory and computational efficiency. 

MDPs can be divided into two types: the episodic finite-horizon MDPs and the infinite-horizon MDPs. Although these two types of MDPs can be approached in similar ways under the generative setting, there is a clear dichotomy between them in the online setting. In an episodic MDP, sample trajectories are only defined in fixed-length episodes, so samples are collected in episodes, and a reset to an arbitrary initial state occurs at the end of every online episode. Its transition kernel is usually assumed to be non-stationary over time. In contrast, the transition kernel of an infinite-horizon MDP stays stationary over time, and the online sample collection process amounts to drawing a single infinitely long sample trajectory with no reset. These differences render most optimal algorithms for episodic MDPs suboptimal when applied to infinite-horizon MDPs. Without reset and non-stationarity, the high dependency between consecutive trajectory steps in the infinite-horizon setting presents a new challenge over the episodic setting. In this work, we consider the infinite-horizon discounted MDPs, which is widely used in practice but still has some fundamental questions unanswered in theory.

\subsection{Sample Efficiency in Infinite-Horizon MDPs}

To evaluate the sample efficiency of online RL algorithms, a natural and widely-accepted metric is the \textit{cumulative regret}. It captures the performance difference between the optimal policy and the learned policy of an algorithm over its online interactions with a given MDP. The notion of cumulative regret was first introduced in the bandit literature and later adopted in the RL literature \citep{auer05,Jin-provably}. It is profusely used in the online episodic RL literature. Such works aim to prove regret guarantees for algorithms and provide analyses that characterize such regret guarantees in terms of all problem parameters such as state space, action space and sample size in a non-asymptotic fashion. A cumulative regret guarantee can also suggest the sample complexity needed to reach a certain level of average regret.

In the online infinite-horizon setting, many works study a different metric called the sample complexity of exploration, first introduced in \citet{kakade2003sample}. In essence, given a target accuracy level $\epsilon$, this metric characterizes the total number of $\epsilon$-suboptimal steps committed by an algorithm over an infinitely-long trajectory in the MDP. While this is indicative of the sample efficiency of an algorithm, the focus of this metric is very different from that of cumulative regret, as it only reflects the total number of failures but does not distinguish their sizes. As \citet{liu2020regret,zhou2021nearly} point out, even an optimal guarantee on the sample complexity of exploration can only be converted to a very suboptimal guarantee on the cumulative regret. To obtain a more quantitative characterization of the total volume of failures in the regime of finite samples, some works have turned to studying cumulative regret guarantees for algorithms. 

It was not until recently that some works \citep{liu2020regret,zhou2021provably,zhou2021nearly,kash2022slowly} begin to research into the problem of cumulative regret minimization in infinite-horizon discounted MDPs. Among them, \citet{zhou2021provably} focus on linear MDPs while others study tabular MDPs. In this work, we study the regret minimization problem in the tabular case. Hereafter and throughout, we denote the size of the state space, the size of the action space and the discount factor of the problem MDP with $S$, $A$ and $\gamma$, respectively, and let $T$ denote the sample size.

\subsection{Model-Based and Model-Free Methods}

Since modern RL applications are often large-scale, algorithms with low space complexity and computational complexity are much desired. This renders the distinction between model-based algorithms and model-free algorithms particularly important. The procedure of a model-based method includes a model estimation stage that involves estimating the transition kernel and a subsequent planning stage that searches the optimal policy in the learned model. Thus, $O(S^2A)$ space is required to store the estimated model. This is unfavorable when the state space is large and a memory constraint is present. Additionally, updating the transition kernel estimate brings a large computational burden. In comparison, model-free methods do not learn the entire model and thus can run with $o(S^2A)$ space. Notably, most value-based methods such as Q-learning only require storage of an estimated Q-function, which can take as little as $O(SA)$ memory. In the infinite-horizon discounted setting, although {\sf UCBVI-$\gamma$} in \citet{zhou2021nearly} can achieve optimal regret, its model-based nature exacts a $O(S^2A)$ memory and computational cost; conversely, the algorithms in \citet{liu2020regret} and \citet{kash2022slowly} are model-free but have suboptimal regret guarantee.

\begin{table*}[!ht]
\begin{center}
    \begin{tabular}{c|c|c|c|c}
    Algorithm & \shortstack{Sample complexity\\ of exploration} & \shortstack{Cumulative\\ Regret} & \shortstack{Range of $T$ \\ with optimality} & \shortstack{Space\\ complexity}\\
    \hline
    {\sf Delayed Q-learning} & \multirow{2}{*}{$\frac{SA}{(1-\gamma)^8 \epsilon^4}$} & \multirow{2}{*}{$\frac{S^{\frac{1}{5}}A^{\frac{1}{5}}T^{\frac{4}{5}}}{(1-\gamma)^{\frac{9}{5}}}$} & \multirow{2}{*}{never} & \multirow{2}{*}{$SA$}\\
    \citep{strehl06pac} & & & & \\
    \hline
    {\sf R-Max} & \multirow{2}{*}{$\frac{S^2A}{(1-\gamma)^6 \epsilon^3}$} & \multirow{2}{*}{$\frac{S^{\frac{1}{2}}A^{\frac{1}{4}}T^{\frac{3}{4}}}{(1-\gamma)^{\frac{7}{4}}}$} & \multirow{2}{*}{never} & \multirow{2}{*}{$S^2A$}\\
    \citep{Brafman03} & & & & \\
    \hline
    {\sf UCB-Q} & \multirow{2}{*}{$\frac{SA}{(1-\gamma)^7 \epsilon^2}$} & \multirow{2}{*}{$\frac{S^{\frac{1}{3}}A^{\frac{1}{3}}T^{\frac{2}{3}}}{(1-\gamma)^{\frac{8}{3}}}$} & \multirow{2}{*}{never} & \multirow{2}{*}{$SA$}\\
    \citep{dong2019q} & & & & \\
    \hline
    {\sf MoRmax} & \multirow{2}{*}{$\frac{SA}{(1-\gamma)^6 \epsilon^2}$} & \multirow{2}{*}{$\frac{S^{\frac{1}{3}}A^{\frac{1}{3}}T^{\frac{2}{3}}}{(1-\gamma)^{\frac{7}{3}}}$} & \multirow{2}{*}{never} & \multirow{2}{*}{$S^2A$}\\
    \citep{szita10} & & & & \\
    \hline
    {\sf UCRL} & \multirow{2}{*}{$\frac{S^2A}{(1-\gamma)^3 \epsilon^2}$} & \multirow{2}{*}{$\frac{S^{\frac{2}{3}}A^{\frac{1}{3}}T^{\frac{2}{3}}}{(1-\gamma)^{\frac{4}{3}}}$} & \multirow{2}{*}{never} & \multirow{2}{*}{$S^2A$}\\
    \citep{lattimore2012pac} & & & & \\
    \hline
    {\sf UCB-multistage} & \multirow{2}{*}{$\frac{SA}{(1-\gamma)^{\frac{11}{2}} \epsilon^2}$} & \multirow{2}{*}{$\frac{S^{\frac{1}{3}}A^{\frac{1}{3}}T^{\frac{2}{3}}}{(1-\gamma)^{\frac{13}{6}}}$} & \multirow{2}{*}{never} & \multirow{2}{*}{$SA$}\\
    \citep{zhang2021model} & & & & \\
    \hline
    {\sf UCB-multistage-adv} & \multirow{2}{*}{$\frac{SA}{(1-\gamma)^3 \epsilon^2}$} & \multirow{2}{*}{$\frac{S^{\frac{1}{3}}A^{\frac{1}{3}}T^{\frac{2}{3}}}{(1-\gamma)^{\frac{4}{3}}}$} & \multirow{2}{*}{never} & \multirow{2}{*}{$SA$}\\
    \citep{zhang2021model} & & & & \\
    \hline
    {\sf MAIN} & \multirow{2}{*}{N/A} & \multirow{2}{*}{$\kappa\sqrt{\frac{(S^4+S^2A^2)T}{(1-\gamma)^8}}$} & \multirow{2}{*}{never} & \multirow{2}{*}{$SA$}\\
    \citep{kash2022slowly} & & & & \\
    \hline
    {\sf Double Q-learning} & \multirow{2}{*}{N/A} & \multirow{2}{*}{$\sqrt{\frac{SAT}{(1-\gamma)^5}}$} & \multirow{2}{*}{never} & \multirow{2}{*}{$SA$}\\
    \citep{liu2020regret} & & & & \\
    \hline
    {\sf UCBVI-$\gamma$} & \multirow{2}{*}{N/A} & \multirow{2}{*}{$\sqrt{\frac{SAT}{(1-\gamma)^3}}$} & \multirow{2}{*}{$\Big[\frac{S^3A^2}{(1-\gamma)^4}, \infty\Big)$ $^\dagger$} & \multirow{2}{*}{$S^2A$}\\
    \citep{zhou2021nearly} & & & & \\
    \hline
    \algsf & \multirow{2}{*}{N/A} & \multirow{2}{*}{$\sqrt{\frac{SAT}{(1-\gamma)^3}}$} & \multirow{2}{*}{$\Big[\frac{SA}{(1-\gamma)^{13}}, \infty\Big)$} & \multirow{2}{*}{$SA$}\\
    (\textbf{This work}) & & & & \\
    \hline
    {\sf Lower bound} & \multirow{3}{*}{$\frac{SA}{(1-\gamma)^3 \epsilon^2}$} & \multirow{3}{*}{$\sqrt{\frac{SAT}{(1-\gamma)^3}}$} & \multirow{3}{*}{N/A} & \multirow{3}{*}{N/A}\\
    (\citet{lattimore2012pac}; & & & & \\
    \citet{zhou2021nearly}) & & & & 
\end{tabular}
\end{center}
\caption{A comparison between our results and existing work in the online infinite-horizon discounted setting. Logarithmic factors are omitted for clearer presentation. The second column shows the sample complexity when the target accuracy $\epsilon$ is sufficiently small. The third column shows the regret when sample size $T$ is sufficiently large (beyond the burn-in period). The algorithms in the first seven rows only have sample complexity results in their original works; their regret bounds are derived from their respective sample complexity bounds and presented in this table for completeness. Details about the conversions can be found in \citet{zhou2021nearly}. Note that {\sf UCB-multistage-adv} achieves optimal sample complexity only in the high-accuracy regime when $\epsilon \le S^{-2}A^{-2}(1-\gamma)^{14}$. This is similar to a burn-in threshold in that the optimal guarantee cannot be achieved unless in a specific range. The fourth column lists the sample size range in which regret optimality can be attained, which shows the burn-in time. We would like to point out that the results in \citet{zhou2021nearly,liu2020regret} are under slightly different regret definitions from the regret definition in \citet{kash2022slowly} and this work. In fact, their regret metric can be more lenient, and our algorithm can also achieve $\widetilde{O}(\sqrt{\frac{SAT}{(1-\gamma)^3}})$ optimal regret under it. This is further discussed in Remark \ref{remark:metric-diff} and Appendix \ref{sec:appendix-metric-diff}. Lastly, note that \citet{kash2022slowly} assume an ergodicity parameter $\kappa$. $^\dagger$ {\sf UCBVI-$\gamma$} achieves optimal regret for $T \ge \frac{S^3A^2}{(1-\gamma)^4}$ only if the MDP satisfies $SA \ge \frac{1}{1-\gamma}$.}
	\label{fig:table1}
\end{table*}

\subsection{Burn-in Cost in Regret-Optimal RL}

Naturally, one aims to develop algorithms that find the optimal policy with the fewest number of samples. In regards to regret, this motivates numerous works to work towards algorithms with minimax-optimal cumulative regret. However, the job is not done once such an algorithm is found. As can be seen in the episodic RL literature, algorithms that achieve optimal regret as sample size $T$ tends towards infinity can still have different performance in the regime when $T$ is limited. Specifically, for every existing algorithm, there exists a certain sample size threshold such that regret is suboptimal before $T$ exceeds it. Such threshold is commonly referred to as the initial \textit{burn-in time} of the algorithm. Therefore, it is of great interest to find an algorithm with low burn-in time so that it can still attain optimal regret in the sample-starved regime. Such effort has been made by \citet{li2020breaking,agarwal2020model} in the generative setting and by \citet{Li-finite,menard2021ucb} in the online episodic setting. Yet, this important issue has not been addressed in the infinite-horizon setting, as optimal algorithms all suffer long burn-in times.

Specifically, while {\sf UCBVI-$\gamma$} in \citet{zhou2021nearly} achieves a state-of-the-art regret guarantee of $\widetilde{O}\Big(\sqrt{\frac{SAT}{(1-\gamma)^3}}\Big)$, which they prove minimax-optimal, their theory does not guarantee optimality unless the samples size $T$ becomes as large as
\begin{equation*}
	T \ge \frac{S^3A^2}{(1-\gamma)^4}.
\end{equation*}
This threshold can be prohibitively large when $S$ and $A$ are huge, which is true in most applications. For instance, a 5-by-5 tic-tac-toe has a state space of size $3^{25}$. While this is a manageable number in modern machine learning, any higher power of it may cause computational difficulties; in contrast, the horizon of this game is much smaller---no more than $25$. The issue exacerbates in more complex applications such as the game of Go \citep{silver2016}; thus, reducing the $S$ and $A$ factors in the burn-in cost is particularly important. Since no lower bound precludes regret optimality for $T \ge \frac{SA}{(1-\gamma)^4}$, one might hope to design an algorithm with smaller $S$ and $A$ factors in the burn-in cost so that it can achieve optimality even in the sample-starved regime. 

\subsection{Summary of Contributions}

While it is encouraging to see recent works have shown that in the discounted setting, model-free methods can provide nearly optimal guarantees on sample complexity of exploration and that model-based methods can provide nearly optimal finite-sample regret guarantees, there still lacks a \textit{model-free} approach that can attain \textit{regret optimality}. In the orthogonal direction, there is still a vacancy for algorithms that can attain optimal regret for a broader sample size range, i.e., with fewer samples than $\frac{S^3A^2}{\poly(1-\gamma)}$. 

In fact, we can summarize these two lingering theoretical questions as follows:
\begin{center}
	{\em Is there an algorithm that can achieve minimax regret optimality with low space complexity and computational complexity in the infinite-horizon discounted setting, even when sample size is limited?}
\end{center}
We answer this question affirmatively with a new algorithm \algo, which uses variance reduction and a novel adaptive switching technique. It is the first model-free algorithm that achieves optimal regret in the infinite-horizon discounted setting. This result can be summarized as follows:
\begin{theorem*}[informal]
	For any sample size $T \ge \frac{SA}{\poly(1-\gamma)}$, \alg is guaranteed to achieve near-optimal cumulative regret $\widetilde{O}\left(\sqrt{\frac{SAT}{(1-\gamma)^3}}\right)$ with space complexity $O(SA)$ and computational complexity $O(T)$.
\end{theorem*}

A formal theorem is presented in Section \ref{sec:results}. We also provide a complete summary of related prior results in Table \ref{fig:table1}. 

\subsection{Related work} 

Now, let us take a moment to discuss the related work beyond those in Table \ref{fig:table1}.

\paragraph{Regret analysis for online episodic RL} In the online episodic setting, regret is the predominant choice of metric for demonstrating the sample efficiency of a method \citep{yang2021q,pacchiano2020optimism,Jaksch10}. \citet{azar2017minimax} was the first to introduce a model-based method that can achieve near-optimal regret guarantee, but the model-based nature of their method induces a high space complexity and burn-in time. On the other hand, model-free methods are proposed in \citet{Jin-provably,bai2019provably}, which are motivated by Q-learning and thus enjoy a low space complexity. However, these methods cannot guarantee optimal regret. It was not until \citet{zhang2020almost} that proposed the first model-free method with optimal regret guarantee {\sf UCB-Q-Advantage}, but it incurs a large burn-in time of $S^6A^4H^{28}$, where $H$ is the horizon of the episodic MDP. In addition, \citet{menard2021ucb} proposed {\sf UCB-M-Q}, a Q-learning variant with momentum, which can achieve optimal regret with low burn-in time, but it requires the storage of all momentum bias and thus incurs high memory cost. Recently, \citet{Li-finite} propose a Q-learning variant with variance reduction that achieves optimal regret with $O(SAH)$ space complexity and $SA\poly(H)$ burn-in threshold at the same time. Table 1 in \citet{Li-finite} provides a more detailed comparison of related work from the online episodic RL literature.

\paragraph{Sample complexity for infinite-horizon RL} In the infinite-horizon setting, there exist other sample efficiency metrics besides sample complexity of exploration. Initially, \citet{kearns99} considered the sample complexity needed to find an $\epsilon$-approximate optimal policy. The same definition is also considered in \citet{wang2017randomized,sidford2018variance,sidford2018near,li2020breaking}. Later, \citet{wainwright2019variance} studied the sample complexity needed to find an $\epsilon$-approximate optimal Q-function. Note that all of these works assume the generative setting. Indeed, a limitation of these aforementioned sample complexity definitions is that they only measure the performance of the final output policy and do not reflect the online regret during learning. Thus, existing works that study the online setting consider sample complexity of exploration and cumulative regret instead.

\paragraph{Variance reduction in RL} The idea of variance reduction was first introduced to accelerate stochastic finite-sum optimization by \citet{Johnson-SVRG}, which is followed by a rich literature \citep{xiao2014proximal,nguyen2017sarah,ge2019stabilized}. Later, for better sample efficiency in RL, it is applied to policy gradient methods \citep{liu2020improved,zhang2021convergence,papini2018stochastic} as well as value-based methods in various problems including generative setting RL \citep{sidford2018near,sidford2018variance,wainwright2019variance}, policy evaluation \citep{du2017stochastic,xu2020reanalysis}, asynchronous Q-learning \citep{li2020sample,yan2022efficacy} and offline RL \citep{yin2021near,shi2022pessimistic}.

\paragraph{Low-switching algorithms} Since our algorithm includes a novel feature that switches the execution policy slowly, we make a review of the low-switching approaches in RL. The idea of changing the execution policy infrequently during learning was first introduced by \citet{Auer02} as an approach to minimize regret in the multi-armed bandit problem. \citet{bai2019provably} adapted this idea to tabular RL and formalized the switching cost as a secondary metric that an algorithm can minimize. To reduce the number of policy switches and thus the switching cost, their algorithm updates the policy in geometrically longer intervals. Similar techniques can be found in \citet{zhang2020almost}, whose algorithm can achieve regret optimality while maintaining low switching cost. Later, \citet{gao2021provably,wang2021provably} introduced a new low-switching approach in linear MDPs by switching policies when the estimated covariance matrix gets a significant update. All these methods guarantee a $O(\log T)$ switching cost. The switching cost guarantee was later improved to $O(\log\log T)$ by the algorithms proposed in \citet{qiao2022sample,zhang2022near}. 


\section{Problem Formulation}\label{sec:prelim}

Let us specify the problem we aim to study in this section. Throughout this paper, we let $\Delta(\cX)$ denote the probability simplex over any set $\cX$. We also introduce the notation $[m] := \{1,2,\cdots, m\}$ for a positive integer $m$.

\subsection{Infinite-Horizon Discounted Markov Decision Process}

We consider an infinite-horizon discounted Markov decision process (MDP) represented with $(\cS, \cA, \gamma, P, r)$. Notably, we consider a tabular one, in which $\cS := \{1, 2, \cdots, S\}$ denotes the state space with size $S$ and $\cA := \{1, 2, \cdots, A\}$ denotes the action space with size $A$. $P:\cS\times \cA \to \Delta(\cS)$ denotes the probability transition kernel in that $P(\cdot | s,a) \in \Delta(\cS)$ is the transition probability vector from state $s \in \cS$ when action $a \in \cA$ is taken. $r: \cS \times \cA \to [0,1]$ denotes the reward function, which is assumed to be deterministic in this work. Specifically, $r(s,a)$ is the immediate reward for taking action $a\in \cA$ at state $s \in \cS$. Lastly, $\gamma$ denotes the discount factor for the reward, which makes $\frac{1}{1-\gamma}$ the effective horizon.

A (stationary) policy $\pi: \cS\to\Delta(\cA)$ specifies a rule for action selection in that $\pi(\cdot | s) \in \Delta(\cA)$ is the action selection probability vector at state $s \in \cS$. We overload this notation by letting $\pi(s)$ denote the action policy $\pi$ takes at state $s$. Given a policy $\pi$, the Q-function of $\pi$ is defined as 
\begin{align*}
	Q^\pi(s,a) := \EE\left[\sum_{t=0}^{\infty} \gamma^t r(s_{t}, a_{t}) ~\Big|~ s_0 = s, a_0 = a\right],
\end{align*}
in which $s_{t+1}\sim P(\cdot | s_{t},a_{t})$ for $t \ge 0$ and $a_{t}\sim\pi(\cdot | s_t)$ for $t \ge 1$. Moreover, the value function of $\pi$ is defined as 
\begin{align*}
	V^\pi(s) := \EE\left[\sum_{t=0}^{\infty} \gamma^t r(s_{t}, a_{t}) ~\Big|~ s_0 = s\right],
\end{align*}
in which $s_{t+1}\sim P(\cdot | s_{t},a_{t})$ and $a_{t}\sim\pi(\cdot | s_t)$ for $t \ge 0$. The Q-function and value function satisfy an equation, called the Bellman equation \citep{Bert05}:
\begin{equation}\label{eq:bellman}
    Q^\pi(s,a) = r(s,a) + \gamma \EE_{s'\sim P(\cdot|s,a)}\left[V^\pi(s')\right].
\end{equation}

A policy $\pi^\star$ is called an optimal policy if it maximizes the value function for all states simultaneously. The optimal value function and optimal Q-function can be defined as
\begin{align*}
    V^\star(s) &:= \max_{\pi} V^\pi(s) = V^{\pi^\star}(s) \\
    Q^\star(s,a) &:= \max_{\pi} Q^\pi(s,a) = Q^{\pi^\star}(s,a),
\end{align*}
which satisfy 
\begin{align*}
    V^\star(s) = V^{\pi^\star}(s) \quad \text{and} \quad
    Q^\star(s,a) = Q^{\pi^\star}(s,a)
\end{align*}
for any optimal policy $\pi^\star$. The optimal policy always exists and satisfies the Bellman optimality equation \citep{puterman94}:
\begin{align}\label{eq:bellman-opt}
    Q^{\pi^\star}(s,a) &= r(s,a) + \gamma \EE_{s'\sim P(\cdot|s,a)}\left[\max_{a'\in\cA}Q^{\pi^\star}(s',a')\right] \notag\\
    &= r(s,a) + \gamma \EE_{s'\sim P(\cdot|s,a)}\left[V^{\star}(s')\right].
\end{align}

\subsection{Online Learning in an Infinite-Horizon MDP}

We consider the online (single-trajectory) setting, in which the agent is permitted to execute a total of $T$ steps sequentially in the MDP. More specifically, the agent starts from an arbitrary (and possibly adversarial) initial state $s_1$. At each step $t \in [T]$, the agent at state $s_t$ computes policy $\pi_t$, takes action $a_t$ based on $\pi_t(\cdot|s_t)$, receives reward $r(s_t,a_t)$, and transitions to state $s_{t+1}$ in the following step. At the end of execution, the agent generates a trajectory $(s_1,a_1,r_1, s_2,a_2,r_2, \cdots,s_T,a_T,r_T)$, which amounts to $T$ samples.

\subsection{Problem: Regret Minimization} 

As a standard metric to evaluate the performance of the aforementioned agent over a finite number of $T$ steps, the cumulative regret with respect to the sequence of stationary policies $\{\pi_t\}_{t=1}^T$ learned by the algorithm is defined as follows:
\begin{equation}\label{eq:regret}
    \Regret(T) := \sum_{t=1}^T \Big( V^\star(s_t) - V^{\pi_t}(s_t)\Big).
\end{equation}

Verbally, the regret measures the cumulative suboptimality between the optimal policy and each learned policy $\pi_t$ throughout the $T$-step online interaction process. Naturally, one aims to minimize this regret by finding an algorithm whose regret scales optimally in $T$. This would require a strategic balance between exploration and exploitation, which can be difficult when sample size $T$ is small.

\begin{remark}\label{remark:metric-diff}
    In the infinite-horizon setting, many prior works \citep{zhou2021nearly,liu2020regret} consider slightly different regret definitions with respect to non-stationary policies. Specifically, at each $s_t$ along the trajectory, this different regret metric compares the optimal value function $V^{\star}(s_t)$ against the expected cumulative reward of running the non-stationary policy $\{\pi_k\}_{k=t}^\infty$ starting from $s_t$. By doing this, it is effectively evaluating the cumulative reward difference between the stationary optimal policy and a non-stationary algorithm. While there exists no formal conversion between the regret defined in this way and the one in \eqref{eq:regret}, it is expected to be smaller and thus more easily controlled than \eqref{eq:regret}, because the execution policy $\pi_t$ improves over time. In addition, we can show our algorithm also achieves the same level of regret under this different definition with an analysis specifically tailored to our algorithm, which is deferred to Appendix \ref{sec:appendix-metric-diff}. Furthermore, since the transition kernel in the infinite-horizon setting is invariant over time and the optimal policy itself is also stationary, it is more natural to compare the optimal policy to a stationary policy, e.g., the policy $\pi_t$ deployed by the algorithm at each step, as in \eqref{eq:regret}. Before this work, this has also been recently studied in \citet{zhang2021model,kash2022slowly}.  
\end{remark}

\paragraph{Notation.} Given any vector $x\in\RR^{SA}$ that represents a function $x:\cS\times\cA\to\RR$, we use $x(s,a)$ to denote the entry corresponding to the state-action pair $(s,a)$. We also denote the probability transition vector at $(s,a)$ with  
\begin{equation}\label{eq:P-def}
    P_{s,a} = P(\cdot~|~s,a)\in \RR^{1\times S},
\end{equation}
that is, given any $V\in \RR^S$, $P_{s,a}V = \EE_{s'\sim P(\cdot|s,a)}[V(s')]$. For two vectors $x,y\in\RR^{SA}$, we override the notation $x \le y$ to mean that $x(s,a) \le y(s,a)$ in every dimension $(s,a)$.


\section{Algorithm}\label{sec:alg}

In this section, we present our algorithm \alg and some relevant discussion.

\subsection{Review: Q-Learning with UCB and reference advantage}

First, we make a brief review of the Q-learning with UCB method proposed in \citet{Jin-provably}, referred to as {\sf UCB-Q} hereafter, and its variance-reduced variant {\sf UCB-Q-Advantage}, later introduced in \citet{zhang2020almost}. The Q-function updates in \alg are inspired by these two methods.

\paragraph{Q-learning with UCB}

The original Q-learning \citep{Watkins-Q,Watkins-thesis} is a fixed-point iteration based on a stochastic approximation of the Bellman optimality equation \eqref{eq:bellman-opt}. It uses a greedy policy with respect to its estimate of $Q^\star$, whose update rule can be summarized as: 
\begin{equation}
	Q(s, a) \leftarrow (1 - \eta)Q(s, a) + \eta\left(r(s, a) + \gamma\widehat{P}_{s,a} V\right). \label{eq:Q-update}
\end{equation}
Above, $Q$ (resp. $V$) is the running estimate of $Q^\star$ (resp. $V^\star$); $\eta \in (0,1]$ is the (possibly varying) learning rate; $\widehat{P}_{s,a} V$ is a stochastic approximation of $P_{s,a} V$ (cf. \eqref{eq:P-def}). Commonly, $V(s')$ is used for $\widehat{P}_{s,a} V$ in \eqref{eq:Q-update} as an unbiased estimate of $P_{s,a} V$, when a sample of state transition from $(s,a)$, namely $(s,a,s')$, is available. 

However, as \citet{Jin-provably} point out, using \eqref{eq:Q-update} na\"{i}vely suffers from great regret suboptimality, for it rules out the state-action pairs with high value but few observations. To promote the exploration of such state-action pairs, {\sf UCB-Q} appends \eqref{eq:Q-update} with an exploration bonus. Its update rule can be written as:
\begin{align}
	Q^\UCB(s, a) &\leftarrow (1 - \eta)Q^\UCB(s, a) + \eta\Big(r(s, a) + \gamma\widehat{P}_{s,a} V + b\Big). \label{eq:UCB-Q-update}
\end{align}
To encourage exploration, the bonus $b \ge 0$ is designed to maintain an upper confidence bound (UCB) on $(\widehat{P}_{s,a} - P_{s,a}) V$, which in turn keeps $Q^\UCB(s, a)$ as an ``optimistic'' overestimate of $Q^\star(s,a)$.

\paragraph{Q-learning with UCB and reference advantage}

The regret guarantee for {\sf UCB-Q} is still shy of being optimal. In order to attain optimality, one can turn to the celebrated idea of variance reduction \citep{Johnson-SVRG,Li-finite,sidford2018variance,wainwright2019variance}, which decomposes the stochastic approximation target into two parts: a low-variance reference estimated with batches of samples and a low-magnitude advantage estimated with every new sample. In this spirit, \citet{zhang2020almost} introduce {\sf UCB-Q-Advantage} based on {\sf UCB-Q} and a reference-advantage decomposition. Specifically, given a reference $V^\rref$ that is maintained as an approximation for $V^\star$, the update rule of {\sf UCB-Q-Advantage} reads:
\begin{align}
	Q^\rref(s, a) &\leftarrow (1 - \eta)Q^\rref(s, a) + \eta\Big(r(s, a) + \gamma\big(\widehat{P}_{s,a} \left(V - V^\rref \right) + \widehat{PV^\rref}(s,a)\big)  + b^{\rref}\Big). \label{eq:UCB-Q-Adv-update}
\end{align}
Let us discuss the update rule \eqref{eq:UCB-Q-Adv-update} in more details:
\begin{itemize}
	\item Given a transition sample $(s,a,s')$, we can take $V(s') - V^\rref(s')$ as an unbiased estimate of the advantage $P_{s,a} (V - V^\rref )$. Note that the magnitude of $V - V^\rref$ is small when $V$ and $V^\rref$ are close. This engenders smaller stochastic volatility, compared to $\widehat{P}_{s,a} V$ in \eqref{eq:UCB-Q-update} from {\sf UCB-Q}.
	
	\item The reference estimate $\widehat{PV^\rref}$ is a stochastic approximation of $PV^\rref$. In our algorithm, the auxiliary estimate $\mu^\re$ (cf. Line \ref{alg-line:mu-ref} of Algorithm \ref{alg:aux}) is used as the estimate for $PV^\rref$. Specifically, $\mu^\re(s,a)$ is a running mean of $P_{s,a}V^\rref$ based on the samples from all past visitations of $(s,a)$. In contrast to the advantage, which is computed every time a new sample arrives, the reference is computed with a batch of samples and thus more stable. In sacrifice, the reference is only updated intermittently and not as up-to-date as the advantage. 
\end{itemize}
The exploration bonus $b^{\rref}$ is computed from the auxiliary estimates in Line \ref{alg-line:update-moments-call} and \ref{alg-line:update-bonus-call} to serve as an upper confidence bound on the aggregation of the aforementioned reference and advantage. Specifically, for each $(s,a)$, $\mu^\re(s,a)$ and $\sigma^\re(s,a)$ are the running mean and $2$nd moment of the reference $[PV^\rref](s,a)$ respectively; $\mu^\adv(s,a)$ and $\sigma^\adv(s,a)$ are the running mean and $2$nd moment of the advantage $[P(V - V^\rref)](s,a)$ respectively; $B^\rref(s,a)$ combines the empirical standard deviations of the reference and the advantage; $\delta^\rref(s,a)$ is the temporal difference between $B^\rref(s,a)$ and its previous value. $b^{\rref}(s,a)$ can be computed from these estimates as a temporally-weighted average of $B^\rref(s,a)$. Thanks to the low variability of the reference $PV^\rref$, we can obtain a more accurate, milder overestimation in the upper confidence bound for faster overall convergence. 

\begin{algorithm}[t]
	\caption{\alg \label{alg:main}}
	\textbf{Initialize:} $\forall (s,a)$, $\Qlazy(s,a),Q(s,a),Q^\UCB(s,a),Q^\rref(s,a), \QM(s,a) \leftarrow \frac{1}{1-\gamma}$; $N(s,a) \leftarrow 0$; $V(s), V^{\rref}(s) \leftarrow \frac{1}{1-\gamma}$; $Q^{\mathrm{LCB}}(s,a), \qdiff(s,a) \leftarrow 0$; $\qset \leftarrow \dict()$; $\mu^{\re}(s,a), \sigma^{\re}(s,a), \mu^{\adv}(s,a), \sigma^{\adv}(s,a), B^{\rref}(s,a),$ $\delta^{\rref}(s,a) \leftarrow 0$; $u^{\mathrm{switch}} \leftarrow \mathrm{False}$; $u^{\re}(s) \leftarrow \mathrm{True}$; $H = \lceil\frac{2}{1-\gamma} \rceil$; $\iota = \log\frac{SAT}{\delta}$.
	
	\For{$t = 1, \cdots, T$}{
            Take action $a_t = \pi_{t}(s_t) = \arg \max_a Q^{\mathrm{lazy}}(s_t, a)$, and draw $s_{t+1} \sim P(\cdot | s_t, a_t)$;\\
            $N(s_t, a_t) \leftarrow N(s_t, a_t) + 1$; $n \leftarrow N(s_t, a_t)$; \quad \textcolor{blue}{\# Update the counter}\\
		    $\eta_n \leftarrow \frac{H+1}{H+n}$; \quad \textcolor{blue}{\# Update the learning rate}\\
		    $Q^{\mathrm{UCB}}(s_t, a_t) \leftarrow \textbf{update-q-ucb}()$; \quad \textcolor{blue}{\# Compute the UCB. See Algorithm \ref{alg:aux}}\\
		    $Q^{\mathrm{LCB}}(s_t, a_t) \leftarrow \textbf{update-q-lcb}()$; \quad \textcolor{blue}{\# Compute the LCB. See Algorithm \ref{alg:aux}}\\
		    $Q^{\mathrm{R}}(s_t, a_t) \leftarrow \textbf{update-q-reference}()$; \quad \textcolor{blue}{\# Compute the reference value. See Algorithm \ref{alg:aux}}\\
		    $Q(s_t, a_t) \leftarrow \min \{Q^{\mathrm{R}}(s_t, a_t), Q^{\mathrm{UCB}}(s_t, a_t), Q(s_t, a_t)\}$; \label{alg-line:Q-min-update}\\
		    $V(s_t) \leftarrow \max_a Q(s_t, a)$;\\
		    $V^{\mathrm{LCB}}(s_t) \leftarrow \max\{\max_a Q^{\mathrm{LCB}}(s_t, a), V^{\mathrm{LCB}}(s_t)\}$;\\
		    $\qdiff(s_t,a_t) \leftarrow \qdiff(s_t,a_t) + \QM(s_t, a_t) - Q(s_t, a_t)$. \quad \textcolor{blue}{\# Track the staleness of current policy} \label{alg-line:q-update}

            \If{$u^{\mathrm{switch}} = \mathrm{True}$}{
                 $\Qlazy \leftarrow \textbf{update-q-lazy}()$; \quad \textcolor{blue}{\# Update execution policy's Q-function (switch policy)}\\
                 $\qset \leftarrow \dict()$; \quad \textcolor{blue}{\# Reset the buffer}\\
                 $u^{\mathrm{switch}} \leftarrow \mathrm{False}$;
            }
            
            $\qset[(s_t, a_t)] \leftarrow Q(s_t, a_t)$. \quad \textcolor{blue}{\# Add the new transition and $Q$ entry to the buffer}\\

            \textcolor{blue}{"""Switch policy when the staleness tracker is large"""}\\
    		\If{$\qdiff(s_t,a_t) > \frac{1}{1-\gamma}$ \label{alg-line:q-cond}}{
                $\QM(s_t, a_t) \leftarrow Q(s_t,a_t)$; \quad \textcolor{blue}{\# Save the current $Q$ entry for staleness determination later}\\
    			$u^{\mathrm{switch}} \leftarrow \mathrm{True}$; \quad \textcolor{blue}{\# Signal to switch policy at the following step}\\
    			$\qdiff(s_t,a_t) \leftarrow 0$; \quad \textcolor{blue}{\# Reset the staleness tracker}
            }

            \uIf{$V(s_{t}) - V^{\mathrm{LCB}}(s_{t}) > 3$ \label{alg-line:if-cond-1}}{
                $V^{\mathrm{R}}(s_{t}) \leftarrow V(s_{t})$, $u^{\mathrm{ref}}(s_{t}) = \mathrm{True}$; \label{alg-line:if-stat-1}
            }
            \ElseIf{$u^{\mathrm{ref}}(s_{t}) = \mathrm{True}$ \label{alg-line:elif-cond-1}}{
                $V^{\mathrm{R}}(s_{t}) \leftarrow V(s_{t})$, $u^{\mathrm{ref}}(s_{t}) = \mathrm{False}$. \label{alg-line:elif-stat-1} \quad \textcolor{blue}{\# Update the reference only on certain conditions} 
            }
        }
\end{algorithm}

\begin{algorithm}[t]
    \caption{Auxiliary functions \label{alg:aux}}
    \SetKwFunction{FUCB}{update-q-ucb}
    \SetKwProg{Fn}{Function}{:}{}
    \Fn{\FUCB{}}{
        $Q^{\mathrm{UCB}}(s_t, a_t) \leftarrow (1 - \eta_n)Q^{\mathrm{UCB}}(s_t, a_t) + \eta_n\Big(r(s_t, a_t) + \gamma V(s_{t+1}) + c_b\sqrt{\frac{\iota}{(1-\gamma)^3n}}\Big).$ \label{alg-line:QUCB-update}
    }

    \SetKwFunction{FLCB}{update-q-lcb}
    \SetKwProg{Fn}{Function}{:}{}
    \Fn{\FLCB{}}{
        $Q^{\mathrm{LCB}}(s_t, a_t) \leftarrow (1 - \eta_n)Q^{\mathrm{LCB}}(s_t, a_t) + \eta_n\Big(r(s_t, a_t) + \gamma V^{\mathrm{LCB}}(s_{t+1}) - c_b\sqrt{\frac{\iota}{(1-\gamma)^3n}}\Big).$ \label{alg-line:QLCB-update}
    }

    \SetKwFunction{Flazy}{update-q-lazy}
    \SetKwProg{Fn}{Function}{:}{}
    \Fn{\Flazy{}}{
        \textbf{for} every $((s,a), q) \in \qset$ \textbf{do} $\Qlazy(s,a) \leftarrow q$. \quad \textcolor{blue}{\# Update execution policy with the buffer} \label{alg-line:Qlazy-update}
    }

    \SetKwFunction{FUCBA}{update-q-reference}
    \SetKwProg{Fn}{Function}{:}{}
    \Fn{\FUCBA{}}{
        $[\mu^{\mathrm{ref}}, \sigma^{\mathrm{ref}}](s_t, a_t) \leftarrow \textbf{update-moments}()$;\label{alg-line:update-moments-call}
	
    	$[\delta^{\mathrm{R}}, B^{\mathrm{R}}](s_t, a_t) \leftarrow \textbf{update-reference-bonus}()$; \label{alg-line:update-bonus-call}
    	
    	$b^{\mathrm{R}} \leftarrow B^{\mathrm{R}}(s_t, a_t) + (1 - \eta_n)\frac{\delta^{\mathrm{R}}(s_t, a_t)}{\eta_n} + c_b\frac{\iota^2}{n^{3/4}(1-\gamma)^2}$; \label{alg-line:b}
    	
    	$Q^{\mathrm{R}}(s_t, a_t) \leftarrow (1 - \eta_n)Q^{\mathrm{R}}(s_t, a_t) + \eta_n\left(r(s_t, a_t) + \gamma \left(V(s_{t+1}) - V^{\mathrm{R}}(s_{t+1}) + \mu^{\mathrm{ref}}(s_t, a_t)\right) + b^{\mathrm{R}}\right)$. \label{alg-line:QR-update}
    }

    \SetKwFunction{Fm}{update-moments}
    \SetKwProg{Fn}{Function}{:}{}
    \Fn{\Fm{}}{
        $\mu^{\mathrm{ref}}(s_t, a_t) \leftarrow (1 - \frac{1}{n})\mu^{\mathrm{ref}}(s_t, a_t) + \frac{1}{n}V^{\mathrm{R}}(s_{t+1})$; \quad \textcolor{blue}{\# Running average of the reference} \label{alg-line:mu-ref}
	
    	$\sigma^{\mathrm{ref}}(s_t, a_t) \leftarrow (1 - \frac{1}{n})\sigma^{\mathrm{ref}}(s_t, a_t) + \frac{1}{n}\left(V^{\mathrm{R}}(s_{t+1})\right)^2$; \quad \textcolor{blue}{\# Running $2^{\text{nd}}$ moment of the reference} \label{alg-line:sigma-ref}
    	
    	$\mu^{\mathrm{adv}}(s_t, a_t) \leftarrow (1 - \eta_n)\mu^{\mathrm{adv}}(s_t, a_t) + \eta_n \left(V(s_{t+1}) - V^{\mathrm{R}}(s_{t+1})\right)$; \quad \textcolor{blue}{\# Running average of the advantage} \label{alg-line:mu-adv}
    	
    	$\sigma^{\mathrm{adv}}(s_t, a_t) \leftarrow (1 - \eta_n)\sigma^{\mathrm{adv}}(s_t, a_t) + \eta_n \left(V(s_{t+1}) - V^{\mathrm{R}}(s_{t+1})\right)^2$. \textcolor{blue}{\# Running $2^{\text{nd}}$ moment of the advantage} \label{alg-line:sigma-adv}
    }

     \SetKwFunction{Fb}{update-reference-bonus}
    \SetKwProg{Fn}{Function}{:}{}
    \Fn{\Fb{}}{
        $B^{\mathrm{next}}(s_t, a_t) \leftarrow c_b\sqrt{\frac{\iota}{n}}\left(\sqrt{\sigma^{\mathrm{ref}}(s_t, a_t) - \left(\mu^{\mathrm{ref}}(s_t, a_t)\right)^2} + \frac{1}{\sqrt{1-\gamma}}\sqrt{\sigma^{\mathrm{adv}}(s_t, a_t) - \left(\mu^{\mathrm{adv}}(s_t, a_t)\right)^2}\right)$; \label{alg-line:Bnext}
	
    	$\delta^{\mathrm{R}}(s_t, a_t) \leftarrow B^{\mathrm{next}}(s_t, a_t) - B^{\mathrm{R}}(s_t, a_t)$; \label{alg-line:delta}
    	
    	$B^{\mathrm{R}}(s_t, a_t) \leftarrow B^{\mathrm{next}}(s_t, a_t)$. \label{alg-line:BR}
    }
\end{algorithm}

\subsection{Review: Early settlement of reference value}

Given the optimistic overestimates $Q^\UCB$ and $Q^\rref$, it is natural to design an update rule of our Q-function estimate as the minimum of the estimate itself and these two overestimates (Line \ref{alg-line:Q-min-update} of Algorithm \ref{alg:main}). This makes our Q-function estimate monotonically decrease without violating the optimistic principle $Q \ge Q^\star$, which effectively enables us to lessen the overestimation in $Q$ over time until it converges to $Q^\star$. In fact, this is precisely the update rule in {\sf UCB-Q-Advantage} \citep{zhang2020almost}. Nevertheless, we need to equip our algorithm with additional features to strive for regret optimality.

\citet{Li-finite} introduced a way to update the reference with higher sample efficiency in the finite-horizon setting. As they noted, it is critical to update the reference $V^\rref$ in a smart fashion so as to balance the tradeoff between its synchronization with $V$ and the volatility that results from too many stochastic updates. Concretely, the reference $V^\rref$ needs to be updated in a timely manner from $V$ so that the magnitude of $\widehat{P}_{s,a} (V - V^\rref )$ can be kept low as desired, but the updates cannot be too frequent either, because the stochasticity or variance in $\widehat{PV^\rref}(s,a)$ could be as high as that in $\widehat{P}_{s,a} V$ of \eqref{eq:UCB-Q-update} and thus lead to suboptimality if it is not carefully controlled. To resolve this dilemma, we can update $V^\rref$ until it becomes sufficiently close to $V^\star$ and fix its value thereafter. 

To this end, we maintain a ``pessimistic'' underestimate $Q^\LCB$ (resp. $V^\LCB$) of $Q^\star$ (resp. $V^\star$) in the algorithm, which are computed from the lower confidence bound for $Q^\star$ (resp. $V^\star$). This can provide us with an upper bound on $V^\rref - V^\star$, which will be used to determine when the update of the reference $V^\rref$ should be stopped.

In particular, the if-else block in Line \ref{alg-line:if-cond-1} to \ref{alg-line:elif-stat-1} is designed to keep the reference $V^\rref$ synchronized with $V$ for each state $s$ respectively and terminate the update once 
\begin{equation}
	V(s) \le V^\LCB + 3 \le V^\star + 3.
\end{equation}
This can guarantee $|V - V^\rref| \le 6$ throughout the execution of the algorithm. As a result, the standard deviation of $\widehat{P}_{s,a} (V - V^\rref )$ is guaranteed to be $O(1)$, which can be $O(\frac{1}{1-\gamma})$ times smaller than the standard deviation of $\widehat{P}_{s,a} V$ in \eqref{alg-line:QUCB-update}. This can lead to smaller $\frac{1}{1-\gamma}$ factor in the final regret guarantee.

\subsection{Adaptive low-switching greedy policy}\label{sec:alg:low-switching}

Although these aforementioned designs from the finite-horizon literature help increase the accuracy of our estimate $Q$, they are still insufficient to attain regret optimality in the infinite-horizon setting. Since data collection takes place over a single trajectory with no reset, drastic changes in the execution policy can inflict a long-lasting volatility on the future trajectory and slow down the convergence. This is precisely the difficulty of the infinite-horizon setting over the finite-horizon one. The need to control the trajectory variability motivates us to design a novel adaptive switching technique.

Recall in {\sf UCB-Q} and {\sf UCB-Q-Advantage}, the execution policy is greedy with respect to the estimate $Q$, i.e., $\pi_t(s_t) = \arg\max_a Q(s_t,a)$. Every time $Q$ gets updated, what the algorithm effectively does is to make an estimate of $Q^{\pi_t}$ with the samples generated by $\pi_t$. Such $Q^{\pi_t}$ is only estimated and updated once before the execution policy is switched to $\pi_{t+1}$. This seems insufficient from a stochastic fixed-point iteration perspective, so we seek to update it more and learn each $Q^{\pi_t}$ better before switching to a new policy.

To tackle this issue, we make the execution policy $\pi_t$ greedy to $\Qlazy$, which is updated lazily yet adaptively in \algo. Specifically, for every $(s,a)$, we use $\qdiff(s,a)$ (cf. Line \ref{alg-line:q-update} in Algorithm \ref{alg:aux}) to keep track of the cumulative difference between the current $Q(s,a)$ and $\QM(s,a)$, the latter of which is defined to be the value of $Q(s,a)$ last time $\Qlazy$ is updated immediately after visiting $(s,a)$. Whenever $\qdiff(s,a)$ exceeds $\frac{1}{1-\gamma}$, indicating $\Qlazy(s,a)$ and the execution policy has become outdated with respect to the current $Q(s,a)$, we reset $\qdiff(s,a)$ and set $u^{\mathrm{switch}}$ to {\sf True}, which will direct the algorithm to update the entire function $\Qlazy$ in the following step. \textbf{update-q-lazy}() updates $\Qlazy$ with the samples from $\qset$, which is a dictionary that serves as a buffer to store all the new sample transitions and their latest estimates since the last update of $\Qlazy$.

In contrast, conventional low-switching algorithms update the execution policy on a predetermined, exponentially phased schedule \citep{bai2019provably,zhang2020almost}. While trajectory stability is attained with these algorithms, as time goes on, it takes them exponentially longer to update policy, making them oblivious to recent large updates in the estimated Q-function. This would lead to suboptimal regret in the infinite-horizon setting, as continual choices of suboptimal actions will keep a lasting effect on future trajectory in the absence of trajectory reset. This issue is overcome in our algorithm by ignoring minor changes in function $Q$ yet still being adaptive to substantial changes in any state-action pair.


\section{Main Results}\label{sec:results}

Our model-free algorithm \alg can achieve optimal regret with short burn-in time. Its theoretical guarantee can be summarized in the following theorem.

\begin{theorem}\label{thm:main}
    Choose any $\delta \in (0,1)$. Suppose that $c_b$ is chosen to be a sufficiently large universal constant in Algorithm \ref{alg:main} and let $\iota := \log\frac{SAT}{\delta}$. Then there exists an absolute constant $C_0 > 0$ such that Algorithm \ref{alg:main} achieves
    \begin{equation}\label{eq:thm-result}
        \Regret(T) \leq C_0\left(\sqrt{\frac{SAT\iota^3}{(1-\gamma)^3}} + \frac{SA\iota^{\frac{7}{2}}}{(1-\gamma)^8}\right)
    \end{equation}
    with probability at least $1-\delta$.
\end{theorem}
The proof of Theorem \ref{thm:main} is deferred to Appendix \ref{sec:analysis}, in which we use a recursive error decomposition scheme different from the existing work. The stationary nature of the infinite-horizon setting gives rise to several error terms unique to the infinite-horizon setting, and our novel switching technique is crucial at controlling them optimally (Lemma \ref{lem:xi}). We will present a proof overview for Theorem \ref{thm:main} in Section \ref{sec:proof-overview}. Now let us highlight a few key properties of our algorithm.

\paragraph{Optimal regret with low burn-in.} \alg achieves optimal regret modulo some logarithmic factor as soon as the sample size $T$ exceeds 
	\begin{equation}
		T \ge \frac{SA}{\poly(1-\gamma)}. \label{eq:optimal-sample-range}
	\end{equation}
	This burn-in threshold is significantly lower than the $\frac{S^3A^2}{\poly(1-\gamma)}$ threshold in \citet{zhou2021nearly} when $SA \gg \frac{1}{1-\gamma}$. In other words, in the regime of \eqref{eq:optimal-sample-range}, the regret of \alg is guaranteed to satisfy
	\begin{align}
		\Regret(T) \leq \widetilde{O}\left(\sqrt{\frac{SAT}{(1-\gamma)^3}}\right),
	\end{align}
	which matches the lower bound in Table \ref{fig:table1}.

\paragraph{Sample complexity.} As a corollary of Theorem \ref{thm:main}, it can be seen that \alg attains $\epsilon$-average regret (i.e. $\frac{1}{T}\Regret(T) \le \epsilon$ for any fixed $T$) with sample complexity 
	\begin{equation}
		\widetilde{O}\left(\frac{SA}{(1-\gamma)^3\epsilon^2}\right). \label{eq:optimal-sample-comp}
	\end{equation}
	This is lower than the sample complexity of the model-free algorithm in \citet{liu2020regret}, which is $\widetilde{O}\big(\frac{SA}{(1-\gamma)^5\epsilon^2}\big)$. Moreover, \eqref{eq:optimal-sample-comp} holds true for any desired accuracy $\epsilon\in\big(0,\frac{(1-\gamma)^{13}}{SA}\big]$. This is a broader range than the ones in \citet{zhou2021nearly,zhang2021model}, which involve higher order of $S$ and $A$ and only allow their algorithms to attain their respective optimal sample complexity in the high-precision regime.

\paragraph{Space complexity.} \alg is a model-free algorithm that keeps a few estimates of the Q-function during execution, so its memory cost is as low as $O(SA)$. This is not improvable in the tabular setting, since it requires $O(SA)$ units of space to store the optimal policy. In contrast, the model-based {\sf UCBVI-$\gamma$} in \cite{zhou2021nearly} stores an estimate of the probability transition kernel and thus incurs a higher memory cost of $O(S^2A)$. 

\paragraph{Computational complexity.} The computational cost of \alg is only $O(T)$. This is on the same order as reading samples along the $T$-length executed trajectory and is thus unimprovable. In comparison, our algorithm has a considerably lower computational cost than the one in \citet{zhou2021nearly}, which requires $O(ST)$ operations overall.


\section{Analysis Overview}\label{sec:proof-overview}

In this section, we present an overview for the proof of Theorem 
\ref{thm:main}. Let us first introduce some additional notation that is used in this proof overview.

\subsection{Additional notation}

We let $P_t\in \{0,1\}^{1\times S}$ denote the empirical transition at time step $t$, namely,
\begin{equation}\label{eq:Pt-def}
	P_t(s) = \mathds{1}\left[s = s_{t+1}\right].
\end{equation}
Under this notation, given any value function $V\in [0,1]^S$, $P_t V = V(s_{t+1})$.

In addition, let $f$ and $g$ be two real-valued functions that take $\cX := (S,A,\gamma,T,\frac{1}{\delta})$ as arguments. If there exists a universal constant $C > 0$ such that $f(\cX) \le Cg(\cX)$ for any instantiation of $\cX$, we can denote this with the notation $f(\cX) \lesssim g(\cX)$. In addition, $g(\cX) \gtrsim f(\cX)$ is defined as an equivalent way of writing $f(\cX) \lesssim g(\cX)$. We can write $f(\cX) \asymp g(\cX)$ if and if only both $f(\cX) \lesssim g(\cX)$ and $f(\cX) \gtrsim g(\cX)$ are true.

\subsection{Proof overview for Theorem \ref{thm:main}}

Towards an upper bound on the regret of \algo, we first need to introduce a few lemmas that summarize some crucial properties of the estimates in the algorithm. They serve as important building blocks in our proof of Theorem \ref{thm:main}.

In our algorithm, we keep an optimistic estimate $Q$ (resp. $V$) of the optimal function $Q^\star$ (resp. $V^\star$) to encourage exploration of less-observed state-action pairs. This can be summarized in the following lemma, whose proof can be found in Appendix \ref{sec:proof:Q_t-lower-bound}.

\begin{lemma}\label{lem:Q_t-lower-bound}
    Let $\delta \in (0,1)$. Suppose that $c_b$ is chosen to be a sufficiently large universal constant in Algorithm \ref{alg:main}. With probability at least $1 - \delta$, 
    \begin{equation*}\label{eq:Q-monotonicity}
        Q_t(s,a) \ge Q^\star(s,a) \qquad \text{and} \qquad V_t(s) \ge V^\star(s)
    \end{equation*}
    for all $(t,s,a)\in [T]\times \cS \times \cA$ simultaneously.
\end{lemma}

We can prove a similar lemma for the pessimistic estimate $Q^\LCB$ (resp. $V^\LCB$) of the optimal function $Q^\star$ (resp. $V^\star$). As an implication of the simultaneous usage of optimism and pessimism, we can prove that $V$ and $V^\LCB$ are mostly close to each other throughout the algorithm execution, i.e. \eqref{eq:main-lemma}. This result will be used to characterize the behavior of the reference $V^\rref$ in the remaining of the proof, as the reference $V^\rref$ is controlled based on the size of $V(s_t) - V^\LCB(s_t)$ in our algorithm. The proof of this lemma can be found in Appendix \ref{sec:proof:Q_t-lcb}.

\begin{lemma} \label{lem:Q_t-lcb}
	Let $\delta \in (0,1)$. Suppose that $c_b$ is chosen to be a sufficiently large universal constant in Algorithm \ref{alg:main}. With probability at least $1 - \delta$, 
\begin{equation} 
	\label{eq:Q-lcb-monotonicity}
	Q^{\mathrm{LCB}}_t(s, a) \le Q^{\star}(s, a)
	\qquad \text{and} \qquad
	V^{\mathrm{LCB}}_t(s) \le V^{\star}(s)
\end{equation}
for all $(t,s,a)\in [T]\times \cS \times \cA$ simultaneously, and
\begin{align}
	\label{eq:main-lemma}
	&\sum_{t=1}^T \mathds{1}\left(V_{t}(s_{t+1}) - V_{t}^{\LCB}(s_{t+1})>3\right) \lesssim \frac{\left(SA\right)^{3/4}T^{1/4}}{(1-\gamma)^{9/4}}\left(\log\frac{SAT}{\delta}\right)^{5/4} + \frac{SA}{(1-\gamma)^5}\log \frac{SAT}{\delta} \notag\\ & \qquad \qquad \qquad \qquad \qquad \qquad \qquad \qquad \qquad \qquad \qquad + \sqrt{\frac{SA\log^2 T}{(1-\gamma)^3}\sum_{t=1}^T \big(V_{t-1}(s_{t}) - V^{\pi_{t}}(s_{t})\big)}.
\end{align}
\end{lemma}

The following lemma summarizes the important properties of the reference $V^\rref$ in our algorithm. This precisely reflects the discussion about the variance reduction technique in Section \ref{sec:alg}: the reference $V^\rref$ is kept sufficiently close to our running estimate $V$ at all times, but it changes in a stable manner throughout the execution to avoid volatility. The proof can be found in Appendix \ref{sec:proof:VR_properties}.

\begin{lemma}\label{lem:VR_properties}
Let $\delta \in (0,1)$. Suppose that $c_b$ is chosen to be a sufficiently large universal constant in Algorithm \ref{alg:main}. With probability at least $1 - \delta$, 
 \begin{equation}
	 \label{eq:VR-V-proximity}
	\left| V_{t}(s) - V^{\rref}_{t}(s) \right| \le 6
\end{equation}
for all $(t, s) \in [T] \times \cS$, and
\begin{align}  \label{eq:VR_lazy}
&\quad \ \sum_{t=1}^T \big(V^{\rref}_{t}(s_{t+1}) - V^{\rref}_{T}(s_{t+1})\big)  \nonumber \\
&\leq \frac{S}{1-\gamma} + \sum_{t=1}^T \Big(V_{t}(s_{t+1}) - V^{\mathrm{LCB}}_{t}(s_{t+1}) \Big) \ind\Big(V_{t}(s_{t+1}) - V^{\mathrm{LCB}}_{t}(s_{t+1}) > 3\Big) \\
&\lesssim \frac{\left(SA\right)^{3/4}T^{1/4}}{(1-\gamma)^{13/4}}\left(\log\frac{SAT}{\delta}\right)^{5/4} + \frac{SA}{(1-\gamma)^6}\log \frac{SAT}{\delta} + \sqrt{\frac{SA\log^2 T}{(1-\gamma)^5}\sum_{t=1}^T \big(V_{t-1}(s_{t}) - V^{\pi_{t}}(s_{t})}\big).
\end{align}
\end{lemma}

The preceding lemmas allow us to quantify the closeness of our Q-function estimate $Q_t$ and $Q^\star$ over the execution trajectory. It reflects the ``cumulative'' accuracy of our estimate $Q_t$ over the entire trajectory and is an important step towards the final upper bound on the regret. This can be summarized with the lemma below, whose proof can be found in Appendix \ref{sec:proof:Q_t-upper-bound}.

\begin{lemma}\label{lem:Q_t-upper-bound}
    Fix $\delta \in (0,1)$. Suppose that $c_b$ is chosen to be a sufficiently large universal constant in Algorithm \ref{alg:main}. Then there exists some absolute constant $C_1 > 0$ such that  
    \begin{align}
        &\sum_{t=1}^T \big(Q_t(s_{t},a_{t}) - Q^\star(s_{t},a_{t})\big) \leq \frac{\gamma(3-\gamma)}{2}\sum_{t=1}^T\big(V_t(s_{t+1}) - V^\star(s_{t+1})\big)
        + C_1\Bigg(\sqrt{\frac{SAT}{1-\gamma}\log^{3}\frac{SAT}{\delta}} \notag\\
        &\quad \qquad + \frac{SA}{(1-\gamma)^7}\log^{7/2} \frac{SAT}{\delta} + \sqrt{\frac{SA\log^2 T}{(1-\gamma)^5}\sum_{t=1}^T \big(V_{t-1}(s_{t}) - V^{\pi_{t}}(s_{t})\big)}\Bigg)
    \end{align}
    with probability at least $1 - \delta$.
\end{lemma}

Equipped with all these preceding lemmas, we can decompose the regret as follows 
\begin{align}\label{eq:sketch-regret-decomp1}
    \mathrm{Regret}(T) &\le \sum_{t=1}^T \big(V_{t-1}(s_t) - V^{\pi_t}(s_t)\big) \notag\\
    &\le \sum_{t=1}^T\bigg(\frac{\gamma(3-\gamma)}{2}\big(V^\star(s_{t+1}) - V^{\pi_{t+1}}(s_{t+1})\big) + Q_{t}(s_t, a_t) - Q^\star(s_t, a_t) + \zeta_t\bigg),
\end{align}
in which 
\begin{equation*}
    \zeta_t := \underbrace{V_{t-1}(s_t) - Q_t(s_t,a_t)}_{\zeta_{t,1}} + \underbrace{\gamma\left(P_{s_t,a_t} - P_t\right)\left(V^\star - V^{\pi_t}\right)}_{\zeta_{t,2}} + \underbrace{\gamma\left(V^{\pi_{t+1}}(s_{t+1}) - V^{\pi_t}(s_{t+1})\right)}_{\zeta_{t,3}}.
\end{equation*}

To proceed, we need to find an upper bound on $\sum_{t=1}^T \zeta_t$. We treat the three terms in $\zeta_t$ separately. $\zeta_{t,2}$ is a higher-order noise term that can be bounded with concentration inequalities, in particular, Lemma \ref{lemma:freedman-application2} and the Azuma-Hoeffding inequality (Theorem \ref{thm:hoeffding}). The proof can be found in Appendix \ref{sec:zeta-2-proof}. 

On the other hand, $\zeta_{t,1}$ and $\zeta_{t,3}$ are two types of error terms unique to the infinite-horizon discounted setting. There exists a tradeoff between $\zeta_{t,1}$ and $\zeta_{t,3}$. Specifically, the sum of $\zeta_{t,3}$'s is governed by the total number of policy switches over the $T$ steps, whereas the sum of $\zeta_{t,1}$'s grows with the staleness of the executed action $a_t$ (i.e., the staleness of the execution policy $\pi_t$) at each step. In order to attain the optimal regret, we balance this tradeoff carefully with the use of our adaptive switching technique, which controls both the sum of $\zeta_{t,1}$'s and the sum of $\zeta_{t,3}$'s at $O(\sqrt{\frac{T}{1-\gamma}})$. It is necessary to control the sum of $\zeta_{t,1}$'s and $\zeta_{t,3}$'s at this order of magnitude. Indeed, this is precisely why the existing techniques mentioned in Section \ref{sec:alg:low-switching} that switch the execution policy on a predetermined, exponentially phased schedule \cite{bai2019provably,zhang2020almost} would fail to attain regret optimality in the infinite-horizon setting. These existing techniques would reduce the sum of $\zeta_{t,3}$'s to $O(\frac{\log T}{(1-\gamma)^2})$ because they switch much less frequently than our algorithm; however, this causes the sum of $\zeta_{t,1}$'s to grow beyond the necessary $O(\sqrt{\frac{T}{1-\gamma}})$, which renders the overall regret larger than the optimal $O(\sqrt{\frac{T}{(1-\gamma)^3}})$.

Overall, the upper bound on $\sum_{t=1}^T \zeta_t$ for our algorithm can be summarized in the following lemma, with its proof provided in Appendix \ref{sec:appendix-proof-lem-xi}.

\begin{lemma}\label{lem:xi}
    Fix $\delta \in (0,1)$. Suppose that $c_b$ is chosen to be a sufficiently large universal constant in Algorithm \ref{alg:main}. Then there exists some absolute constant $C_2 > 0$ such that $\zeta_t$ defined in \eqref{eq:zeta-def} satisfies
    \begin{align*}
        \sum_{t=1}^T\zeta_t &\leq C_2\Bigg(\sqrt{\frac{SAT}{1-\gamma}\log\frac{SAT}{\delta}} + \frac{SA\log^2\frac{SAT}{\delta}}{(1-\gamma)^{5/2}} + \sqrt{\frac{SA\log^2 T}{1-\gamma}\sum_{t=1}^T \big(V_{t-1}(s_{t}) - V^{\pi_{t}}(s_{t})\big)}\\ & \qquad \qquad \qquad \qquad \qquad \qquad \qquad \qquad \qquad \qquad + \sqrt{\frac{SA\log\frac{T}{\delta}}{1-\gamma} \sum_{t=1}^{T}\big(V^{\star}(s_{t}) - V^{\pi_{t}}(s_{t})\big)}\Bigg)
    \end{align*}
    with probability at least $1 - \delta$.
\end{lemma}

Finally, we can invoke Lemma \ref{lem:Q_t-upper-bound} on \eqref{eq:sketch-regret-decomp1} and reorganize the terms, which turns the inequality into a recursion about $\sum_{t=1}^T V_{t-1}(s_t) - V^{\pi_t}(s_t)$. Then, we can solve for $\sum_{t=1}^T V_{t-1}(s_t) - V^{\pi_t}(s_t)$ in the inequality to arrive at the final upper bound on $\mathrm{Regret}(T)$.

This concludes the proof overview for Theorem \ref{thm:main}. We defer the actual proof of Theorem \ref{thm:main} to Appendix \ref{sec:analysis}.


\section{Discussion}

This work has introduced a model-free algorithm that achieves optimal regret in infinite-horizon discounted MDPs, which reduces the space and computational complexity requirement for regret optimality in the existing work. It also achieves optimal sample efficiency with a short burn-in time compared to other algorithms, including \citet{zhou2021nearly,zhang2021model}. Moreover, our algorithm has demonstrated the importance of switching policies slowly in infinite-horizon MDPs and introduced a novel technique might be of additional interest to future work. While our burn-in threshold is considerably reduced with respect to the order of $S$ and $A$, it is not yet optimal in the effective horizon $\frac{1}{1-\gamma}$. This gap between our result and the lower bound is open for future work to investigate. 

\section*{Acknowledgements}

The authors are very grateful to Yuxin Chen for helpful advice as well as suggesting the direction. X. Ji was supported in part by the NSF grants 1907661 and 2014279.

\bibliographystyle{apalike}
\bibliography{ref}  
\newpage

\appendix


\section{A Discussion about Different Regret Metrics}\label{sec:appendix-metric-diff}

As we have discussed earlier, multiple metrics exist in the existing literature for evaluating the online performance of an RL algorithm in the infinite-horizon setting. In fact, \cite{zhou2021nearly} and \cite{liu2020regret} both use a respectively different regret metric from the one in this work and \cite{kash2022slowly}. An argument showing the equivalence between the regret metric in \cite{liu2020regret} and the one in \cite{zhou2021nearly} is discussed in Appendix A.2 of \cite{zhou2021nearly}. For this reason and the fact that the regret guarantee in \cite{liu2020regret} is not minimax-optimal, we focus on the relation between our theoretical guarantee in Theorem \ref{thm:main} and the regret metric from \cite{zhou2021nearly} in this section.

Recall that the goal of RL is to learn the optimal policy through online interactions with the environment (and there always exists a stationary optimal policy), so the regret metric we consider in this work is
\begin{equation}\label{eq:regret-again}
    \Regret(T) := \sum_{t=1}^T \Big( V^\star(s_t) - V^{\pi_t}(s_t)\Big),
\end{equation}
where $\pi_t$ is the stationary policy that the algorithm uses to take action $a_t$ at step $t$.

The non-stationary regret metric considered in \cite{zhou2021nearly} is
\begin{equation}\label{eq:regret-other}
    \Regret_{\mathrm{NS}}(T) := \sum_{t=1}^T \Big( V^\star(s_t) - V^{\{\pi_j\}_{j=t}^{\infty}}(s_t)\Big).
\end{equation}
Here, $V^{\{\pi_j\}_{j=t}^{\infty}}(s) := \EE\left[\sum_{i=0}^\infty \gamma^i r(s_i, \pi_{t+i}(s_i)) ~\vert~ s_0=s\right]$, which is the expected cumulative future reward of the non-stationary algorithm starting from time $t$.

In fact, similar difference can also be observed in the literature focused on sample complexity of exploration. Notably, the metric in \cite{zhang2021model} compares the optimal value function $V^\star(s_t)$ against $V^{\pi_t}(s_t)$, while the metric in works such as \cite{strehl06pac,dong2019q} compares against $V^{\{\pi_j\}_{j=t}^{\infty}}(s_t)$. 

There is no formal equivalence between the two regret metrics in general, despite the intuition we have provided in Remark \ref{remark:metric-diff} that $\Regret_{\mathrm{NS}}(T)$ should be smaller than $\Regret(T)$ for any improving algorithm. However, we can show that the specific algorithm \alg (Algorithm \ref{alg:main}) achieves $\widetilde{O}(\sqrt{\frac{SAT}{(1-\gamma)^3}})$ regret under the regret metric defined in \cite{zhou2021nearly}, which matches the lower bound in \cite{zhou2021nearly} and is thus optimal for this metric as well.

Before the analysis, let us first define a notation. Let $f$ and $g$ be two real-valued functions that take $\cX := (S,A,\gamma,T,\frac{1}{\delta})$ as arguments. If there exists a universal constant $C > 0$ such that $f(\cX) \le Cg(\cX)$ for any instantiation of $\cX$, we can denote this with the notation $f(\cX) \lesssim g(\cX)$. 

Let us define
\begin{align}
    \mathcal{T} = \{1 \le t \le T : \pi_t \ne \pi_{t+1}\},
\end{align}
which is the set of time indices that the execution policy switches in the following step, and
\begin{align}  
    \mathcal{T}_H = \{1 \le t \le T : t + h \in \mathcal{T}\text{ for some }0 \le h \le H\},
\end{align}
which is the set of time indices that the execution policy switches in any of the following $H$ steps.

First, notice
\begin{align}
&\quad\ \bigg|\sum_{t \notin \mathcal{T}_H} \Big(V^{\{\pi_j\}_{j=t}^\infty}(s_t) - \sum_{i = 0}^H \gamma^i r(s_{t+i}, a_{t+i})\Big)\bigg| \notag\\
&\le \bigg|\sum_{t \notin \mathcal{T}_H} \Big(\mathbb{E}\Big[\sum_{i = 0}^H \gamma^i r(s_{t+i}, a_{t+i})\Big] - \sum_{i = 0}^H \gamma^i r(s_{t+i}, a_{t+i})\Big)\bigg| + \frac{\gamma^HT}{1 - \gamma} \notag\\
&\le \sum_{k = 1}^{H}\bigg|\sum_{t = jH + k \notin \mathcal{T}_H} \Big(\mathbb{E}\Big[\sum_{i = 0}^H \gamma^i r(s_{t+i}, a_{t+i})\Big] - \sum_{i = 0}^H \gamma^i r(s_{t+i}, a_{t+i})\Big)\bigg| + \frac{1}{T} \notag\\
&\lesssim \sqrt{\frac{T\log^3 T}{(1-\gamma)^3}}, \label{eq:regret-diff-1}
\end{align}
where the second line holds when $H \gtrsim \frac{\log T}{1-\gamma}$, and the last line makes use of the Azuma-Hoeffding inequality.

Similarly, we have
\begin{align}
\bigg|\sum_{t \notin \mathcal{T}_H} \Big(V^{\pi_t}(s_t) - \sum_{i = 0}^H \gamma^i r(s_{t+i}, a_{t+i})\Big)\bigg| 
\lesssim \sqrt{\frac{T\log^3 T}{(1-\gamma)^3}} \label{eq:regret-diff-2}
\end{align} 
since $\pi_t = \pi_{t+i}$ for $0 \le i \le H$ and $t \notin \mathcal{T}_H$. 

Putting \eqref{eq:regret-diff-1} and \eqref{eq:regret-diff-2} together leads to
\begin{align}
\big|\Regret(T) - \Regret_{\mathrm{NS}}(T)\big| 
&= \left|\sum_{t=1}^T \big(V^{\pi_t}_t(s_t) - V^{\{\pi_j\}_{j=t}^\infty}(s_t)\big)\right| \notag\\
&\lesssim \sqrt{\frac{T\log^3 T}{(1-\gamma)^3}} + \frac{|\mathcal{T}|\log T}{(1-\gamma)^2} \label{eq:regret-diff-3}
\end{align}
by noticing that $\Big|\sum_{t \in \mathcal{T}_H} \big(V^{\pi_t}(s_t) - V^{\{\pi_j\}_{j=t}^\infty}(s_t)\big)\Big| \lesssim \frac{|\mathcal{T}_H|}{1-\gamma} \lesssim \frac{|\mathcal{T}|\log T}{(1-\gamma)^2}$.

Since our algorithm is low-switching and $|\mathcal{T}| \le \widetilde{O}(\frac{SA}{(1-\gamma)^{9/2}}+\frac{(SA)^{3/4}T^{1/4}}{(1-\gamma)^{5/4}})$, the difference between the two regret metrics \eqref{eq:regret-diff-3} is dominated by the upper bound on the regret itself. Thus, given our result from Theorem \ref{thm:main} that $\Regret(T) \le \widetilde{O}(\sqrt{\frac{SAT}{(1-\gamma)^3}})$, we can conclude $\Regret_{\mathrm{NS}}(T) \le \widetilde{O}(\sqrt{\frac{SAT}{(1-\gamma)^3}})$ for our algorithm \algo. In fact, this conversion holds as long as the algorithm's switching cost is dominated by the regret itself, e.g., when the switching cost is $o(\sqrt{T})$.

\section{An Alternative Presentation of Algorithm \ref{alg:main}}\label{sec:appendix-alg-t}

We present a rewritten version of Algorithm \ref{alg:main} in Algorithm \ref{alg:main-t}, which specifies the time step index of all iterates for easier identification in later analysis. 

\begin{algorithm}[t]
    \caption{\alg (a rewrite of Algorithm \ref{alg:main} that specifies dependency on $t$) \label{alg:main-t}}
        \textbf{Initialize:} For all $(s,a)$, $Q_1(s,a), Q^{\mathrm{UCB}}_1(s,a), Q^{\mathrm{R}}_1(s,a) \leftarrow \frac{1}{1-\gamma}$; $Q^{\mathrm{LCB}}_1(s,a) \leftarrow 0$; $N_0(s,a) \leftarrow 0$; $V_1(s) \leftarrow V^{\mathrm{R}}_1(s) = \frac{1}{1-\gamma}$; $\mu^{\mathrm{ref}}_1(s,a), \sigma^{\mathrm{ref}}_1(s,a), \mu^{\mathrm{adv}}_1(s,a), \sigma^{\mathrm{adv}}_1(s,a), \delta^{\mathrm{R}}_1(s,a), B^{\mathrm{R}}_1(s,a) \leftarrow 0$, $\qdiff_1(s,a) \leftarrow 0$; $M_1(s,a) \leftarrow 1$; $u^{\mathrm{ref}}_1(s) \leftarrow \mathrm{True}$; $Z_0, Z_1 \leftarrow 1$; $H = \lceil\frac{2}{1-\gamma} \rceil$; $\iota = \log\frac{SAT}{\delta}$.
        
		\For{$t = 1, \cdots, T$}{
            Take action $a_t = \pi_{t}(s_t) = \arg \max_a Q_{Z_{t-1}}(s_t, a)$, and draw $s_{t+1} \sim P(\cdot|s_t, a_t)$. \label{alg-t-line:a}\\
		    $N_{t}(s_t, a_t) \leftarrow N_{t-1}(s_t, a_t) + 1$; $N_{t}(s, a) \leftarrow N_{t-1}(s, a)$ for all $(s,a) \neq (s_t, a_t)$; $n \leftarrow N_{t}(s_t, a_t)$.\\
		    $\eta_n \leftarrow \frac{H+1}{H+n}$.\\
		    $Q^{\mathrm{UCB}}_{t+1}(s_t, a_t) \leftarrow \textbf{update-ucb-q}()$; $Q^{\mathrm{UCB}}_{t+1}(s, a) \leftarrow Q^{\mathrm{UCB}}_{t}(s, a)$ for all $(s,a) \neq (s_t, a_t)$.\\
		    $Q^{\mathrm{LCB}}_{t+1}(s_t, a_t) \leftarrow \textbf{update-lcb-q}()$; $Q^{\mathrm{LCB}}_{t+1}(s, a) \leftarrow Q^{\mathrm{LCB}}_{t}(s, a)$ for all $(s,a) \neq (s_t, a_t)$.\\
		    $Q^{\mathrm{R}}_{t+1}(s_t, a_t) \leftarrow \textbf{update-ucb-q-advantage}()$; $Q^{\mathrm{R}}_{t+1}(s, a) \leftarrow Q^{\mathrm{R}}_{t}(s, a)$ for all $(s,a) \neq (s_t, a_t)$.\\
		    $Q_{t+1}(s_t, a_t) \leftarrow \min \{Q^{\mathrm{R}}_{t+1}(s_t, a_t), Q^{\mathrm{UCB}}_{t+1}(s_t, a_t), Q_t(s_t, a_t)\}$; $Q_{t+1}(s, a) \leftarrow Q_{t}(s, a)$ for all $(s,a) \neq (s_t, a_t)$. \label{alg-t-line:Q}\\
		    $V_{t+1}(s_t) \leftarrow \max_a Q_{t+1}(s_t, a)$; $V_{t+1}(s) \leftarrow V_{t}(s)$ for all $s \neq s_t$.\\
		    $V^{\mathrm{LCB}}_{t+1}(s_t) \leftarrow \max\{\max_a Q^{\mathrm{LCB}}_{t+1}(s_t, a), V^{\mathrm{LCB}}_{t}(s_t)\}$; $V^{\mathrm{LCB}}_{t+1}(s) \leftarrow V^{\mathrm{LCB}}_{t}(s)$ for all $s \neq s_t$. \label{alg-t-line:VLCB}

            \eIf{$\qdiff_{t}(s_t,a_t) + Q_{M_{t}(s_t, a_t)}(s_t, a_t) - Q_{t+1}(s_t, a_t) > \frac{1}{1-\gamma}$ \label{alg-t-line:q-cond}}{
                $M_{t+1}(s_t, a_t) \leftarrow t+1$;\\
                $Z_{t+1} \leftarrow t+1$; \label{alg-t-line:Z-switch}\\
		        $\qdiff_{t+1}(s_t,a_t) \leftarrow 0$; $\qdiff_{t+1}(s, a) \leftarrow \qdiff_{t}(s, a)$ for all $(s,a) \neq (s_t, a_t)$;
            }{
                $Z_{t+1} \leftarrow Z_t$;\\
                $M_{t+1}(s_t, a_t) \leftarrow M_{t}(s_t, a_t)$;\\
                $\qdiff_{t+1}(s_t,a_t) \leftarrow \qdiff_{t}(s_t,a_t) + Q_{M_{t}}(s_t, a_t) - Q_{t+1}(s_t, a_t)$; $\qdiff_{t+1}(s, a) \leftarrow \qdiff_{t}(s, a)$ for all $(s,a) \neq (s_t, a_t)$.
            }
            
		    \uIf{$V_{t+1}(s_{t}) - V^{\mathrm{LCB}}_{t+1}(s_{t}) > 3$ \label{alg-t-line:if-cond-1}}{
                $V^{\mathrm{R}}_{t+1}(s_{t}) \leftarrow V_{t+1}(s_{t})$, $u^{\mathrm{ref}}_{t+1}(s_{t}) = \mathrm{True}$; \label{alg-t-line:if-stat-1}
            }
            \ElseIf{$u^{\mathrm{ref}}_{t}(s_{t}) = \mathrm{True}$ \label{alg-t-line:elif-cond-1}}{
                $V^{\mathrm{R}}_{t+1}(s_{t}) \leftarrow V_{t+1}(s_{t})$, $u^{\mathrm{ref}}_{t+1}(s_{t}) = \mathrm{False}$. \label{alg-t-line:elif-stat-1}
            }
        }
\end{algorithm}

\section{Supporting Lemmas}

\subsection{Additional Notation}

For any vector $V\in\RR^n$, we denote its $\ell_\infty$-norm with $\norm{V}_{\infty}$, with $\norm{V}_{\infty} := \max_{i\in[n]} \left|V_i\right|$. We also denote its $\ell_1$-norm with $\norm{V}_{1}$, with $\norm{V}_{1} := \sum_{i=1}^n \left|V_i\right|$. Similarly, we can overload this notation for matrices: for any matrix $P\in\RR^{m\times n}$, we denote its $\ell_1$-norm with $\norm{P}_{1}$, with $\norm{P}_{1} := \max_{i\in[m]}\sum_{i=1}^n \left|P_i\right|$.

In addition, for any given vector $V \in \mathbb{R}^{S}$, let us define the variance parameter with respect to $P_{s,a}$ as follows:
\begin{equation} \label{eq:var-def}
	\Var_{s, a}(V) := \mathbb{E}_{s' \sim P_{s,a}} \left [\left(V(s') -  P_{s, a}V \right)^2\right] 
	= P_{s,a} \left(V^{ 2}\right) - \left(P_{s,a}V \right)^2.
\end{equation}

In our analysis, we let $[\mu_t^\re, \sigma_t^\re, \mu_t^\adv, \sigma_t^\adv, \delta_t^\rref, B_t^\rref, b_t^\rref, B_t^\mathrm{next}]$ denote the value of $[\mu^\re, \sigma^\re, \mu^\adv, \sigma^\adv, \delta^\rref, B^\rref, b^\rref, B^\mathrm{next}]$ in Algorithm \ref{alg:aux} at the beginning of Step $t$. This is similar to the time indices in Algorithm \ref{alg:main-t}.

Occasionally, we let the notation $t_n(s,a)$ denote the step index when $(s,a)$ is visited for the $n$-th time, i.e., when $N_t(s_t,a_t) = N_t(s,a) = n$. In addition, $t_n(s,a)$ might be abbreviated to $t_n$ when $(s,a)$ is clear from the context. This will be specified again in the later analysis when such notation is used. 

\subsection{Learning Rate Lemma and Its Proof}

Let $\eta_n$ be as defined in Algorithm \ref{alg:main-t}. For any integers $N\ge 0$ and $n\ge 1$, define
    \begin{equation}\label{eq:learning-rate-def2}
        \eta_0^N := \begin{cases}\prod_{i=1}^N(1-\eta_i)=0, \quad &\text{if }N>0;\\ 1, \quad &\text{if }N=0\end{cases}\quad \text{and}\quad \eta_n^N := \begin{cases}\eta_n\prod_{i=n+1}^N(1-\eta_i), \quad &\text{if }N>n;\\ \eta_n, \quad &\text{if }N=n;\\ 0, \quad &\text{if }N<n.\end{cases}
    \end{equation}

\begin{lemma}\label{lem:learning-rate}
    Let $\eta_n^N$ be as defined in \eqref{eq:learning-rate-def2}. One has
    \begin{equation}
        \sum_{n=1}^N\eta_n^N := \begin{cases}1, \quad &\text{if }N>0;\\ 0, \quad &\text{if }N=0.\end{cases} \label{eq:learning-rate1}
    \end{equation}
    Moreover, for any integer $N > 0$, the following properties are true:
    \begin{equation}
        \frac{1}{N^a} \le \sum_{n=1}^N\frac{\eta_n^N}{n^a} \le \frac{2}{N^a}, \quad \text{for all } \frac{1}{2} \le a \le 1, \label{eq:learning-rate2}
    \end{equation}
    and 
    \begin{align}
        \max_{1\le n \le N}\eta_n^N &\le \frac{6}{(1-\gamma)N}\label{eq:learning-rate3}\\
        \sum_{n=1}^N\left(\eta_n^N\right)^2 &\le \frac{6}{(1-\gamma)N}. \label{eq:learning-rate4}
    \end{align}
    In addition, for any integer $n \ge 0$, one has
    \begin{equation}
        \sum_{N=n}^\infty\eta_n^N \le 1 + \frac{1-\gamma}{2}. \label{eq:learning-rate5}
    \end{equation}
\end{lemma}

\begin{proof}[Proof of Lemma \ref{lem:learning-rate}]
    We can directly invoke Lemma 1 of \cite{Li-finite} with $H = \lceil\frac{2}{1-\gamma}\rceil$. Note that the proofs of \eqref{eq:learning-rate1} and \eqref{eq:learning-rate2} do not depend on $H$, so (20) and (21a) from \cite{Li-finite} are directly applicable. 
    
    (21b) from \cite{Li-finite} gives $\max_{1\le n \le N}\eta_n^N \le \frac{2H}{N}$ and $\sum_{n=1}^N\left(\eta_n^N\right)^2 \le \frac{2H}{N}$. \eqref{eq:learning-rate3} and \eqref{eq:learning-rate4} can be obtained by the fact $2\lceil\frac{2}{1-\gamma}\rceil \le 2\left(\frac{2}{1-\gamma}+1\right) \le \frac{6}{1-\gamma}$, where the last step is obtained by $2\le \frac{2}{1-\gamma}$. 
    
    Lastly, from (21b) in \cite{Li-finite}, we can write $\sum_{N=n}^\infty\eta_n^N \le 1 + \frac{1}{\lceil\frac{1-\gamma}{2}\rceil}$, which immediately leads to \eqref{eq:learning-rate5}.
\end{proof}

\subsection{Proposition \ref{prop:algebraic} and Its Proof} \label{sec:appendix-proof-prop:algebraic}

\begin{proposition}\label{prop:algebraic}
    For $x,y,z \in \RR$, if
    \begin{equation}
        z^2 \le xz + y, \label{eq:algebraic-lem-cond}
    \end{equation}
    this implies
    \begin{equation}
        z^2 \le x^2 + 2y. \label{eq:algebraic-lem-stat}
    \end{equation}
\end{proposition}

\begin{proof}
Beginning with \eqref{eq:algebraic-lem-cond}, we can write
\begin{align}
    z^2 \le xz + y &\iff z^2 - xz \le y \notag\\
    &\iff z^2 - xz + \frac{1}{4}x^2 \le y + \frac{1}{4}x^2 \notag\\
    &\iff \left(z - \frac{x}{2}\right)^2 \le y + \frac{1}{4}x^2\label{eq:algebraic1}
\end{align}

On the other hand, 
\begin{align}
    z^2 = (z - \frac{x}{2} + \frac{x}{2})^2 \overset{(\mathrm{i})}{\le} 2(z - \frac{x}{2})^2 + 2\left(\frac{x}{2}\right)^2 \overset{(\mathrm{ii})}{\le} 2y + \frac{1}{2}x^2 + 2\cdot\frac{x^2}{4} = 2y + x^2,
\end{align}
in which (i) is by the algebraic facts $2ab \le a^2 + b^2$ and $(a+b)^2 \le 2a^2 + 2b^2$; (ii) is by \eqref{eq:algebraic1}.    
\end{proof}

\subsection{Concentration Lemmas}

Since we cope with quantities with Markovian nature in RL, let us introduce two established results about the concentration of martingale sequences. The first one is the Azuma-Hoeffding inequality \cite{azuma67,vershynin2018high}, which can be stated as follows:

\begin{theorem}[Azuma-Hoeffding inequality]\label{thm:hoeffding}
	Consider a filtration $\cF_0\subset \cF_1 \subset \cF_2 \subset \cdots$,
	and let $\mathbb{E}_{i}$ stand for the expectation conditioned
	on $\mathcal{F}_i$. Suppose that $Y_{n}=\sum_{i=1}^{n}X_{i}\in\mathbb{R}$,
	where $\{X_{i}\}$ is a real-valued scalar sequence obeying
    $$\left|X_{i}\right|\leq R\qquad\text{and}\qquad\mathbb{E}_{i-1} \big[X_{i}\big]=0\quad\quad\quad\text{for all }i\ge 1$$
	for some $R<\infty$. 
	Then with probability at least $1-\delta$,
	\begin{equation}
		\left|Y_{n}\right|\leq \sqrt{R^2n\log\frac{1}{\delta}}.\label{eq:hoeffding}
	\end{equation}
\end{theorem}

The other concentration result is Freedman's inequality \cite{freedman75,tropp11}, which characterizes the concentration in a more specific way with variance. We state a version of it taken from Theorem 3 of \cite{Li-finite-arxiv}.

\begin{theorem}[Freedman's inequality]\label{thm:Freedman}
	Consider a filtration $\cF_0\subset \cF_1 \subset \cF_2 \subset \cdots$,
	and let $\mathbb{E}_{i}$ stand for the expectation conditioned
	on $\mathcal{F}_i$. Suppose that $Y_{n}=\sum_{i=1}^{n}X_{i}\in\mathbb{R}$,
	where $\{X_{i}\}$ is a real-valued scalar sequence obeying
    $$\left|X_{i}\right|\leq R\qquad\text{and}\qquad\mathbb{E}_{i-1} \big[X_{i}\big]=0\quad\quad\quad\text{for all }i\ge 1$$
	for some $R<\infty$. Moreover, define
    $$W_{n}:=\sum_{i=1}^{n}\mathbb{E}_{i-1}\left[X_{i}^{2}\right],$$
	and suppose that $W_{n}\leq\sigma^{2}$ holds deterministically for some given $\sigma^2<\infty$.
	Then for any positive integer $m \geq1$, with probability at least $1-\delta$,
	\begin{equation}
		\left|Y_{n}\right|\leq\sqrt{8\max\Big\{ W_{n},\frac{\sigma^{2}}{2^{m}}\Big\}\log\frac{2m}{\delta}}+\frac{4}{3}R\log\frac{2m}{\delta}.\label{eq:Freedman}
	\end{equation}
\end{theorem}

In the following, we introduce two more concentration results based on Freedman's inequality to handle some quantities of specific types during the execution of our algorithm. Our concentration results are adapted from Lemma 7 and 8 in \citet{Li-finite-arxiv}, which handle the finite-horizon setting.

\begin{lemma}
	\label{lemma:freedman-application}
	For any $(s,a,N)\in \cS\times \cA\times [T]$, let $\left\{ W_{i} \in \mathbb{R}^S \mid 1\leq i\leq T \right\}$ and $\left\{u_i(s,a,N)\in \mathbb{R} \mid 1\leq i\leq T\right\}$ be a collection of vectors and scalars, respectively, and suppose that both obey the following properties:
	\begin{itemize}
		\item For any $i \in [T]$, $W_{i}$ is fully determined by the samples collected up to the end of the $(i-1)$-th step;  
		\item For any $i \in [T]$, $\norm{W_i}_{\infty}\leq C_\mathrm{w}$ and $W_i \geq 0$ deterministically; 
		\item For any $i \in [T]$, $u_i(s, a,N)$ is fully determined by the samples collected up to the end of the $(i-1)$-th step and a given positive integer $N\in [T]$; 
		\item For any $i \in [T]$, $0\leq u_i(s, a, N) \leq C_{\mathrm{u}}$ deterministically;
		\item Given any $N_t(s,a) \in [T]$ and sequence $\{t_n(s,a) \in [T] \mid 1 \le n \le N_t(s,a)\}$, $0\leq \sum_{n=1}^{N_t(s,a)} u_{t_n(s,a)}(s,a, N) \leq 2$ deterministically.  
	\end{itemize}
	In addition, define the sequence  
	\begin{align}
		X_i (s,a,N) &:= u_i(s,a, N) \left(P_{i} - P_{s,a}\right) W_{i} 
		\ind\left\{ (s_i, a_i) = (s,a)\right\},
		\qquad \text{for}\ 1\leq i\leq T.
	\end{align}
	Let $\delta \in (0,1)$. With probability at least $1-\delta$, 
	\begin{align}
		&\quad\ \left|\sum_{i=1}^t X_i(s,a,N) \right| \notag\\
		& \lesssim \sqrt{C_{\mathrm{u}}\log^{2}\frac{SAT}{\delta}}\sqrt{\sum_{n=1}^{N_{t}(s,a)}u_{t_{n}(s,a)}(s,a,N)\Var_{s,a}\left(W_{t_{n}(s,a)}\right)}+\left(C_{\mathrm{u}}C_{\mathrm{w}}+\sqrt{\frac{C_{\mathrm{u}}}{N}}C_{\mathrm{w}}\right)\log^{2}\frac{SAT}{\delta}
	\end{align}
	for all $(t, s, a, N) \in [T] \times \cS \times \cA \times [T]$ simultaneously.  
\end{lemma}

\begin{proof}[Proof of Lemma \ref{lemma:freedman-application}]
    In this proof, we apply Freedman's inequality (Theorem \ref{thm:Freedman}) for an upper bound on $\sum_{i=1}^t X_i(s,a,N)$. We begin by verifying whether the conditions of Freedman's inequality are satisfied.

    First, we verify that $X_i(s,a,N)$ is deterministically bounded:
    \begin{align}
		\left| X_i(s,a,N) \right| &\le u_i(s,a,N) \left|\left( P_{i} - P_{s,a} \right) W_{i} \right| \leq u_i(s,a,N) \left(\norm{P_{i}}_1 + \norm{P_{s,a}}_1 \right) \norm{W_i}_{\infty}  
		\leq 2C_\mathrm{w} C_\mathrm{u}, \label{eq:freedman-app1-1}
	\end{align}
    in which we use the assumptions $\norm{W_i}_{\infty} \leq C_{\mathrm{w}}$ and $0 \le u_i(s,a,N) \le C_{\mathrm{u}}$ as well as $\norm{P_{i}}_1 = \norm{P_{s,a}}_1=1$. 

    Then, we can also verify that 
    \begin{align}
        \EE_{i-1}\left[ X_i(s,a,N) \right] = u_i(s,a,N) \EE_{i-1}\left[\left( P_{i} - P_{s,a} \right)\ind\left\{ (s_i, a_i) = (s,a)\right\}\right] W_{i} = 0,
    \end{align}
    because $u_i(s, a,N)$ and $W_i$ are deterministic when conditioned on the $(i-1)$-th step.

    With the conditions verified, we give a bound on the second moment of $X_i(s,a,N)$ before invoking Freedman's inequality:
    \begin{align}
		\sum_{i=1}^{t}\EE_{i-1}\left[\left|X_i(s,a,N)\right|^{2}\right] &= \sum_{i=1}^{t}\left(u_{i}(s,a, N)\right)^{2} \mathds{1}\left\{(s_{i},a_{i})=(s,a)\right\}
		\EE_{i-1}\left[\left|(P_{i}-P_{s,a})W_{i}\right|^{2}\right] \notag\\
		&= \sum_{n=1}^{N_{t}(s,a)}\left(u_{t_{n}(s,a)}(s,a,N)\right)^{2}\Var_{s,a}\left(W_{t_{n}(s,a)}\right) \notag\\
		&\le C_{\mathrm{u}}\left(\sum_{n=1}^{N_{t}(s,a)}u_{t_{n}(s,a)}(s,a,N)\right)\Var_{s,a}\left(W_{t_{n}(s,a)}\right) \label{eq:freedman-app1-2}\\
		&\le C_{\mathrm{u}}\left(\sum_{n=1}^{N_{t}(s,a)}u_{t_{n}(s,a)}(s,a,N)\right)\norm{W_{t_{n}(s,a)}}_{\infty}^{2}\notag\\
		&\le 2C_{\mathrm{u}}C_{\mathrm{w}}^{2}, \label{eq:freedman-app1-3}
	\end{align}
	in which we use the assumptions $\norm{W_i}_{\infty} \leq C_{\mathrm{w}}$ and $0 \le u_i(s,a,N) \le C_{\mathrm{u}}$ as well as $0\leq \sum_{n=1}^{N_t(s,a)} u_{t_n(s,a)}(s,a,N) \leq 2$.

    Finally, with \eqref{eq:freedman-app1-1}, \eqref{eq:freedman-app1-2} and \eqref{eq:freedman-app1-3} established, we can invoke Freedman's inequality (Theorem \ref{thm:Freedman}) with $m=\lceil \log_2 N \rceil$ and take the union bound over all possible $(t, s, a, N) \in [T] \times \cS \times \cA \times [T]$. This can lead to the desired result: with probability at least $1- \delta$, 
	\begin{align}
		&\quad \ \left|\sum_{i=1}^{t}X_{i}(s,a,N)\right| \notag\\
        &\lesssim \sqrt{\max\left\{ C_{\mathrm{u}}\sum_{n=1}^{N_{t}(s,a)}u_{t_{n}(s,a)}(s,a,N)\Var_{s,a}\left(W_{t_{n}(s,a)}\right),\frac{C_{\mathrm{u}}C_{\mathrm{w}}^{2}}{N}\right\}\log\frac{SAT^{2}\log N}{\delta}} + C_{\mathrm{u}}C_{\mathrm{w}}\log\frac{SAT^{2}\log N_t}{\delta} \notag\\
		&\lesssim \sqrt{C_{\mathrm{u}}\log^{2}\frac{SAT}{\delta}}\sqrt{\sum_{n=1}^{N_{t}(s,a)}u_{t_{n}(s,a)}(s,a,N)\Var_{s,a}\left(W_{t_{n}(s,a)}\right)}+\left(C_{\mathrm{u}}C_{\mathrm{w}} + \sqrt{\frac{C_{\mathrm{u}}}{N}}C_{\mathrm{w}}\right)\log^{2}\frac{SAT}{\delta}\notag
	\end{align}
	for all $(t, s, a, N) \in [T] \times \cS \times \cA \times [T]$ simultaneously. 
\end{proof}

\begin{lemma}\label{lemma:freedman-application2}
    Let $\big\{ N(s,a) \in [T] \mid (s,a)\in \cS\times \cA\big\}$ be a collection of positive integers. For any $(s,a)\in \cS\times \cA$, let $\left\{ W_{i} \in \mathbb{R}^S \mid 1\leq i\leq T \right\}$ and $\left\{u_i(s_i,a_i)\in \mathbb{R} \mid 1\leq i\leq T\right\}$ be a collection of vectors and scalars, respectively, and suppose that both obey the following properties:
	\begin{itemize}
		\item For any $i \in [T]$, $W_{i}$ is fully determined by the samples collected up to the end of the $(i-1)$-th step;  
		\item For any $i \in [T]$, $\norm{W_i}_{\infty}\leq C_\mathrm{w}$ and $W_i \geq 0$ deterministically; 
		\item For any $i \in [T]$, $u_i(s_i, a_i)$ is fully determined by the integer $N(s_i, a_i)$ and all samples collected up to the end of the $(i-1)$-th step; 
		\item For any $i \in [T]$, $0\leq u_i(s_i, a_i) \leq C_{\mathrm{u}}$ deterministically.
	\end{itemize}
	In addition, define the sequences  
	\begin{align}
		X_{i} &:= u_i(s_i,a_i)\big(P_{i} - P_{s_i,a_i}\big) W_{i},
		\qquad &\text{for}\ 1\leq i\leq T,\\
		Y_{i} &:= \big(P_{i} - P_{s_i,a_i}\big) W_{i},\qquad &\text{for}\ 1\leq i\leq T.
	\end{align}
	Let $\delta \in (0,1)$. With probability at least $1-\delta$, 
	\begin{align}
		\left|\sum_{i=1}^{T}X_{i}\right| & \lesssim\sqrt{C_{\mathrm{u}}^{2}\sum_{i=1}^{T}\EE_{i-1}\left[\left|(P_{i}-P_{s_{i},a_{i}})W_{i}\right|^{2}\right]\log\frac{T^{SA}}{\delta}}+C_{\mathrm{u}}C_{\mathrm{w}}\log\frac{T^{SA}}{\delta} \notag\\
		&\lesssim \sqrt{C_{\mathrm{u}}^{2}C_{\mathrm{w}}\sum_{i=1}^{T}\EE_{i-1}\left[P_{i}W_{i}\right]\log\frac{T^{SA}}{\delta}}+C_{\mathrm{u}}C_{\mathrm{w}}\log\frac{T^{SA}}{\delta} \notag\\
		\left|\sum_{i=1}^{T}Y_{i}\right| &\lesssim \sqrt{TC_{\mathrm{w}}^{2}\log\frac{1}{\delta}}+C_{\mathrm{w}}\log\frac{1}{\delta}\notag
	\end{align}
	for all possible collections $\{N(s,a) \in [T] \mid (s,a) \in \cS\times \cA\}$ simultaneously. 
\end{lemma}

\begin{proof}[Proof of Lemma \ref{lemma:freedman-application2}]
    In this proof, we apply Freedman's inequality (Theorem \ref{thm:Freedman}) for an upper bound on $\sum_{i=1}^T X_i$ and $\sum_{i=1}^T Y_i$ respectively. 
    
    We begin by focusing on $X_i$ and verifying whether the conditions of Freedman's inequality are satisfied. We keep $\left\{N(s,a) \in [T] \mid (s,a)\in \cS\times \cA\right\}$ fixed when checking the conditions. First, we verify that $X_i$ is deterministically bounded:
    \begin{align}
		\left| X_i \right| &\le u_i(s_i,a_i) \left|\left( P_{i} - P_{s_i,a_i} \right) W_{i} \right| \leq u_i(s_i,a_i) \left(\norm{P_{i}}_1 + \norm{P_{s_i,a_i}}_1 \right) \norm{W_i}_{\infty}  
		\leq 2C_\mathrm{w} C_\mathrm{u}, \label{eq:freedman-app2-1}
	\end{align}
    in which we use the assumptions $\norm{W_i}_{\infty} \leq C_{\mathrm{w}}$ and $0 \le u_i(s_i,a_i) \le C_{\mathrm{u}}$ as well as $\norm{P_{i}}_1 = \norm{P_{s_i,a_i}}_1=1$. 

    Then, since $u_i(s_i,a_i)$ is deterministic when conditioned on the $(i-1)$-th step and $N(s_i,a_i)$, we can also verify that
    \begin{align}
        \EE_{i-1}\left[ X_i \right] = u_i(s_i,a_i) \EE_{i-1}\left[\left( P_{i} - P_{s_i,a_i} \right)\right]W_i = 0,
    \end{align}
    because $W_i$ is deterministic when conditioned on the $(i-1)$-th step. Note that $a_i$ is determined by $Q_{j}$ for some $j < i$, which is deterministic when conditioned on the $(i-1)$-th step, so $\EE_{i-1}\left[P_{i} - P_{s_i,a_i}\right] = 0$.

    With the conditions verified, we give a bound on the second moment of $X_i$ before invoking Freedman's inequality:
    \begin{align}
		\sum_{i=1}^{T}\EE_{i-1}\left[\left|X_i\right|^{2}\right] &= \sum_{i=1}^{T}\left(u_{i}(s_i,a_i)\right)^{2} \EE_{i-1}\left[\left|(P_{i}-P_{s,a})W_{i}\right|^{2}\right] \notag\\
		& \overset{(\mathrm{i})}{\le} C_{\mathrm{u}}^{2}\sum_{i=1}^{T}\EE_{i-1}\left[\left|(P_{i}-P_{s_{i},a_{i}})W_{i}\right|^{2}\right] \label{eq:freedman-app2-3}\\
			& \le C_{\mathrm{u}}^{2}\sum_{i=1}^{T}\EE_{i-1}\left[\left|P_{i}W_{i}\right|^{2}\right]\notag\\
			& \overset{(\mathrm{ii})}{=} C_{\mathrm{u}}^{2}\sum_{i=1}^{T}\EE_{i-1}\left[P_{i}(W_{i})^{2}\right]\notag\\
			& \overset{(\mathrm{iii})}{\le}C_{\mathrm{u}}^{2}\sum_{i=1}^{T}\norm{W_i}_{\infty}\EE_{i-1}\left[P_{i}W_{i}\right] \notag\\
			& \overset{(\mathrm{iv})}{\le}C_{\mathrm{u}}^{2}C_{\mathrm{w}}\sum_{i=1}^{T}\EE_{i-1}\left[P_{i}W_{i}\right] \label{eq:freedman-app2-4}\\
			& \leq C_{\mathrm{u}}^{2}C_{\mathrm{w}}\sum_{i=1}^{T} \norm{W_i}_{\infty} \notag\\
			& \overset{\mathrm{(v)}}{\le} TC_{\mathrm{u}}^{2}C_{\mathrm{w}}^2. \label{eq:freedman-app2-5}
	\end{align}
    In the series of equalities and inequalities above, (i) is due to the assumption $0 \leq u_i(s_i,a_i) \leq C_{\mathrm{u}}$. (ii) is due to $P_i$ is a canonical basis vector. (iii) is due to the non-negativity of $W_i$. (iv) and (v) are obtained using $\norm{W_i}_{\infty} \le C_{\mathrm{w}}$.

    Finally, with \eqref{eq:freedman-app2-1}, \eqref{eq:freedman-app2-4} and \eqref{eq:freedman-app2-5} established, we can invoke Freedman's inequality (Theorem \ref{thm:Freedman}) with $m=\lceil \log_2 N \rceil$ and take the union bound over all possible collections $\{N(s,a) \in [T] \mid (s,a)\in \cS\times \cA\}$, which has at most $T^{SA}$ possibilities. This can lead to the desired result: with probability at least $1- \delta$, 
    \begin{align*}
		\left|\sum_{i=1}^{T}X_{i}\right| & \lesssim \sqrt{\max\left\{ C_{\mathrm{u}}^{2}\sum_{i=1}^{T}\EE_{i-1}\left[ \left|(P_{i}-P_{s_{i},a_{i}})W_{i}\right|^{2}\right],\frac{TC_{\mathrm{u}}^{2}C_{\mathrm{w}}^{2}}{2^{m}}\right\} \log\frac{T^{SA}\log T}{\delta}} + C_{\mathrm{u}}C_{\mathrm{w}}\log\frac{T^{SA}\log T}{\delta}\\
		&\lesssim \sqrt{C_{\mathrm{u}}^{2}\sum_{i=1}^{T}\EE_{i-1}\left[\left|(P_{i}-P_{s_{i},a_{i}})W_{i}\right|^{2}\right]\log\frac{T^{SA}}{\delta}} + C_{\mathrm{u}}C_{\mathrm{w}}\log\frac{T^{SA}}{\delta}\\
		&\lesssim\sqrt{C_{\mathrm{u}}^{2}C_{\mathrm{w}}\sum_{i=1}^{T}\EE_{i-1}\left[P_{i}W_{i}\right]\log\frac{T^{SA}}{\delta}} + C_{\mathrm{u}}C_{\mathrm{w}}\log\frac{T^{SA}}{\delta}
	\end{align*}
    for all possible collections $\{N(s,a) \in [T] \mid (s,a)\in \cS\times \cA\}$ simultaneously.

    On the other hand, we can bound $\left|\sum_{i=1}^T Y_{i}\right|$ with a simple application of Freedman's inequality. After checking the conditions for $X_i$, we can readily see
		\begin{align*}
			\left|Y_i\right| \le 2C_{\mathrm{w}}, \\
			\sum_{i=1}^{T}\EE_{i-1}\left[\left|Y_{i}\right|^{2}\right] \le T C_{\mathrm{w}}^2.
		\end{align*}
	
	We can invoke Freedman's inequality (Theorem \ref{thm:Freedman}) with $m=1$ to arrive at
	\begin{align}
		\left|\sum_{i=1}^{T} Y_{i} \right| \lesssim \sqrt{T C_{\mathrm{w}}^{2} \log\frac{1}{\delta}} + C_{\mathrm{w}} \log\frac{1}{\delta}
	\end{align}
	with probability at least $1-\delta$.
\end{proof}

\section{Proof of Theorem \ref{thm:main}}\label{sec:analysis}

For an upper bound on the cumulative regret, we begin our analysis with a decomposition of the regret. By the optimistic property of our value function estimates in Lemma \ref{lem:Q_t-lower-bound}, we can 
\begin{equation}
    \mathrm{Regret}(T) = \sum_{t=1}^T \big(V^\star(s_t) - V^{\pi_t}(s_t)\big) \le \sum_{t=1}^T \big(V_{t-1}(s_t) - V^{\pi_t}(s_t)\big).
\end{equation}
In the inequality, since $V_0(s)$ is undefined in our algorithm, we can simply let $V_0(s) = \frac{1}{1-\gamma}$ for all $s\in \cS$. 

Onward, we can focus on finding an upper bound on $V_{t-1}(s_t) - V^{\pi_t}(s_t)$. To this end, we further decompose $V_{t-1}(s_t) - V^{\pi_t}(s_t)$ into  three component terms and bound each of them respectively. The decomposition is as follows:
\begin{align}
    &\quad\ V_{t-1}(s_t) - V^{\pi_t}(s_t) \notag\\
    &\overset{(\mathrm{i})}{=} V_{t-1}(s_t) - Q^{\pi_t}(s_t, a_t) \notag\\
    &= V_{t-1}(s_t) - Q^\star(s_t, a_t) + Q^\star(s_t, a_t) - Q^{\pi_t}(s_t, a_t) \notag\\
    &\overset{(\mathrm{ii})}{=} V_{t-1}(s_t) - Q^\star(s_t, a_t) + \gamma P_{s_t,a_t}\left(V^\star - V^{\pi_t}\right) \notag\\
    &\overset{(\mathrm{iii})}{=} V_{t-1}(s_t) - Q^\star(s_t, a_t) + \gamma \left(P_{s_t,a_t} - P_t\right)\left(V^\star - V^{\pi_{t}}\right) + \gamma \left(V^\star(s_{t+1}) - V^{\pi_t}(s_{t+1})\right) \notag\\
    &\overset{(\mathrm{iv})}{=} \gamma\left(V^\star(s_{t+1}) - V^{\pi_{t+1}}(s_{t+1})\right) + Q_{t}(s_t, a_t) - Q^\star(s_t, a_t) + \zeta_t \notag\\
    &\overset{(\mathrm{v})}{\leq} \frac{\gamma(3-\gamma)}{2}\left(V^\star(s_{t+1}) - V^{\pi_{t+1}}(s_{t+1})\right) + Q_{t}(s_t, a_t) - Q^\star(s_t, a_t) + \zeta_t. \label{eq:regret-decomp1}
\end{align}
In the series of equalities above, (i) is because $\pi_t(s_t) = a_t$ in Algorithm \ref{alg:main}. (ii) is by the Bellman equation \eqref{eq:bellman} and Bellman optimality equation \eqref{eq:bellman-opt}. (iii) is because by the definition in \eqref{eq:Pt-def}, $P_t(V^\star - V^{\pi_t}) = V^\star(s_{t+1}) - V^{\pi_t}(s_{t+1})$. In (iv), we group a few component terms into a newly-defined term
\begin{equation}\label{eq:zeta-def}
    \zeta_t := V_{t-1}(s_t) - Q_t(s_t,a_t) + \gamma\left(P_{s_t,a_t} - P_t\right)\left(V^\star - V^{\pi_t}\right) + \gamma\left(V^{\pi_{t+1}}(s_{t+1}) - V^{\pi_t}(s_{t+1})\right).
\end{equation}
Lastly, (v) is due to $\gamma \le \frac{\gamma(3-\gamma)}{2}$.

The regret can be further manipulated with an upper bound on $Q_{t}(s_t, a_t) - Q^\star(s_t, a_t)$ via the application of Lemma \ref{lem:Q_t-upper-bound}. Such manipulation leads to a recursion on $V_{t-1}(s_t) - V^{\pi_t}(s_t)$. Rearranging \eqref{eq:regret-decomp1}, we obtain a recursion as follows:
    \begin{align}
        &\quad \left(1 - \frac{\gamma(3-\gamma)}{2}\right)\sum_{t=1}^T \big(V_{t-1}(s_t) - V^{\pi_t}(s_t)\big) \notag\\
        &= \sum_{t=1}^T\left(V_{t-1}(s_t) - V^{\pi_t}(s_t)\right) - \frac{\gamma(3-\gamma)}{2}\sum_{t=1}^T\left(V_{t-1}(s_t) - V^{\pi_t}(s_t)\right) \notag\\
        &\overset{(\mathrm{i})}{\le} \sum_{t=1}^T\left(V_{t-1}(s_t) - V^{\pi_t}(s_t)\right) - \frac{\gamma(3-\gamma)}{2}\sum_{t=2}^T\left(V_{t-1}(s_t) - V^{\pi_t}(s_t)\right) \notag\\
        &= \sum_{t=1}^T\left(V_{t-1}(s_t) - V^{\pi_t}(s_t)\right) - \frac{\gamma(3-\gamma)}{2}\sum_{t=1}^{T-1}\left(V_{t}(s_{t+1}) - V^{\pi_{t+1}}(s_{t+1})\right) \notag\\
        &\overset{(\mathrm{ii})}{\le} \sum_{t=1}^T\bigg(\frac{\gamma(3-\gamma)}{2}\left(V^\star(s_{t+1}) - V^{\pi_{t+1}}(s_{t+1})\right) + Q_{t}(s_t, a_t) - Q^\star(s_t, a_t) + \zeta_t\bigg) - \frac{\gamma(3-\gamma)}{2}\sum_{t=1}^{T-1}\left(V_{t}(s_{t+1}) - V^{\pi_{t+1}}(s_{t+1})\right) \notag\\
        &\overset{(\mathrm{iii})}{\le} C_1\left(\sqrt{\frac{SAT}{1-\gamma}\log^{3}\frac{SAT}{\delta}} + \frac{SA}{(1-\gamma)^7}\log^{7/2} \frac{SAT}{\delta} + \sqrt{\frac{SA\log^2 T}{(1-\gamma)^5}\sum_{t=1}^T \big(V_{t-1}(s_{t}) - V^{\pi_{t}}(s_{t})\big)}\right) \notag\\ & \qquad + \sum_{t=1}^T\bigg(\frac{\gamma(3-\gamma)}{2}\left(V_t(s_{t+1}) - V^{\pi_{t+1}}(s_{t+1})\right) + \zeta_t\bigg) - \frac{\gamma(3-\gamma)}{2}\sum_{t=1}^{T-1}\left(V_{t}(s_{t+1}) - V^{\pi_{t+1}}(s_{t+1})\right) \notag\\
        &= C_1\left(\sqrt{\frac{SAT}{1-\gamma}\log^{3}\frac{SAT}{\delta}} + \frac{SA}{(1-\gamma)^7}\log^{7/2} \frac{SAT}{\delta} + \sqrt{\frac{SA\log^2 T}{(1-\gamma)^5}\sum_{t=1}^T \big(V_{t-1}(s_{t}) - V^{\pi_{t}}(s_{t})\big)}\right) \notag\\ & \qquad + \frac{\gamma(3-\gamma)}{2}\left(V^\star(s_{T+1}) - V^{\pi_{T+1}}(s_{T+1})\right) + \sum_{t=1}^T \zeta_t \notag\\
        &\overset{(\mathrm{iv})}{\le} C_1\left(\sqrt{\frac{SAT}{1-\gamma}\log^{3}\frac{SAT}{\delta}} + \frac{SA}{(1-\gamma)^7}\log^{7/2} \frac{SAT}{\delta} + \sqrt{\frac{SA\log^2 T}{(1-\gamma)^5}\sum_{t=1}^T \big(V_{t-1}(s_{t}) - V^{\pi_{t}}(s_{t})}\big)\right) \notag\\ & \qquad + \frac{3}{2(1-\gamma)} + \sum_{t=1}^T \zeta_t \notag\\
        &\lesssim \sqrt{\frac{SAT}{1-\gamma}\log^{3}\frac{SAT}{\delta}} + \frac{SA}{(1-\gamma)^7}\log^{7/2} \frac{SAT}{\delta} + \sqrt{\frac{SA\log^2 T}{(1-\gamma)^5}\sum_{t=1}^T \big(V_{t-1}(s_{t}) - V^{\pi_{t}}(s_{t})\big)} + \sum_{t=1}^T\zeta_t. \label{eq:regret-decomp2}
    \end{align}
    Above, (i) is because by the optimistic property in Lemma \ref{lem:Q_t-upper-bound}, $V_{t-1}(s_t) \ge V^\star(s_t) \ge V^{\pi_t}(s_t) \ge 0$ and thus $V_{t-1}(s_t) - V^{\pi_t}(s_t) \ge 0$. (ii) is obtained by replacing $\left(V_{t-1}(s_t) - V^{\pi_t}(s_t)\right)$ with \eqref{eq:regret-decomp1}. (iii) is obtained by invoking Lemma \ref{lem:Q_t-upper-bound} to replace $\sum_{t=1}^T \big(Q_t(s_{t},a_{t}) - Q^\star(s_{t},a_{t})\big)$. (iv) is due to $V^\star(s_{T+1}) - V^{\pi_{T+1}}(s_{T+1}) \le \frac{1}{1-\gamma}$ and $\gamma\in(0,1)$. 

Note $0 < 1 - \frac{\gamma(3-\gamma)}{2} < 1$ since $\gamma\in(0,1)$. Dividing \eqref{eq:regret-decomp2} by $1 - \frac{\gamma(3-\gamma)}{2}$, we have
\begin{align}
    & \quad\ \sum_{t=1}^T \big(V_{t-1}(s_t) - V^{\pi_t}(s_t)\big) \notag\\
    &\lesssim \sqrt{\frac{SAT\log^{3}\frac{SAT}{\delta}}{(1-\gamma)^3}} + \frac{SA\log^{7/2} \frac{SAT}{\delta}}{(1-\gamma)^8} + \sqrt{\frac{SA\log^2 T}{(1-\gamma)^7}\sum_{t=1}^T \big(V_{t-1}(s_{t}) - V^{\pi_{t}}(s_{t})\big)} + \frac{1}{1-\gamma}\sum_{t=1}^T\zeta_t, \label{eq:regret-decomp3}
\end{align}
in which the additional $\frac{1}{1-\gamma}$ factor is due to the fact $1 - \frac{\gamma(3-\gamma)}{2} \ge \frac{1-\gamma}{2}$.

We continue by invoking Lemma \ref{lem:xi} on \eqref{eq:regret-decomp3}, which allows us to write \eqref{eq:regret-decomp3} as
\begin{align}
    & \quad\ \sum_{t=1}^T \big(V_{t-1}(s_t) - V^{\pi_t}(s_t)\big) \notag\\
    &\leq C_3\Bigg(\sqrt{\frac{SAT\log^3\frac{SAT}{\delta}}{(1-\gamma)^3}} + \frac{SA\log^{7/2} \frac{SAT}{\delta}}{(1-\gamma)^8} + \sqrt{\frac{SA\log^2 T}{(1-\gamma)^7}\sum_{t=1}^T\left(V_{t-1}(s_{t}) - V^{\pi_{t}}(s_{t})\right)} \notag\\ & \qquad + \sqrt{\frac{SA\log\frac{T}{\delta}}{(1-\gamma)^3} \sum_{t=1}^{T}\big(V^{\star}(s_{t}) - V^{\pi_{t}}(s_{t})\big)}\Bigg) \label{eq:regret-decomp4}
\end{align}
for some absolute constant $C_3 > 0$.

Finally, some algebraic manipulation on \eqref{eq:regret-decomp4} can lead to the desired result in Theorem \ref{thm:main}. In particular, we can use Proposition \ref{prop:algebraic}, which is an algebraic fact about quadratic inequalities like \eqref{eq:regret-decomp4}. The proof of Proposition \ref{prop:algebraic} is provided in Appendix \ref{sec:appendix-proof-prop:algebraic}.

Specifically, we can invoke Proposition \ref{prop:algebraic} on \eqref{eq:regret-decomp4} with
\begin{align*}
    x &:= C_3\log T\sqrt{\frac{SA}{(1-\gamma)^7}},\\
    y &:= C_3\Bigg(\sqrt{\frac{SAT\log^3\frac{SAT}{\delta}}{(1-\gamma)^3}} + \frac{SA\log^{7/2} \frac{SAT}{\delta}}{(1-\gamma)^8} + \sqrt{\frac{SA\log\frac{T}{\delta}}{(1-\gamma)^3} \sum_{t=1}^{T}\big(V^{\star}(s_{t}) - V^{\pi_{t}}(s_{t})\big)}\Bigg),\\
    z &:= \sum_{t=1}^T \big(V_{t-1}(s_t) - V^{\pi_t}(s_t)\big)
\end{align*}
in \eqref{eq:algebraic-lem-cond}. \eqref{eq:algebraic-lem-stat} in Proposition \ref{prop:algebraic} would give
\begin{align}
    & \quad\ \sum_{t=1}^T \big(V^\star(s_t) - V^{\pi_t}(s_t)\big) \notag\\
    & \leq \sum_{t=1}^T \big(V_{t-1}(s_t) - V^{\pi_t}(s_t)\big) \notag\\
    &\leq 2C_3\sqrt{\frac{SAT\log^3\frac{SAT}{\delta}}{(1-\gamma)^3}} + (2C_3 + C_3^2)\frac{SA\log^{7/2}\frac{SAT}{\delta}}{(1-\gamma)^8} + 2C_3\sqrt{\frac{SA\log\frac{T}{\delta}}{(1-\gamma)^3} \sum_{t=1}^{T}\big(V^{\star}(s_{t}) - V^{\pi_{t}}(s_{t})\big)}\Bigg). \label{eq:regret-decomp5}
\end{align}

We apply Proposition \ref{prop:algebraic} once again on \eqref{eq:regret-decomp5} by letting 
\begin{align*}
    x &:= 2C_3\sqrt{\frac{SA\log\frac{T}{\delta}}{(1-\gamma)^3}},\\
    y &:= 2C_3\sqrt{\frac{SAT\log^3\frac{SAT}{\delta}}{(1-\gamma)^3}} + (2C_3 + C_3^2)\frac{SA\log^{7/2}\frac{SAT}{\delta}}{(1-\gamma)^8},\\
    z &:= \sum_{t=1}^T \big(V^\star(s_t) - V^{\pi_t}(s_t)\big)
\end{align*}
in \eqref{eq:algebraic-lem-cond}. \eqref{eq:algebraic-lem-stat} would give
\begin{equation}
    \Regret(T) = \sum_{t=1}^T \big(V^\star(s_t) - V^{\pi_t}(s_t)\big) \lesssim \sqrt{\frac{SAT\log^3\frac{SAT}{\delta}}{(1-\gamma)^3}} + \frac{SA\log^{7/2}\frac{SAT}{\delta}}{(1-\gamma)^8} \label{eq:regret-decomp-final}
\end{equation}
with probability at least $1-\delta$, which is as claimed in Theorem \ref{thm:main}.


\section{Proof of Lemma \ref{lem:Q_t-lower-bound}}\label{sec:proof:Q_t-lower-bound}

We can prove this lemma by induction. 

For the base case, it is clear that $Q_1(s,a) = \frac{1}{1-\gamma} \ge Q^\star(s,a)$ for all $(s,a)\in \cS\times\cA$. For the inductive step, assuming the inductive hypothesis $Q_{t}(s,a) \ge Q^\star(s,a)$ all the way up to $t$ for all $(s,a)\in \cS\times\cA$, we need to prove $Q_{t+1}(s,a) \ge Q^\star(s,a)$. To this end, given any $t \in [T]$ and the inductive hypothesis at the $t$-th step, it suffices to prove for all $(s,a) \in \cS \times \cA$,
    \begin{equation}\label{eq:Q-monotonicity-intermediate}
		\min\{Q^{\UCB}_{t+1}(s, a), Q^{\mathrm{R}}_{t+1}(s, a)\} \ge Q^\star(s,a).
	\end{equation}
	To prove \eqref{eq:Q-monotonicity-intermediate}, we need to examine the two arguments of the $\min$ function separately and prove 
	\begin{equation}\label{eq:QUCB-monotonicity}
		Q^{\UCB}_{t+1}(s, a) \ge Q^\star(s,a).
	\end{equation}
	and 
	\begin{equation}\label{eq:QR-monotonicity}
		Q^{\mathrm{R}}_{t+1}(s, a) \ge Q^\star(s,a).
	\end{equation}
	respectively. In the proof of \eqref{eq:QUCB-monotonicity}, we will use the abbreviation $N := N_t(s,a)$ for notational simplicity. In this notation, we can see that $(s,a) := (s_{t_N}, a_{t_N})$. 
    
    $\bullet$ Proof of \eqref{eq:QUCB-monotonicity}: First, let us prove the optimism of $Q^{\UCB}$ in \eqref{eq:QUCB-monotonicity}.

    Let $b_n := c_b\sqrt{\frac{\log\frac{SAT}{\delta}}{(1-\gamma)^3 n}}$. We rewrite the subroutine \textbf{update-q-ucb}() in Line \ref{alg-line:QUCB-update} of Algorithm \ref{alg:aux}:
    \begin{equation*}
        Q^{\UCB}_{t+1}(s, a) = Q^{\UCB}_{t_N+1}(s, a) = (1 - \eta_N)Q^{\UCB}_{t_N}(s, a) + \eta_{N}\left(r(s, a) + \gamma V_{t_N}(s_{t_N+1}) + b_N\right).
    \end{equation*}
    This recursive relation can be expanded as
    \begin{equation}
        Q^{\UCB}_{t+1}(s, a) = \eta_0^N Q^{\UCB}_{1}(s, a) + \sum_{n=1}^N \eta_n^N \big(r(s, a) + \gamma V_{t_n}(s_{t_n+1}) + b_n\big), \label{eq:QUCB}
    \end{equation}
    in which $\eta_0^N$ and $\eta_n^N$ are defined in Lemma \ref{lem:learning-rate}.
    
    Using \eqref{eq:QUCB}, we have 
    \begin{align}\label{eq:Q_UCB-Q*}
        &\quad Q^{\UCB}_{t+1}(s, a) - Q^\star(s, a) \nonumber\\
        &= \eta_0^N \left(Q^{\UCB}_{1}(s, a) - Q^\star(s, a)\right) + \sum_{n=1}^N \eta_n^N \big(r(s, a) + \gamma V_{t_n}(s_{t_n+1}) + b_n - Q^\star(s, a)\big) \nonumber\\
        &= \eta_0^N \left(Q^{\UCB}_{1}(s, a) - Q^\star(s, a)\right) + \sum_{n=1}^N \eta_n^N \Big(\gamma\big(V_{t_n}(s_{t_n+1}) - V^\star(s_{t_n+1})\big) + \gamma\left(P_{t_n} - P_{s,a}\right)V^\star + b_n\Big),
    \end{align}
    where the last step is by $Q^\star(s,a) = r(s,a) + \gamma P_{s,a} V^\star$ and $V^\star(s_{t_n+1}) = P_{t_n}V^\star$.
    
    We can invoke Lemma 4.3 from \citet{Jin-provably}, which relies on Azuma-Hoeffding inequality, with only notational modification and obtain with probability at least $1-\delta$,
    \begin{equation}\label{eq:Jin-Hoeffding}
        \left|\sum_{n=1}^N\eta_n^N\gamma\left(P_{t_n} - P_{s,a}\right)V^\star\right| \le c_b\sqrt{\frac{\log\frac{SAT}{\delta}}{(1-\gamma)^3N}}\sum_{n=1}^N\eta_n^N \le \sum_{n=1}^N\eta_n^N b_n,
    \end{equation}
    for all $(N,s,a) \in [T] \times \cS \times \cA$. This allows us to conclude that with probability at least $1-\delta$,
    \begin{align*}
        &\quad Q^{\UCB}_{t+1}(s, a) - Q^\star(s, a)\\
        &\ge \eta_0^N \left(Q^{\UCB}_{1}(s, a) - Q^\star(s, a)\right) + \sum_{n=1}^N \eta_n^N \Big(\gamma\big(V_{t_n}(s_{t_n+1}) - V^\star(s_{t_n+1})\big) + \gamma\left(P_{t_n} - P_{s,a}\right)V^\star + b_n\Big)\\
        &\ge \eta_0^N \left(Q^{\UCB}_{1}(s, a) - Q^\star(s, a)\right) + \sum_{n=1}^N \eta_n^N \gamma\big(V_{t_n}(s_{t_n+1}) - V^\star(s_{t_n+1})\big)\\
        &\ge \sum_{n=1}^N \eta_n^N \gamma\big(V_{t_n}(s_{t_n+1}) - V^\star(s_{t_n+1})\big),
    \end{align*}
    where the last step is because $Q^{\UCB}_{1}(s,a) = \frac{1}{1-\gamma} \ge Q^\star(s,a)$.
    
    Therefore, we can prove \eqref{eq:QUCB-monotonicity}
    \begin{equation*}
        Q^{\UCB}_{t+1}(s, a) - Q^\star(s, a) \ge \sum_{n=1}^N \eta_n^N \gamma\big(V_{t_n}(s_{t_n+1}) - V^\star(s_{t_n+1})\big) \ge 0,
    \end{equation*}
    where the last inequality is a result of the inductive hypothesis: $V_t(s) \ge Q_t(s, \pi^\star(s)) \ge Q^\star(s, \pi^\star(s)) = V^\star(s)$, for all $(t,s,a)\in [T]\times \cS \times \cA$.

    $\bullet$ Proof of \eqref{eq:QR-monotonicity}: Next, let us prove the optimism of $Q^{\rref}$ in \eqref{eq:QR-monotonicity}. Since only $Q_{t+1}^{\rref}(s_t,a_t)$ is updated at the $t$-th step with all other entries of $Q_{t}^{\rref}$ fixed, it suffices to check 
    \begin{equation}
        Q^{\rref}_{t+1}(s_t,a_t) \ge Q^{\star}(s_t,a_t). \label{eq:optimism-induction-step}
    \end{equation}
    In the proof of \eqref{eq:QR-monotonicity}, we will use the abbreviation $N_t := N_t(s_t,a_t)$ and $t_n := t_n(s_t,a_t)$ for notational simplicity.
    \paragraph{Step 1: decomposing $Q^{\rref}_t(s_t,a_t) - Q^{\star}(s_t,a_t)$.}
    
    According to the update rule of $Q_t^\rref$ in Line \ref{alg-line:QR-update} of Algorithm \ref{alg:aux}, we have 
    \begin{align*}
    	&\quad\ Q_{t+1}^{\rref}(s_t,a_t) \\
        & = Q^{\rref}_{t_{N_t}+1}(s_t,a_t) \\
        &=(1-\eta_{N_t})Q^{\rref}_{t_{N_t}}(s_t,a_t) +\eta_{N_t}\left\{ r(s_t,a_t)+\gamma \left(V_{t_{N_t}}(s_{t_{N_t}+1})-V^{\rref}_{t_{N_t}}(s_{t_{N_t}+1})+\mu^{\re}_{t_{N_t}+1}(s_t,a_t)\right)+b^{\rref}_{t_{N_t}+1}\right\} \\
    	& = (1-\eta_{N_t})Q^{\rref}_{t_{N_t-1}+1}(s_t,a_t) +\eta_{N_t}\left\{ r(s_t,a_t)+\gamma \left(V_{t_{N_t}}(s_{t_{N_t}+1})-V^{\rref}_{t_{N_t}}(s_{t_{N_t}+1})+\mu^{\re}_{t_{N_t}+1}(s_t,a_t)\right)+b^{\rref}_{t_{N_t}+1}\right\},	
    \end{align*}
    in which the first and last step follow from the fact $t_N = t_{N_t} = t$. Expanding this recursively, we have
    \begin{align}
    	Q_{t+1}^{\rref}(s_t,a_t) 
    	& = \eta_0^{N_t} Q_1^{\rref}(s_t, a_t)   \notag \\
    	& \qquad + \sum_{n = 1}^{N_t} \eta_n^{N_t} \left\{ r(s_t,a_t) + \gamma\left(V_{t_n}(s_{t_n+1}) - V^{\rref}_{t_n}(s_{t_n+1}) + \mu^{\re}_{t_{n}+1}(s_t,a_t)\right) + b^{\rref}_{t_{n}+1} \right\}.
    	\label{eq:Q_R-decompose}
    \end{align}
    On the other hand, since $\eta_0^{N_t} + \sum_{n = 1}^{N_t} \eta_n^{N_t} = 1$ by \eqref{eq:learning-rate1}, we can write $Q^{\star} (s_t,a_t)$ as a sum weighted by $\{\eta^{N_t}_n\}_{n=0}^{N_t}$. This leads to
    \begin{align}\label{eq:Q_R-Q*-decompose}
    	& Q_{t+1}^{\rref}(s_t,a_t) - Q^{\star} (s_t,a_t) = \eta_0^{N_t} \left(Q_1^{\rref}(s_t, a_t) - Q^{\star} (s_t,a_t)\right) \notag \\
    	& \qquad + \sum_{n = 1}^{N_t} \eta_n^{N_t} \left\{ r(s_t,a_t) + \gamma\left(V_{t_n}(s_{t_n+1}) - V^{\rref}_{t_n}(s_{t_n+1}) + \mu^{\re}_{t_{n}+1}(s_t,a_t)\right) + b^{\rref}_{t_{n}+1} - Q^{\star} (s_t,a_t)\right\}.
    \end{align}

    Moreover, to further expand \eqref{eq:Q_R-Q*-decompose}, we can invoke the definition of $\mu^{\re}$ (which is a moving average of $V^{\rref}$) and expand $Q^{\star}(s_t,a_t)$ using the Bellman optimality equation (cf. \eqref{eq:bellman-opt}), which gives:
    \begin{align}
    	&\quad\ r(s_t,a_t) + \gamma\left(V_{t_n}(s_{t_n+1}) - V^{\rref}_{t_n}(s_{t_n+1}) + \mu^{\re}_{t_{n}+1}(s_t,a_t)\right) + b^{\rref}_{t_{n}+1} - Q^{\star} (s_t,a_t)\nonumber\\
    	&= \gamma\left(V_{t_n}(s_{t_n+1}) - V^{\rref}_{t_n}(s_{t_n+1})+\frac{\sum_{i=1}^{n}V^{\rref}_{t_i}(s_{t_i+1})}{n}-P_{s_t,a_t}V^{\star}\right)+b^{\rref}_{t_{n}+1} \label{eq:Q_R-Q*-decompose1}\\
    	&= \gamma\left(P_{t_n}\left\{V_{t_n} - V^{\rref}_{t_n}\right\}+\frac{\sum_{i=1}^{n}P_{t_n}V^{\rref}_{t_i}}{n}-P_{s_t,a_t}V^{\star}\right)+b^{\rref}_{t_{n}+1}\nonumber\\
    	&= \gamma\left(P_{s_t,a_t}\left\{V_{t_n} - V^{\rref}_{t_n}\right\}+\frac{\sum_{i=1}^{n}P_{s_t,a_t}V^{\rref}_{t_i}}{n}-P_{s_t,a_t}V^{\star}\right)+b^{\rref}_{t_{n}+1} + \xi_{t_n}\nonumber\\
    	&= \gamma P_{s_t,a_t}\left(V_{t_n} - V^\star +\frac{\sum_{i=1}^{n}V^{\rref}_{t_i} - V^{\rref}_{t_n}}{n}\right)+b^{\rref}_{t_{n}+1} + \xi_{t_n},\label{eq:Q_R-Q*-decompose2}
    \end{align}
    in which we group the higher-order terms into a newly-defined term
    \begin{equation}
    	\xi_{t_n} :=  \gamma\big(P_{t_n} - P_{s_t,a_t} \big)\big(V_{t_n} - V^{\rref}_{t_n} \big) +  \frac{\gamma}{n}\sum_{i=1}^n \big( P_{t_n} - P_{s_t,a_t} \big) V^{\rref}_{t_i}.
    	\label{eq:xi-def}
    \end{equation}
    
    Combining \eqref{eq:Q_R-Q*-decompose} and \eqref{eq:Q_R-Q*-decompose2} leads to the following decomposition
    \begin{align}
    	&\quad\ Q^{\rref}_t(s_t,a_t) - Q^{\star}(s_t,a_t)\notag\\
        &= \eta^{N_t}_0 \left(Q^{\rref}_1(s_t,a_t) - Q^{\star}(s_t,a_t)\right) + \sum_{n = 1}^{N_t} \eta^{N_t}_n \left\{\gamma P_{s_t,a_t}\left(V_{t_n} - V^\star +\frac{\sum_{i=1}^{n}V^{\rref}_{t_i} - V^{\rref}_{t_n}}{n}\right)+b^{\rref}_{t_{n}+1} + \xi_{t_n}\right\}. 
    	\label{eq:QR-Q*-final}
    \end{align}

    Finally, we can notice that several terms in \eqref{eq:QR-Q*-final} can be dropped:
    \begin{enumerate}[label=(\arabic*)]
		\item Our initialization satisfies $Q^{\rref}_1(s,a) - Q^{\star}(s,a) \geq 0$ for all $(s,a) \in \cS\times\cA$.
		
		\item The inductive hypothesis implies that for any $n$ such that $1 \leq t_n \leq t$, one has $V_{t_n}\geq V^{\star}$.
		
		\item For all $0 \leq i \leq n$ and any $s\in \cS$, one has 
		\begin{equation}
			V^{\rref}_{t_i}(s) - V^{\rref}_{t_n}(s) \geq 0,
			\label{eq:VRti-VRtn}
		\end{equation}
        because $V^{\rref}_t(s)$ is updated with $V_t(s)$ and $V_t(s)$ is non-increasing by its update rule.
	\end{enumerate}

    Thus, \eqref{eq:QR-Q*-final} can be lower-bounded as follows:
    \begin{equation}
		Q^{\rref}_t(s_t,a_t) - Q^{\star}(s_t,a_t)
		\geq \sum_{n = 1}^{N_t} \eta^{N_t}_n \left(b^{\rref}_{t_{n}+1} + \xi_{t_n}\right).
		\label{eq:QR-Q-LB}
	\end{equation}

    To establish the desired result of optimism, it remains to show the right-hand side of \eqref{eq:QR-Q-LB} is non-negative. In particular, it suffices for us to show
	\begin{equation}
		\left|\sum_{n = 1}^{N_t} \eta^{N_t}_n \xi_{t_n} \right| \le 
		\sum_{n = 1}^{N_t} \eta^{N_t}_n b^{\rref}_{t_{n}+1}.
		\label{eq:claim-sum-eta-sum-b}
	\end{equation}

    To justify \eqref{eq:claim-sum-eta-sum-b}, we will examine the two quantities that compose $\xi_{t_n}$. Let us denote them as $I_1$ and $I_2$ respectively:
	\begin{subequations}
		\label{eq:I1-I2-defs}
		\begin{align}
			I_1 &:=  \gamma\sum_{n=1}^{N_t} \eta_n^{N_t} \left(P_{t_n} - P_{s_t,a_t} \right)\left(V_{t_n} - V^{\rref}_{t_n} \right),  \label{eq:I1-def} \\
			I_2 &:= \gamma\sum_{n=1}^{N_t} \frac{1}{n} \eta_n^{N_t} \sum_{i=1}^n \left( P_{t_n} - P_{s_t,a_t} \right) V^{\rref}_{t_i}. \label{eq:I2-def}
		\end{align}
	\end{subequations}
    In the following, we will bound $I_1$ and $I_2$ with Lemma \ref{lemma:freedman-application}, which is a result of Freedman's inequality.

    \paragraph{Step 2: controlling $I_1$.}
    We can invoke Lemma \ref{lemma:freedman-application} to control the term $I_1$ defined in \eqref{eq:I1-def}. 
    
	We begin by constructing the necessary notation and verify the conditions for Lemma \ref{lemma:freedman-application}. Consider any $N \in [T]$. Let
	\begin{align}
		W_i := V_{i} - V^{\rref}_{i}
		\qquad \text{and} \qquad
		u_i(s,a, N) := \eta_{N_i(s,a)}^N \geq 0. \label{I1-def1}
	\end{align}
	Clearly, $W_i$ is deterministic at the end of the $(i-1)$-th step. Note that $N_i(s,a)$ is also deterministic at the end of the $(i-1)$-th step, because $s_t$ and $Q_{Z_{t-1}}$ (and thus $a_t$) are already deterministic at the end of the $(i-1)$-th step. Moreover, we define an upper bound $C_{\mathrm{w}}$ as follows:
	\begin{align}
		\norm{W_i}_\infty \leq \norm{V^{\rref}_{i}}_{\infty} + \norm{V_i}_{\infty} \leq \frac{2}{1-\gamma} =: C_{\mathrm{w}}. \label{I1-cw}
	\end{align}
	In addition, using \eqref{eq:learning-rate-def2} and \eqref{eq:learning-rate3} in Lemma \ref{lem:learning-rate}, we have
    \begin{align*}
		\eta_{N_{i}(s,a)}^{N} \leq\frac{6}{N(1-\gamma)} ,\qquad &\text{if }1\le N_{i}(s,a)\le N;\\
		\eta_{N_{i}(s,a)}^{N} = 0, \qquad &\text{if }N_{i}(s,a)>N,
	\end{align*}
    which in turn allows us to define an upper bound $C_{\mathrm{u}}$: 
	\begin{align}
		\max_{N, s, a\in [T] \times \cS\times \cA} \eta_{N_{i}(s,a)}^{N} \leq \frac{6}{N(1-\gamma)} =: C_{\mathrm{u}}. \label{I1-cu}
	\end{align}
    Lastly, we can verify
    \begin{align}
		0 \leq \sum_{n=1}^N u_{t_n(s,a)}(s,a, N) = \sum_{n=1}^N \eta_{n}^N \leq 1 \label{eq:I1-cond-u}
	\end{align}
	holds for all $(N,s,a) \in [T] \times \cS\times \cA$.
 
	Hence, we can apply Lemma \ref{lemma:freedman-application} with \eqref{I1-def1}, \eqref{I1-cw}, \eqref{I1-cu} and $(N,s,a)=(N_t,s_t, a_t)$ and conclude with probability at least $1-\delta$, 
    \begin{align}
		|I_1| &= \left|\gamma\sum_{n=1}^{N_t} \eta_n^{N_t} \left(P_{t_n} - P_{s_t,a_t} \right)\left(V_{t_n} - V^{\rref}_{t_n} \right)\right| = \gamma\left|\sum_{i=1}^t X_i(s_t,a_t, N_t)\right| \nonumber\\
		&\lesssim \sqrt{C_{\mathrm{u}}\log^{2}\frac{SAT}{\delta}}\sqrt{\sum_{n=1}^{N_{t}}u_{t_{n}}(s_t,a_t,N)\Var_{s_t, a_t}\left(W_{t_{n}(s_t, a_t)}\right)}+\left(C_{\mathrm{u}}C_{\mathrm{w}}+\sqrt{\frac{C_{\mathrm{u}}}{N}}C_{\mathrm{w}}\right)\log^{2}\frac{SAT}{\delta}\nonumber\\
		&\lesssim \sqrt{\frac{1}{N_t(1-\gamma)}\log^{2}\frac{SAT}{\delta}}\sqrt{\sum_{n=1}^{N_{t}}\eta^{N_t}_n \Var_{s_t, a_t}\left(V_{t_n} - V^{\rref}_{t_n}\right)} + \frac{\log^2\frac{SAT}{\delta}}{N_t(1-\gamma)^2} \label{eq:I1-1}\\
		&\lesssim \sqrt{\frac{1}{N_t(1-\gamma)}\log^2\frac{SAT}{\delta}}\sqrt{{\sigma}^{\adv}_{t_{N_t}+1}(s_t,a_t) - \left({\mu}^{\adv}_{t_{N_t}+1}(s_t,a_t) \right)^2} + \frac{\log^2\frac{SAT}{\delta}}{(N_t)^{3/4}(1-\gamma)^2}.
		\label{eq:I1-2}
	\end{align}
	Above, the last inequality \eqref{eq:I1-2} requires additional work to establish. Its proof is deferred to Appendix \ref{sec:proof:eq:I1-2} to streamline the presentation.   

    \paragraph{Step 3: controlling $I_2$.}

    Next, we turn to controlling the term $I_2$ defined in \eqref{eq:I2-def}, which can be rearranged as
	\begin{align}
		I_2 = \gamma\sum_{n=1}^{N_t} \frac{1}{n} \eta_n^{N_t} \sum_{i=1}^n \left( P_{t_n} - P_{s_t,a_t} \right) V^{\rref}_{t_i}
		= \gamma\sum_{n=1}^{N_t} \left( \sum_{n = i}^{N_t} \frac{\eta^{N_t}_n}{n} \right)\left( P_{t_n} - P_{s_t,a_t} \right) V^{\rref}_{t_i}.
	\end{align}
	In the following, we control $I_2$ by invoking Lemma \ref{lemma:freedman-application} again. Consider any $N\in[T]$. We abuse the notation by letting
	\begin{align}
		W_i := V^{\rref}_{i} 
		\qquad \text{and} \qquad
		u_i(s,a, N) := \sum_{n=N_i(s,a)}^N \frac{\eta_n^N}{n} \geq 0. \label{I2-def1}
	\end{align}
	Moreover, define
	\begin{align}
		\norm{W_i}_\infty \leq \norm{V^{\rref}_{i}}_\infty \leq \frac{1}{1-\gamma} := C_{\mathrm{w}} \label{I2-cw}
	\end{align}
	and by \eqref{eq:learning-rate2} in Lemma~\ref{lem:learning-rate}, we have
	\begin{align}
		\sum_{n=N_i(s,a)}^N \frac{\eta_n^N}{n} \leq \sum_{n = 1}^{N} \frac{\eta^{N}_n}{n} \le \frac{2}{N} := C_{\mathrm{u}}, \label{I2-cu}
	\end{align}
	which immediately implies:
	\begin{align}
		0 \leq \sum_{n=1}^N u_{t_n(s,a)}(s,a, N) \leq \sum_{n=1}^{N}\frac{2}{N}\leq 2. 
	\end{align}
    Overall, the conditions for Lemma \ref{lemma:freedman-application} are satisfied. Therefore, we can apply Lemma \ref{lemma:freedman-application} with \eqref{I2-def1}, \eqref{I2-cw}, \eqref{I2-cu} and $(N,s,a)=(N_t,s_t, a_t)$ and conclude with probability at least $1-\delta$,  
	\begin{align}
		| I_2 | &= \left|\gamma\sum_{n=1}^{N_t} \left( \sum_{n = i}^{N_t} \frac{\eta^{N_t}_n}{n} \right) \sum_{i=1}^n \left( P_{t_n} - P_{s_t,a_t} \right) V^{\rref}_{t_i}\right| = \gamma\left|\sum_{i=1}^t X_i(s_t,a_t,N_t)\right|\nonumber\\
		&\lesssim \sqrt{C_{\mathrm{u}}\log^{2}\frac{SAT}{\delta}}\sqrt{\sum_{n=1}^{N_{t}}u_{t_{n}}(s_t,a_t,N)\Var_{s_t,a_t}\left(W_{t_{n}(s_t,a_t)}\right)}+\left(C_{\mathrm{u}}C_{\mathrm{w}}+\sqrt{\frac{C_{\mathrm{u}}}{N}}C_{\mathrm{w}}\right)\log^{2}\frac{SAT}{\delta}\nonumber\\
		&\lesssim \sqrt{\frac{1}{N_t}\log^2\frac{SAT}{\delta}}\sqrt{\frac{1}{N_t}\sum_{n = 1}^{N_t} \Var_{s_t,a_t}\left(W_{t_{n}(s_t,a_t)}\right)} + \frac{1}{N_t(1-\gamma)} \log^2\frac{SAT}{\delta} \label{eq:I2-1}\\
		&\lesssim \sqrt{\frac{1}{N_t}\log^2\frac{SAT}{\delta}}\sqrt{ \sigma^{\re}_{t_{N_t}+1}(s_t,a_t) - \left(\mu^{\re}_{t_{N_t}+1}(s_t,a_t) \right)^2} + \frac{ \log^2\frac{SAT}{\delta}}{(N_t)^{3/4}(1-\gamma)}
		\label{eq:I2-2}
	\end{align}
    Above, the last inequality \eqref{eq:I2-2} requires additional work to establish. Its proof is deferred to Appendix \ref{sec:proof:eq:I2-2} to streamline the presentation.   

    \paragraph{Step 4: combining $I_1$ and $I_2$.}
	
	Summing up the results in \eqref{eq:I1-2} and \eqref{eq:I2-2}, we can find an upper bound on $\left|\sum_{n = 1}^{N_t} \eta^{N_t}_n \xi_{t_n} \right|$ as follows:
    \begin{align}
		&\quad\ \left|\sum_{n = 1}^{N_t} \eta^{N_t}_n \xi_{t_n} \right| \leq |I_1| +|I_2| \notag\\
		&\lesssim \sqrt{\frac{1}{N_t(1-\gamma)}\log^2\frac{SAT}{\delta}}\sqrt{{\sigma}^{\adv}_{t_{N_t}+1}(s_t,a_t) - \left({\mu}^{\adv}_{t_{N_t}+1}(s_t,a_t) \right)^2} \notag\\
		&\qquad + \sqrt{\frac{1}{N_t}\log^2\frac{SAT}{\delta}}\sqrt{ \sigma^{\re}_{t_{N_t}+1}(s_t,a_t) - \left(\mu^{\re}_{t_{N_t}+1}(s_t,a_t) \right)^2} + \frac{\log^2\frac{SAT}{\delta}}{(N_t)^{3/4}(1-\gamma)^2} \nonumber \\
		& \leq B^{\rref}_{t_{N_t}+1}(s_t,a_t) + c_b \frac{\log^2\frac{SAT}{\delta}}{(N_t)^{3/4}(1-\gamma)^2} 
		\label{eq:xi-sum}
	\end{align}
	for some sufficiently large absolute constant $c_b>0$. The last step follows from the definition of $B^{\rref}_{t_{N_t}+1}(s_t,a_t)$ in Line \ref{alg-line:Bnext} and \ref{alg-line:BR} of Algorithm \ref{alg:aux}.

    With $\left|\sum_{n = 1}^{N_t} \eta^{N_t}_n \xi_{t_{n}}\right|$ dealt with, we still need to make some manipulation on $\sum_{n = 1}^{N_t} \eta^{N_t}_n b^{\rref}_{t_{n}+1}$ in order to establish \eqref{eq:claim-sum-eta-sum-b}. We do so by leveraging the following relation, which can be reasoned from the definitions of $b^{\rref}_{t_{n}+1}$ and $\delta^{\rref}_{t_{n}+1}$ in Line \ref{alg-line:b} and \ref{alg-line:delta} of Algorithm \ref{alg:aux}:
	\begin{equation}
		b^{\rref}_{t_{n}+1} =\Big(1-\frac{1}{\eta_{n}}\Big)B_{t_n}^{\rref}(s_{t},a_{t})
		+\frac{1}{\eta_{n}}B_{t_n+1}^{\rref}(s_t,a_t)+\frac{c_b}{n^{3/4}(1-\gamma)^2}\log^{2}\frac{SAT}{\delta}
		\label{eq:b-delta-identity}
	\end{equation}
	by the fact $-\frac{1-\eta_n}{\eta_n}=1-\frac{1}{\eta_n}$.
	Using \eqref{eq:b-delta-identity} and the definition of $\eta_n^N$ in \eqref{eq:learning-rate-def2}, we have
	\begin{align}
		&\quad\ \sum_{n = 1}^{N_t} \eta^{N_t}_n b^{\rref}_{t_{n}+1} \nonumber \\
		&=\sum_{n = 1}^{N_t} \eta_n \prod_{i = n+1}^{N_t}(1-\eta_i) \left( \left(1-\frac{1}{\eta_n}\right) B_{t_n}^{\rref}(s_{t},a_{t})
		+\frac{1}{\eta_{n}}B_{t_n+1}^{\rref}(s_t,a_t) \right) + c_b \sum_{n=1}^{N_t}\frac{ \eta_n^{N_t} }{n^{3/4}(1-\gamma)^2}  \log^2\frac{SAT}{\delta}\nonumber \\
		&= \sum_{n = 1}^{N_t}  \prod_{i = n+1}^{N_t}(1-\eta_i) \left( -\left(1-\eta_n\right) B_{t_n}^{\rref}(s_{t},a_{t})
		+ B_{t_n+1}^{\rref}(s_t,a_t) \right) + c_b \sum_{n=1}^{N_t}\frac{ \eta_n^{N_t} }{n^{3/4}(1-\gamma)^2}  \log^2\frac{SAT}{\delta}\nonumber \\
		&= \sum_{n = 1}^{N_t}  \left(\prod_{i = n+1}^{N_t}(1-\eta_i) B_{t_n+1}^{\rref}(s_{t},a_{t})
		- \prod_{i = n}^{N_t}(1-\eta_i)B_{t_n}^{\rref}(s_t,a_t) \right) + c_b \sum_{n=1}^{N_t}\frac{ \eta_n^{N_t} }{n^{3/4}(1-\gamma)^2}  \log^2\frac{SAT}{\delta}\nonumber \\
		&\overset{\mathrm{(i)}}{=} \sum_{n = 1}^{N_t} \prod_{i = n+1}^{N_t}(1-\eta_i) B_{t_n+1}^{\rref}(s_{t},a_{t})
		- \sum_{n = 2}^{N_t} \prod_{i = n}^{N_t}(1-\eta_i)B_{t_n}^{\rref}(s_t,a_t) + c_b \sum_{n=1}^{N_t}\frac{ \eta_n^{N_t} }{n^{3/4}(1-\gamma)^2}  \log^2\frac{SAT}{\delta}\nonumber \\
		&\overset{\mathrm{(ii)}}{=} \sum_{n = 1}^{N_t} \prod_{i = n+1}^{N_t}(1-\eta_i) B_{t_n+1}^{\rref}(s_{t},a_{t})
		- \sum_{n = 1}^{N_t-1} \prod_{i = n+1}^{N_t}(1-\eta_i)B_{t_n+1}^{\rref}(s_t,a_t) + c_b \sum_{n=1}^{N_t}\frac{ \eta_n^{N_t} }{n^{3/4}(1-\gamma)^2}  \log^2\frac{SAT}{\delta}\nonumber \\
		&= B^{\rref}_{t_{N_t}+1}(s_t,a_t)  + c_b \sum_{n=1}^{N_t}\frac{ \eta_n^{N_t} }{n^{3/4}(1-\gamma)^2}  \log^2\frac{SAT}{\delta}. 
		\label{eq:eta-b-sum-id}
	\end{align}
    Above, (i) is due to $B^{\rref}_{t_1}(s_t,a_t) = 0$. (ii) can be obtained as follows:
	\begin{align*}
		\sum_{n = 2}^{N_t} \prod_{i = n}^{N_t}(1-\eta_i)B_{t_n}^{\rref}(s_t,a_t) & = \sum_{n = 1}^{N_t-1} \prod_{i = n+1}^{N_t}(1-\eta_i)B_{t_{n+1}}^{\rref}(s_t,a_t) = \sum_{n = 1}^{N_t-1} \prod_{i = n+1}^{N_t}(1-\eta_i)B_{t_n+1}^{\rref}(s_t,a_t), 
	\end{align*}
	in which the summation is re-indexed over the first equality by replacing $n$ with $n+1$, and the second equality is true because the state-action pair 
	$(s_t,a_t)$ has not been visited between the $(t_n +1)$-th step and the $(t_{n+1} - 1)$-th step.

    Furthermore, \eqref{eq:eta-b-sum-id} allows us to conclude
	\begin{equation}
		B^{\rref}_{t_{N_t}+1}(s_t,a_t) + c_b\frac{\log^2\frac{SAT}{\delta}}{(N_t)^{3/4}(1-\gamma)^2} \le \sum_{n = 1}^{N_t} \eta^{N_t}_n b^{\rref}_{t_{n}+1} \le B^{\rref}_{t_{N_t}+1}(s_t,a_t) + 2c_b\frac{\log^2\frac{SAT}{\delta}}{(N_t)^{3/4}(1-\gamma)^2},
		\label{eq:eta-b-sum-squeeze}
	\end{equation}
    because \eqref{eq:learning-rate2} of Lemma \ref{lem:learning-rate} implies
    \begin{equation*}
        \frac{1}{(N_t)^{3/4}} \le \sum_{n = 1}^{N_t} \frac{\eta^{N_t}_n}{n^{3/4}} \le \frac{2}{(N_t)^{3/4}}.
    \end{equation*}
	Bringing \eqref{eq:eta-b-sum-squeeze} into \eqref{eq:xi-sum} gives
	\begin{equation}
		\left|\sum_{n = 1}^{N_t} \eta^{N_t}_n \xi_{t_n} \right|
		\le B^{\rref}_{t_{N_t}+1}(s_t,a_t) + c_b \frac{\log^2\frac{SAT}{\delta}}{(N_h^k)^{3/4}(1-\gamma)^2}  \le \sum_{n = 1}^{N_t} \eta^{N_t}_n b^{\rref}_{t_{n}+1} 
	\end{equation}
	as claimed in \eqref{eq:claim-sum-eta-sum-b}. 
 
    Finally, by proving \eqref{eq:claim-sum-eta-sum-b}, we have proven the induction step \eqref{eq:optimism-induction-step} and thus concluded the proof of $Q_t(s, a) \ge Q^\star(s, a)$.

    To prove the optimism in $V_t$, given \eqref{eq:Q-monotonicity}, we can easily see from $V_t$'s update rule that for all $(t,s,a)\in [T]\times \cS \times \cA$, 
    \begin{equation*}
        V_t(s) \ge Q_t(s, \pi^\star(s)) \ge Q^\star(s, \pi^\star(s)) = V^\star(s).
    \end{equation*}
    
    \subsection{Proof of Inequality \eqref{eq:I1-2}}\label{sec:proof:eq:I1-2}

    As it turns out, it suffices to examine the following term in order to prove \eqref{eq:I1-2}:
    \begin{align}
    	I_3 := \sum_{n = 1}^{N_t} \eta^{N_t}_n\Var_{s_t,a_t}\left(V_{t_n} - {V}^{\rref}_{t_n}\right) - {\sigma}^{\adv}_{t_{N_t}+1}(s_t,a_t) + \left({\mu}^{\adv}_{t_{N_t}+1}(s_t,a_t) \right)^2.
    	\label{eq:I3-def}
    \end{align}
    In this section, we would like to show an upper bound on $I_3$. 

    Before we proceed, recall that the update rules of $\mu^{\adv}_{t_{n+1}}$ and ${\sigma}^{\adv}_{t_{n+1}}$ in Line \ref{alg-line:mu-adv} and \ref{alg-line:sigma-adv} of Algorithm \ref{alg:aux} are
    \begin{align*}
    	\mu^{\adv}_{t_{n+1}}(s_t,a_t) &= \mu^{\adv}_{t_n+1}(s_t,a_t) 
    	= (1-\eta_n)\mu^{\adv}_{t_n}(s_t,a_t) + \eta_n \left( V_{t_n}(s_{t_n+1}) - V^{\rref}_{t_n}(s_{t_n+1}) \right) ,\\ 
    	\sigma^{\adv}_{t_{n+1}}(s_t,a_t) &= \sigma^{\adv}_{t_n+1}(s_t,a_t) 
    	= (1-\eta_n)\sigma^{\adv}_{t_n}(s_t,a_t) + \eta_n \left(V_{t_n}(s_{t_n+1}) - V^{\rref}_{t_n}(s_{t_n+1}) \right)^2.
    \end{align*} 
    We can expand these equations recursively and write
    \begin{subequations} \label{eq:I3-1}
    	\begin{align}
    		\mu^{\adv}_{t_{N_{t}}+1}(s_t,a_t) & = \sum_{n=1}^{N_{t}}\eta_{n}^{N_{t}}\left(V_{t_n}(s_{t_n+1}) - V^{\rref}_{t_n}(s_{t_n+1})\right) = \sum_{n=1}^{{N_{t}}}\eta_{n}^{N_{t}}P_{t_{n}}\left(V_{t_{n}}-{V}^{\rref}_{t_n}\right), \label{eq:I3-1.1}\\
    		{\sigma}^{\adv}_{t_{N_{t}}+1}(s_t,a_t) & = \sum_{n=1}^{{N_{t}}}\eta_{n}^{N_{t}}\left(V_{t_n}(s_{t_n+1}) - V^{\rref}_{t_n}(s_{t_n+1})\right)^{2}= \sum_{n=1}^{{N_{t}}}\eta_{n}^{N_{t}}P_{t_{n}}\left(V_{t_{n}}-{V}^{\rref}_{t_n}\right)^{2}. \label{eq:I3-1.2}
    	\end{align}
    \end{subequations}
    Noting that $\sum_{n=1}^{{N_{t}}}\eta_{n}^{N_{t}}=1$, we can use Jensen's inequality and reason from \eqref{eq:I3-1.1} and \eqref{eq:I3-1.2} that: 
    \begin{align}\label{eq:I3-2}
    	{\sigma}^{\adv}_{t_{N_{t}}+1}(s_t,a_t) \geq \left(\mu^{\adv}_{t_{N_{t}}+1}(s_t,a_t)\right)^2.
    \end{align}
    Recall by the definition of variance in \eqref{eq:var-def}, we have
    \begin{equation*}
    	\Var_{s_t,a_t}\left(V_{t_n} - {V}^{\rref}_{t_n} \right) 
    	= P_{s_t,a_t} \left(V_{t_{n}}-{V}^{\rref}_{t_n}\right)^{2} - \left(P_{s_t,a_t} \left(V_{t_{n}}-{V}^{\rref}_{t_n} \right) \right)^2.
    \end{equation*}
    Bringing this into \eqref{eq:I3-2}, we can decompose $I_3$ as follows 
    \begin{align}
    	I_3 & = \sum_{n = 1}^{N_t} \eta^{N_t}_n P_{s_t,a_t} \left(V_{t_{n}}-{V}^{\rref}_{t_n}\right)^{2} - \sum_{n = 1}^{N_t} \eta^{N_t}_n P_{t_n} \left(V_{t_{n}}-{V}^{\rref}_{t_n}\right)^{2} \nonumber \\
    	&\qquad \qquad+ \left(\sum_{n=1}^{{N_{t}}}\eta_{n}^{N_{t}}P_{t_{n}}\left(V_{t_{n}}-{V}^{\rref}_{t_n}\right)\right)^2  - \sum_{n=1}^{{N_{t}}}\eta_{n}^{N_{t}}\left(P_{s_t,a_t} \left(V_{t_{n}}-{V}^{\rref}_{t_n} \right) \right)^2\nonumber \\
    	& \le  \underbrace{ \left|\sum_{n = 1}^{N_t} \eta^{N_t}_n \left(P_{s_t,a_t} - P_{t_n} \right)\left(V_{t_{n}}-{V}^{\rref}_{t_n}\right)^{2}\right| }_{:= I_{3,1}}  \nonumber \\
    	&\qquad \qquad+ \underbrace{\left(\sum_{n=1}^{{N_{t}}}\eta_{n}^{N_{t}}P_{t_{n}}\left(V_{t_{n}}-{V}^{\rref}_{t_n}\right)\right)^2  - \sum_{n=1}^{{N_{t}}}\eta_{n}^{N_{t}}\left(P_{s_t,a_t} \left(V_{t_{n}}-{V}^{\rref}_{t_n} \right) \right)^2}_{:= I_{3,2}} . 
    	\label{eq:I3-3}
    \end{align}
    It remains to control $I_{3,1}$ and $I_{3,2}$ in \eqref{eq:I3-3} separately. 

    \paragraph{Step 1: controlling $I_{3,1}$.}

    Now, let us control the term $I_{3,1}$ defined in \eqref{eq:I3-3} by invoking Lemma \ref{lemma:freedman-application}. Consider any $N\in[T]$. We set
	\begin{align}
    	W_i := \left(V_{i} - {V}^{\rref}_{i} \right)^{ 2}
    	\qquad \text{and} \qquad
    	u_i(s,a, N) := \eta_{N_i(s,a)}^N. 
    	\label{I31-def1}
    \end{align}
	Moreover, define
	\begin{align}
		\norm{W_i}_\infty \leq  \left( \norm{ V^{\rref}_{i} }_{\infty} + \norm{V_{i}}_\infty \right)^2 
	\leq \frac{4}{(1-\gamma)^2}:= C_{\mathrm{w}}, \label{I31-cw}
	\end{align}
	and same as \eqref{I1-cu}, we have
	\begin{align}
		\max_{N, s, a\in [T] \times \cS\times \cA} \eta_{N_{i}(s,a)}^{N} \leq \frac{8}{N(1-\gamma)} := C_{\mathrm{u}}. \label{I31-cu}
	\end{align}
    Clearly, we can verify
    \begin{align}
		0 \leq \sum_{n=1}^N u_{t_n(s,a)}(s,a, N) = \sum_{n=1}^N \eta_{n}^N \leq 1 
	\end{align}
	holds for all $(N,s,a) \in [T] \times \cS\times \cA$, and the conditions for Lemma \ref{lemma:freedman-application} are satisfied.

    Therefore, with $(N,s,a)=(N_t,s_t, a_t)$, we can apply Lemma \ref{lemma:freedman-application} with \eqref{I31-def1}, \eqref{I31-cw} and \eqref{I31-cu} to conclude with probability at least $1-\delta$,  
    \begin{align}
    	|I_{3,1}| &= \left|\sum_{n = 1}^{{N_t}} \eta_n^{N_t} \left(P_{t_n}-P_{s_t,a_t} \right)\left(V_{t_{n}}-{V}^{\rref}_{t_n}\right)^2\right| \notag\\
        &= \left|\sum_{i=1}^t X_i(s_t,a_t,N_t)\right| \notag\\
    	&\lesssim \sqrt{C_{\mathrm{u}} \log^2\frac{SAT}{\delta}}\sqrt{\sum_{n = 1}^{N_t} u_{t_n}(s_t,a_t, N_t) \Var_{s_t,a_t} \left(W_{t_n} \right)} + \left(C_{\mathrm{u}} C_{\mathrm{w}} + \sqrt{\frac{C_{\mathrm{u}}}{N}} C_{\mathrm{w}}\right) \log^2\frac{SAT}{\delta} \nonumber \\
    	& \lesssim \sqrt{\frac{1}{N_t(1-\gamma)}\log^2\frac{SAT}{\delta}}\sqrt{\sum_{n = 1}^{N_t} \eta^{N_t}_n \Var_{s_t,a_t}\left(\left(V_{t_n} - {V}^{\rref}_{t_n} \right)^2\right) } + \frac{\log^2\frac{SAT}{\delta}}{N_t(1-\gamma)^3} \nonumber \\
    	&\lesssim \sqrt{\frac{1}{N_t(1-\gamma)^5}\log^2\frac{SAT}{\delta}} + \frac{1}{N_t(1-\gamma)^3}\log^2\frac{SAT}{\delta}, \label{eq:I31-final}
    \end{align}
    in which the last step follows from $\sum_{n = 1}^{N_t} \eta^{N_t}_n \leq 1$ and $\Var_{s_t,a_t}\left(\left(V_{t_n} - {V}^{\rref}_{t_n} \right)^2\right) \leq \norm{\left(V_{t_n} - {V}^{\rref}_{t_n}  \right)^{4}}_\infty 
    	\leq \frac{16}{(1-\gamma)^4}$.

    \paragraph{Step 2: controlling $I_{3,2}$.} We can use Jensen's inequality and write $I_{3,2}$ as 
    \begin{align}
    	I_{3,2} &= \left(\sum_{n=1}^{{N_{t}}}\eta_{n}^{N_{t}}P_{t_{n}}\left(V_{t_{n}}-{V}^{\rref}_{t_n}\right)\right)^2 - \sum_{n=1}^{{N_{t}}}\eta_{n}^{N_{t}}\left(P_{s_t,a_t}\big(V_{t_{n}}-{V}^{\rref}_{t_n}\big)\right)^2 \notag\\ 
        &\le \left(\sum_{n=1}^{{N_{t}}}\eta_{n}^{N_{t}}P_{t_{n}}\left(V_{t_{n}}-{V}^{\rref}_{t_n}\right)\right)^2 - \left\{\sum_{n=1}^{{N_{t}}}\eta_{n}^{N_{t}}\right\} \left\{ \sum_{n=1}^{{N_{t}}}\eta_{n}^{N_{t}}\left(P_{s_t,a_t}\left(V_{t_{n}}-{V}^{\rref}_{t_n}\right)\right)^2\right\} \notag\\ 
        &\le \left(\sum_{n=1}^{{N_{t}}}\eta_{n}^{N_{t}}P_{t_{n}}\left(V_{t_{n}}-{V}^{\rref}_{t_n}\right)\right)^2 - \left(\sum_{n=1}^{{N_{t}}}\left(\eta_{n}^{N_{t}}\right)^{1/2} \left(\eta_{n}^{N_{t}}\right)^{1/2}P_{s_t,a_t}\left(V_{t_{n}}-{V}^{\rref}_{t_n}\right)\right)^2 \notag\\
        &= \left(\sum_{n=1}^{{N_{t}}}\eta_{n}^{N_{t}}P_{t_{n}}\left(V_{t_{n}}-{V}^{\rref}_{t_n}\right)\right)^2  - \left(\sum_{n=1}^{{N_{t}}}\eta_{n}^{N_{t}}P_{s_t,a_t}\left(V_{t_{n}}-{V}^{\rref}_{t_n}\right)\right)^2 \notag\\
        &\le \left\{\sum_{n=1}^{{N_{t}}}\eta_{n}^{N_{t}}\left(P_{t_{n}} - P_{s_t,a_t}\right)\left(V_{t_{n}}-{V}^{\rref}_{t_n}\right)\right\} \left\{\sum_{n=1}^{{N_{t}}}\eta_{n}^{N_{t}}\left(P_{t_{n}} + P_{s_t,a_t}\right)\left(V_{t_{n}}-{V}^{\rref}_{t_n}\right)\right\}. \label{eq:I32-1}
    \end{align}
    With this upper bound on $I_{3,2}$, we now aim to show
    \begin{equation} \label{eq:I32-WTS}
    	I_{3,2} \leq C_{3,2} \left( \sqrt{\frac{1}{N_t(1-\gamma)^5}\log^2\frac{SAT}{\delta}} + \frac{1}{N_t(1-\gamma)^3}\log^2\frac{SAT}{\delta}\right),
    \end{equation}
    for some universal constant $C_{3,2}>0$. 
    
    Since \eqref{eq:I32-WTS} is trivially true if $I_{3,2}\leq 0$, it suffices to consider the case where $I_{3,2}> 0$ only. To proceed, notice that the first factor in the product in \eqref{eq:I32-1} is exactly $I_1$ in \eqref{eq:I1-def}, which can be bounded with our previous result \eqref{eq:I1-1}:
    \begin{align} 
    	|I_1| 
    	&\lesssim \sqrt{\frac{1}{N_t(1-\gamma)}\log^{2}\frac{SAT}{\delta}}\sqrt{\sum_{n=1}^{N_{t}}\eta^{N_t}_n \Var_{s,a}\left(V_{t_n} - V^{\rref}_{t_n}\right)} + \frac{\log^2\frac{SAT}{\delta}}{N_t(1-\gamma)^2}\nonumber\\
    	&\overset{\mathrm{(i)}}{\lesssim} \sqrt{\frac{1}{N_t(1-\gamma)^3}\log^2\frac{SAT}{\delta}}\sqrt{\sum_{n = 1}^{N_t} \eta^{N_t}_n  } + \frac{\log^2\frac{SAT}{\delta}}{N_t(1-\gamma)^2} \nonumber \\	
    	&\overset{\mathrm{(ii)}}{\lesssim} \sqrt{\frac{1}{N_t(1-\gamma)^3}\log^2\frac{SAT}{\delta}} + \frac{1}{N_t(1-\gamma)^2}\log^2\frac{SAT}{\delta}
    	\label{eq:I32-1.1}
    \end{align}
    with probability at least $1-\delta$. Above, (i) arises from 
    $\Var_{s,a}\big(V_{t_n} - V^{\rref}_{t_n}\big)  
    	\leq \norm{\left(V_{t_n} - V^{\rref}_{t_n}\right)^{2}}_\infty \leq \frac{4}{(1-\gamma)^2}$ and (ii) is true due to $\sum_{n = 1}^{N_t} \eta^{N_t}_n \leq 1$.

    On the other hand, the second factor in the product in \eqref{eq:I32-1} can be simply bounded as follows:
    \begin{align}
    	\left|\sum_{n=1}^{{N_{t}}}\eta_{n}^{N_{t}}\left(P_{t_{n}} + P_{s_t,a_t}\right)\left(V_{t_{n}}-{V}^{\rref}_{t_n}\right)\right|  \le \sum_{n=1}^{{N_{t}}}\eta_{n}^{N_{t}} \left( \norm{P_{t_n}}_1 + \norm{P_{s_t,a_t}}_1 \right) \norm{V_{t_{n}}-{V}^{\rref}_{t_n}}_{\infty} \leq \frac{2}{1-\gamma},
    	\label{eq:I32-1.2}
    \end{align}
    in which we use the elementary facts $\sum_{n = 1}^{N_t} \eta^{N_t}_n \leq 1$, $\norm{V_{t_{n}}-{V}^{\rref}_{t_n}}_{\infty} \leq \frac{1}{1-\gamma}$ and $\norm{P_{t_n}}_1 = \norm{P_{s_t,a_t}}_1=1$.
    
    Substituting \eqref{eq:I32-1.1} and \eqref{eq:I32-1.2} back into \eqref{eq:I32-1}, we can obtain the desired bound \eqref{eq:I32-WTS} for the case $I_{3,2}>0$. 
    Combined this with the trivial case of $I_{3,2}\le 0$, we have established \eqref{eq:I32-WTS}. 

    \paragraph{Step 3: combining the preceding bounds.}

    Let us bring the bounds \eqref{eq:I31-final} and \eqref{eq:I32-WTS} into \eqref{eq:I3-3}. This gives
    \begin{align*} 
    	I_3 \leq |I_{3,1}| + I_{3,2} 
    	\leq C_{3,3}\left\{ \sqrt{\frac{1}{N_t(1-\gamma)^5}\log^2\frac{SAT}{\delta}} + \frac{1}{N_t(1-\gamma)^3}\log^2\frac{SAT}{\delta}\right\} 
    \end{align*}
    for some absolute constant $C_{3,3}>0$. Writing $I_3$ with its definition in \eqref{eq:I3-def}, rearranging the inequality results, and taking square root on both sides, we have
    \begin{align*}
    	\left\{\sum_{n = 1}^{N_t} \eta^{N_t}_n\Var_{s_t,a_t}\left(V_{t_n} - {V}^{\rref}_{t_n}\right)\right\}^{1/2} &\leq \left\{ {\sigma}^{\adv}_{t_{N_t}+1}(s_t,a_t) - \left({\mu}^{\adv}_{t_{N_t}+1}(s_t,a_t) \right)^2\right\}^{1/2}\\
    	& \quad + \sqrt{C_{3,3}} \left(\sqrt{\frac{1}{N_t(1-\gamma)^5}\log^2\frac{SAT}{\delta}} + \frac{1}{N_t(1-\gamma)^3}\log^2\frac{SAT}{\delta}\right)^{1/2}\\
        &\leq \left\{ {\sigma}^{\adv}_{t_{N_t}+1}(s_t,a_t) - \left({\mu}^{\adv}_{t_{N_t}+1}(s_t,a_t) \right)^2\right\}^{1/2}\\
    	& \quad + \frac{1}{\left(N_{t}\right)^{1/4}(1-\gamma)^{5/4}}\log^{1/2}\frac{SAT}{\delta}+\frac{1}{\left(N_t\right)^{1/2}(1-\gamma)^{3/2}}\log\frac{SAT}{\delta}. 
    \end{align*}
    in which we also use the result of \eqref{eq:I3-2}.
    
    Substitution of this into \eqref{eq:I1-1} leads to the desired result \eqref{eq:I1-2}. 

    \subsection{Proof of Inequality \eqref{eq:I2-2}}\label{sec:proof:eq:I2-2}
    
    Similar to the proof of \eqref{eq:I1-2}, it suffices to examine the following term in order to prove \eqref{eq:I2-2}:
    \begin{align}
    	I_4 := \sum_{n = 1}^{N_t} \eta^{N_t}_n\Var_{s_t,a_t}\left({V}^{\rref}_{t_n}\right) - {\sigma}^{\re}_{t_{N_t}+1}(s_t,a_t) + \left({\mu}^{\re}_{ t_{N_t}+1}(s_t,a_t) \right)^2.
    	\label{eq:I4-def}
    \end{align}
    In this section, we would like to show an upper bound on $I_4$. 

    Before we proceed, recall that the update rules of $\mu^{\re}_{t_{n+1}}$ and ${\sigma}^{\re}_{t_{n+1}}$ in Line \ref{alg-line:mu-ref} and \ref{alg-line:sigma-ref} of Algorithm \ref{alg:aux} are
    \begin{align*}
    	{\mu}^{\re}_{t_{n+1}}(s_t,a_t) &= {\mu}^{\re}_{t_n+1}(s_t,a_t) = \left(1-\frac{1}{n}\right){\mu}^{\re}_{t_n}(s_t,a_t) + \frac{1}{n}{V}^{\rref}_{t_n}(s_{t_n+1}), \\
    	{\sigma}^{\re}_{t_{n+1}}(s_t,a_t) &= {\sigma}^{\re}_{t_n+1}(s_t,a_t) = \left(1-\frac{1}{n}\right){\sigma}^{\re}_{t_n}(s_t,a_t) + \frac{1}{n}\left({V}^{\rref}_{t_n}(s_{t_n+1})\right)^2,
    \end{align*} 
    We can expand these equations recursively and write
    \begin{subequations}
    	\label{eq:I4-1}
    	\begin{align}
    		{\mu}^{\re}_{t_{N_{t}}+1}(s_t,a_t) & = \frac{1}{N_{t}}\sum_{n=1}^{N_{t}}{V}_{t_n}^{\rref}(s_{t_{n}+1}) = \frac{1}{N_{t}}\sum_{n=1}^{N_{t}} P_{t_n}{V}_{t_n}^{\rref}, \label{eq:I4-1.1}\\
    		{\sigma}^{\re}_{t_{N_{t}}+1}(s_t,a_t) & = \frac{1}{N_{t}}\sum_{n=1}^{N_{t}}\left({V}_{t_n}^{\rref}(s_{t_{n}+1})\right)^2 = \frac{1}{N_{t}}\sum_{n=1}^{N_{t}} P_{t_n}\left({V}_{t_n}^{\rref}\right)^2. \label{eq:I4-1.2}
    	\end{align}
    \end{subequations}
    Noting that $\sum_{n=1}^{{N_{t}}}\eta_{n}^{N_{t}}=1$, we can use Jensen's inequality and reason from \eqref{eq:I4-1.1} and \eqref{eq:I4-1.2} that: 
    \begin{align}\label{eq:I4-2}
    	{\sigma}^{\re}_{t_{N_{t}}+1}(s_t,a_t) \geq \left({\mu}^{\re}_{t_{N_{t}}+1}(s_t,a_t)\right)^2.
    \end{align}
    Recall by the definition of variance in \eqref{eq:var-def}, we have
    \begin{equation*}
    	\Var_{s_t,a_t}\left({V}^{\rref}_{t_n} \right) 
    	= P_{s_t,a_t} \left({V}^{\rref}_{t_n}\right)^{2} - \left(P_{s_t,a_t} \left({V}^{\rref}_{t_n} \right) \right)^2.
    \end{equation*}
    Combined this with \eqref{eq:I4-2}, we can decompose $I_4$ as follows 
    \begin{align}
    	I_4 & = \frac{1}{{N_t}}\sum_{n=1}^{N_t} \left(P_{s_t,a_t} \left({V}^{\rref}_{t_n}\right)^{2} - \left(P_{s_t,a_t} \left({V}^{\rref}_{t_n} \right) \right)^2\right) - \frac{1}{N_{t}}\sum_{n=1}^{N_{t}} P_{t_n}\left({V}_{t_n}^{\rref}\right)^2 + \left(\frac{1}{N_{t}}\sum_{n=1}^{N_{t}} P_{t_n}{V}_{t_n}^{\rref}\right)^2 \nonumber \\
    	& = \underbrace{ \frac{1}{{N_t}}\sum_{n = 1}^{{N_t}} \left(P_{s_t,a_t} - P_{t_n}\right) \left({V}^{\rref}_{t_n}\right)^2 }_{:= \, I_{4,1}} + \underbrace{ \left(\frac{1}{N_{t}}\sum_{n=1}^{N_{t}} P_{t_n}{V}_{t_n}^{\rref}\right)^2 - \frac{1}{{N_t}}\sum_{n = 1}^{{N_t}} \left(P_{s_t,a_t} {V}^{\rref}_{t_n}\right)^2}_{:=\, I_{4,2}}. \label{eq:I4-3}
    \end{align}
    It remains to control $I_{4,1}$ and $I_{4,2}$ in \eqref{eq:I4-3} separately. 

    \paragraph{Step 1: controlling $I_{4,1}$.}

    Now, let us control $I_{4,1}$ defined in \eqref{eq:I4-3} by invoking Lemma \ref{lemma:freedman-application}. Consider any $N\in[T]$. We set
	\begin{align}
    	W_i := \left({V}^{\rref}_{i} \right)^{ 2}
    	\qquad \text{and} \qquad
    	u_i(s,a, N) := \frac{1}{N}. 
    	\label{I41-def1}
    \end{align}
	Moreover, define
	\begin{align}
		|u_i(s,a, N)| = \frac{1}{N} := C_{\mathrm{u}}
	\qquad \text{and} \qquad \norm{W_i}_\infty \leq  \norm{V_{i}^\rref}_\infty^2 
	\leq \frac{1}{(1-\gamma)^2}:= C_{\mathrm{w}}. \label{I41-cw-cu}
	\end{align}
    Clearly, we can verify
    \begin{align}
		\sum_{n=1}^{N} u_{t_{n}(s,a)}(s,a, N)= \sum_{n=1}^{N}\frac{1}{N}=1
	\end{align}
	holds for all $(N,s,a) \in [T] \times \cS\times \cA$, and the conditions for Lemma \ref{lemma:freedman-application} are satisfied.

    Therefore, we can apply Lemma \ref{lemma:freedman-application} with \eqref{I41-def1}, \eqref{I41-cw-cu} and $(N,s,a)=(N_t,s_t, a_t)$ and conclude with probability at least $1-\delta$,  
    \begin{align}
    	|I_{4,1}| &= \left|\sum_{n = 1}^{{N_t}} \eta_n^{N_t} \left(P_{t_n}-P_{s_t,a_t} \right)\left({V}^{\rref}_{t_n}\right)^2\right| \notag\\
        &= \left|\sum_{i=1}^t X_i(s_t,a_t,N_t)\right| \nonumber\\
    	&\lesssim \sqrt{C_{\mathrm{u}} \log^2\frac{SAT}{\delta}}\sqrt{\sum_{n = 1}^{N_t} u_{t_n}(s_t,a_t, N_t) \Var_{s_t,a_t} \left(W_{t_n} \right)} + \left(C_{\mathrm{u}} C_{\mathrm{w}} + \sqrt{\frac{C_{\mathrm{u}}}{N}} C_{\mathrm{w}}\right) \log^2\frac{SAT}{\delta} \nonumber \\
	    &\lesssim  \sqrt{\frac{\log^2 \frac{SAT}{\delta}}{{N_t(1-\gamma)^4}}} + \frac{\log^2\frac{SAT}{\delta}}{{N_t(1-\gamma)^2}}, \label{eq:I41-final}
    \end{align}
    in which the last inequality results from $\sum_{n = 1}^{N_t} \eta^{N_t}_n \leq 1$ and the fact $\Var_{s_t,a_t}\left(\left(V^{\rref}_{t_n} \right)^2\right)
    	\leq \norm{\left({V}^{\rref}_{t_n}  \right)^{4}}_\infty 
    	\leq \frac{1}{(1-\gamma)^4}$.

    \paragraph{Step 2: controlling $I_{4,2}$.}

    We can use Jensen's inequality and write $I_{4,2}$ as 
    \begin{align}
    	I_{4,2} &= \left(\sum_{n=1}^{{N_{t}}}\eta_{n}^{N_{t}}P_{t_{n}} {V}^{\rref}_{t_n}\right)^2 - \sum_{n=1}^{{N_{t}}}\eta_{n}^{N_{t}}\left(P_{s_t,a_t}{V}^{\rref}_{t_n}\right)^2 \notag\\ 
        &\le \left(\sum_{n=1}^{{N_{t}}}\eta_{n}^{N_{t}}P_{t_{n}}{V}^{\rref}_{t_n}\right)^2 - \left\{\sum_{n=1}^{{N_{t}}}\eta_{n}^{N_{t}}\right\} \left\{ \sum_{n=1}^{{N_{t}}}\eta_{n}^{N_{t}}\left(P_{s_t,a_t}{V}^{\rref}_{t_n}\right)^2\right\} \notag\\ 
        &\le \left(\sum_{n=1}^{{N_{t}}}\eta_{n}^{N_{t}}P_{t_{n}}{V}^{\rref}_{t_n}\right)^2 - \left(\sum_{n=1}^{{N_{t}}}\left(\eta_{n}^{N_{t}}\right)^{1/2} \left(\eta_{n}^{N_{t}}\right)^{1/2}P_{s_t,a_t}{V}^{\rref}_{t_n}\right)^2 \notag\\
        &= \left(\sum_{n=1}^{{N_{t}}}\eta_{n}^{N_{t}}P_{t_{n}}{V}^{\rref}_{t_n}\right)^2  - \left(\sum_{n=1}^{{N_{t}}}\eta_{n}^{N_{t}}P_{s_t,a_t}{V}^{\rref}_{t_n}\right)^2 \notag\\
        &\le \left\{\sum_{n=1}^{{N_{t}}}\eta_{n}^{N_{t}}\left(P_{t_{n}} - P_{s_t,a_t}\right){V}^{\rref}_{t_n}\right\} \left\{\sum_{n=1}^{{N_{t}}}\eta_{n}^{N_{t}}\left(P_{t_{n}} + P_{s_t,a_t}\right){V}^{\rref}_{t_n}\right\}. \label{eq:I42-1}
    \end{align}
    With this upper bound on $I_{4,2}$, we now aim to show
    \begin{equation} \label{eq:I42-WTS}
    	I_{4,2} \leq C_{4,2} \left( \sqrt{\frac{1}{{N_t}(1-\gamma)^4}\log^2\frac{SAT}{\delta}} + \frac{1}{{N_t(1-\gamma)^2}} \log^2\frac{SAT}{\delta}\right),
    \end{equation}
    for some universal constant $C_{4,2}>0$. 
    
    Since \eqref{eq:I42-WTS} is trivially true if $I_{4,2}\leq 0$, it suffices to consider the case where $I_{4,2}> 0$ only. To proceed, we use Lemma \ref{lemma:freedman-application} on the first factor in the product in \eqref{eq:I42-1}. Let us abuse the notation and set
    \begin{align*}
    	W_i := {V}^{\rref}_{i} \qquad \text{and} \qquad u_i(s,a, N) := \frac{1}{N},
    \end{align*}
    as well as
    \begin{align*}
    	|u_i(s,a, N)| =\frac{1}{N} := C_{\mathrm{u}} \qquad \text{and} \qquad
    	\norm{W_i}_\infty \leq \frac{1}{1-\gamma} := C_{\mathrm{w}}. 
    \end{align*} 
    In the same way as our proof for the upper bound on $I_{4,1}$, we Lemma \ref{lemma:freedman-application} with $(N,s,a)=(N_t,s_t,a_t)$ and have
    \begin{equation}
    	\left|\frac{1}{{N_t}}\sum_{n = 1}^{{N_t}} \left( P_{t_n} -P_{s_t,a_t}\right) {V}_{t_n}^{\rref}\right| \lesssim  \sqrt{\frac{\log^2\frac{SAT}{\delta}}{{N_t}(1-\gamma)^2}} + \frac{\log^2\frac{SAT}{\delta}}{{N_t}(1-\gamma)}  \label{eq:I42-1.1}
    \end{equation}
    with probability at least $1-\delta$.

    On the other hand, the second factor in the product in \eqref{eq:I42-1} can be simply bounded as follows:
    \begin{align}
    	\left|\sum_{n=1}^{{N_{t}}}\eta_{n}^{N_{t}}\left(P_{t_{n}} + P_{s_t,a_t}\right){V}^{\rref}_{t_n}\right|  \le \sum_{n=1}^{{N_{t}}}\eta_{n}^{N_{t}} \left( \norm{P_{t_n}}_1 + \norm{P_{s_t,a_t}}_1 \right) \norm{{V}^{\rref}_{t_n}}_{\infty} \leq \frac{2}{1-\gamma},
    	\label{eq:I42-1.2}
    \end{align}
    in which we use the elementary facts $\sum_{n = 1}^{N_t} \eta^{N_t}_n \leq 1$, $\norm{{V}^{\rref}_{t_n}}_{\infty} \leq \frac{1}{1-\gamma}$ and $\norm{P_{t_n}}_1 = \norm{P_{s_t,a_t}}_1=1$.
    
    Substituting \eqref{eq:I42-1.1} and \eqref{eq:I42-1.2} back into \eqref{eq:I42-1}, we can obtain the desired bound \eqref{eq:I42-WTS} for the case $I_{4,2}>0$. 
    Combined this with the trivial case of $I_{4,2}\le 0$, we have established \eqref{eq:I42-WTS}. 

    \paragraph{Step 3: combining the preceding bounds.}

    Let us bring the bounds \eqref{eq:I41-final} and \eqref{eq:I42-WTS} into \eqref{eq:I4-3}. This gives
    \begin{align*} 
    	I_4 \leq |I_{4,1}| + I_{4,2} 
    	\leq C_{4,3}\left\{ \sqrt{\frac{1}{N_t(1-\gamma)^4}\log^2\frac{SAT}{\delta}} + \frac{1}{N_t(1-\gamma)^2}\log^2\frac{SAT}{\delta}\right\} 
    \end{align*}
    for some absolute constant $C_{4,3}>0$. Writing $I_4$ with its definition in \eqref{eq:I4-def} and rearranging the inequality results, and taking square root on both sides, we have 
    \begin{align*}
    	&\quad\ \left\{\sum_{n = 1}^{N_t} \eta^{N_t}_n\Var_{s_t,a_t}\left({V}^{\rref}_{t_n}\right)\right\}^{1/2}\\
        &\leq \left\{ {\sigma}^{\re}_{t_{N_t}+1}(s_t,a_t) - \left({\mu}^{\re}_{t_{N_t}+1}(s_t,a_t) \right)^2\right\}^{1/2} + \sqrt{C_{4,3}} \left(\sqrt{\frac{1}{N_t(1-\gamma)^4}\log^2\frac{SAT}{\delta}} + \frac{1}{N_t(1-\gamma)^2}\log^2\frac{SAT}{\delta}\right)^{1/2}\\
        &\leq \left\{ {\sigma}^{\re}_{t_{N_t}+1}(s_t,a_t) - \left({\mu}^{\re}_{t_{N_t}+1}(s_t,a_t) \right)^2\right\}^{1/2} + \frac{1}{\left(N_{t}\right)^{1/4}(1-\gamma)}\log^{1/2}\frac{SAT}{\delta}+\frac{1}{\left(N_t\right)^{1/2}(1-\gamma)}\log\frac{SAT}{\delta},
    \end{align*}
    in which we also use the result of \eqref{eq:I3-2}.
    
    Substitution into \eqref{eq:I2-1} leads to the desired result \eqref{eq:I2-2}. 

    %
    
\section{Proof of Lemma \ref{lem:xi}}\label{sec:appendix-proof-lem-xi}

    Let us decompose $\zeta_t$ (defined in \eqref{eq:zeta-def}) into three component terms and examine each of them separately:
    \begin{align}\label{eq:zeta-decomposition-def}
        \zeta_t := \underbrace{V_{t-1}(s_t) - Q_{t}(s_t,a_t)}_{\zeta_{t,1}} + \gamma\underbrace{\left(P_{s_t,a_t} - P_t\right)\left(V^\star - V^{\pi_{t}}\right)}_{\zeta_{t,2}} + \gamma\underbrace{\left(V^{\pi_{t+1}}(s_{t+1}) - V^{\pi_t}(s_{t+1})\right)}_{\zeta_{t,3}}.
    \end{align}

    We can find upper bounds for each of these three terms as follows:
    \begin{equation}
        \sum_{t=1}^T \zeta_{t,1} \lesssim \frac{SA}{1-\gamma}\log T + \frac{(SA)^{3/4}T^{1/4}\log^{5/4}\frac{SAT}{\delta}}{(1-\gamma)^{5/4}} + \sqrt{\frac{SA\log^2 T}{1-\gamma}\sum_{t=1}^T \big(V_t(s_{t+1}) - V^{\pi_{t+1}}(s_{t+1})\big)}, \label{eq:xi-1}
    \end{equation}
    
    \begin{align}
        \sum_{t = 1}^{T}\zeta_{t,2} & \lesssim \sqrt{\frac{SA\log\frac{T}{\delta}}{1-\gamma} \sum_{t=1}^{T}\big(V^{\star}(s_{t+1}) - V^{\pi_{t+1}}(s_{t+1})\big)} + \frac{\left(SA\right)^{5/8}T^{1/8}}{(1-\gamma)^{11/8}}\left(\log\frac{SAT}{\delta}\right)^{9/8} \notag\\ & \qquad \qquad + \frac{(SA)^{1/2}T^{1/4}}{(1-\gamma)^{5/4}}\left(\log\frac{T}{\delta}\right)^{3/4} +\frac{SA}{1-\gamma}\log\frac{T}{\delta}, \label{eq:xi-2}
    \end{align}

    \begin{equation}
        \sum_{t=1}^T \zeta_{t,3} \lesssim \frac{SA}{1-\gamma}\log T + \frac{(SA)^{3/4}T^{1/4}\log^{5/4}\frac{SAT}{\delta}}{(1-\gamma)^{5/4}} + \sqrt{\frac{SA\log^2 T}{1-\gamma}\sum_{t=1}^T \big(V_t(s_{t+1}) - V^{\pi_{t+1}}(s_{t+1})\big)}. \label{eq:xi-3}
    \end{equation}
    The proofs for these inequalities are deferred to Appendix \ref{sec:zeta-1-proof}, \ref{sec:zeta-2-proof} and \ref{sec:zeta-3-proof}.

    Combining \eqref{eq:xi-1}, \eqref{eq:xi-2} and \eqref{eq:xi-3}, we can obtain the final bound: with probability at least $1-\delta$,
    \begin{align}
        \sum_{t=1}^T \zeta_{t} &= \sum_{t=1}^T \zeta_{t,1} + \zeta_{t,2} + \zeta_{t,3}\notag\\
        &\lesssim \frac{SA}{1-\gamma}\log T + \frac{(SA)^{3/4}T^{1/4}\log^{5/4}\frac{SAT}{\delta}}{(1-\gamma)^{5/4}} + \sqrt{\frac{SA\log^2 T}{1-\gamma}\sum_{t=1}^T \big(V_t(s_{t+1}) - V^{\pi_{t+1}}(s_{t+1})\big)} \notag\\ & \quad + \sqrt{\frac{SA\log\frac{T}{\delta}}{1-\gamma} \sum_{t=1}^{T}\big(V^{\star}(s_{t+1}) - V^{\pi_{t+1}}(s_{t+1})\big)} + \frac{\left(SA\right)^{5/8}T^{1/8}}{(1-\gamma)^{11/8}}\left(\log\frac{SAT}{\delta}\right)^{9/8} \notag\\ & \quad +\frac{(SA)^{1/2}T^{1/4}}{(1-\gamma)^{5/4}}\left(\log\frac{T}{\delta}\right)^{3/4} +\frac{SA}{1-\gamma}\log\frac{T}{\delta} \notag\\
        &\overset{(\mathrm{i})}{\lesssim} \frac{SA}{1-\gamma}\log T + \frac{(SA)^{3/4}T^{1/4}\log^{5/4}\frac{SAT}{\delta}}{(1-\gamma)^{5/4}} + \sqrt{\frac{SA\log^2 T}{1-\gamma}\sum_{t=2}^T \big(V_{t-1}(s_{t}) - V^{\pi_{t}}(s_{t})\big)} \notag\\ & \quad + \sqrt{\frac{SA\log\frac{T}{\delta}}{1-\gamma} \sum_{t=2}^{T}\big(V^{\star}(s_{t}) - V^{\pi_{t}}(s_{t})\big)} + \frac{\left(SA\right)^{5/8}T^{1/8}}{(1-\gamma)^{11/8}}\left(\log\frac{SAT}{\delta}\right)^{9/8} \notag\\ & \quad +\frac{(SA)^{1/2}T^{1/4}}{(1-\gamma)^{5/4}}\left(\log\frac{T}{\delta}\right)^{3/4} +\frac{SA}{1-\gamma}\log\frac{T}{\delta} \notag\\
        &\overset{(\mathrm{ii})}{\le} \frac{SA}{1-\gamma}\log T + \frac{(SA)^{3/4}T^{1/4}\log^{5/4}\frac{SAT}{\delta}}{(1-\gamma)^{5/4}} + \sqrt{\frac{SA\log^2 T}{1-\gamma}\sum_{t=1}^T \big(V_{t-1}(s_{t}) - V^{\pi_{t}}(s_{t})\big)} \notag\\ & \quad + \sqrt{\frac{SA\log\frac{T}{\delta}}{1-\gamma} \sum_{t=1}^{T} \big(V^{\star}(s_{t}) - V^{\pi_{t}}(s_{t})\big)} + \frac{\left(SA\right)^{5/8}T^{1/8}}{(1-\gamma)^{11/8}}\left(\log\frac{SAT}{\delta}\right)^{9/8} \notag\\ & \quad +\frac{(SA)^{1/2}T^{1/4}}{(1-\gamma)^{5/4}}\left(\log\frac{T}{\delta}\right)^{3/4} +\frac{SA}{1-\gamma}\log\frac{T}{\delta} \notag\\
        &\overset{(\mathrm{iii})}{\lesssim} \sqrt{\frac{SAT}{1-\gamma}\log\frac{SAT}{\delta}} + \frac{SA\log\frac{SAT}{\delta}}{(1-\gamma)^2} + \frac{SA}{(1-\gamma)^{5/2}}\log^2\frac{SAT}{\delta} + \frac{(SA)^{1/2}}{(1-\gamma)^{2}}\log\frac{T}{\delta} + \frac{SA}{1-\gamma}\log\frac{T}{\delta} \notag\\ & \quad + \sqrt{\frac{SA\log^2 T}{1-\gamma}\sum_{t=1}^T \big(V_t(s_{t+1}) - V^{\pi_{t+1}}(s_{t+1})\big)} + \sqrt{\frac{SA\log\frac{T}{\delta}}{1-\gamma} \sum_{t=1}^{T}\big(V^{\star}(s_{t+1}) - V^{\pi_{t+1}}(s_{t+1})\big)} \notag\\
        &\lesssim \sqrt{\frac{SAT}{1-\gamma}\log\frac{SAT}{\delta}} + \frac{SA\log^2\frac{SAT}{\delta}}{(1-\gamma)^{5/2}} + \sqrt{\frac{SA\log^2 T}{1-\gamma}\sum_{t=1}^T \big(V_t(s_{t+1}) - V^{\pi_{t+1}}(s_{t+1})\big)} \notag\\ & \quad + \sqrt{\frac{SA\log\frac{T}{\delta}}{1-\gamma} \sum_{t=1}^{T}\big(V^{\star}(s_{t+1}) - V^{\pi_{t+1}}(s_{t+1})\big)},
    \end{align}
    in which (i) is by upper-bounding $V_T(s_{T+1}) - V^{\pi_{T+1}}(s_{T+1})$ and $V^{\star}(s_{T+1}) - V^{\pi_{T+1}}(s_{T+1})$ with the addition of $\frac{2}{1-\gamma}$; (ii) is due to $V_0(s_1) - V^{\pi_{1}}(s_{1}) \ge 0$ and $V^\star(s_1) - V^{\pi_{1}}(s_{1}) \ge 0$, which can be obtained from the fact $V_0(s) = \frac{1}{1-\gamma}$ and $V^\star(s) \ge V^{\pi_{1}}(s)$ for all $s\in\cS$; (iii) is by the basic inequality $2ab\leq a^2 + b^2$.

\subsection{Bounding $\zeta_{t,1}$}\label{sec:zeta-1-proof}

First, we define some quantities to streamline our analysis.

For every $(s,a) \in \cS\times \cA$, let $\tau_i(s,a)$ denote the step index that $\qdiff(s,a) > \frac{1}{1-\gamma}$ for the $i$-th time, that is, Line \ref{alg-line:q-cond} in Algorithm \ref{alg:main} (equivalently, Line \ref{alg-t-line:q-cond} in Algorithm \ref{alg:main-t}) is triggered for the $i$-th time. (Note that it is the index of the step in which $\qdiff(s,a) > \frac{1}{1-\gamma}$ occurs and not the time index of the $\theta$ that gets updated in this step. The updated $\theta$ is $\theta_{t+1}$ and the update happens after Line \ref{alg-t-line:q-cond}.) Mathematically, this is 
    \begin{equation}
        \tau_i(s,a) := \min\left\{t > \tau_{i-1}(s,a) ~:~ \qdiff_{t}(s_t,a_t) + Q_{M_{t}(s_t, a_t)}(s_t, a_t) - Q_{t+1}(s_t, a_t) > \frac{1}{1-\gamma}\right\} 
    \end{equation}
    with $\tau_0(s,a) = 0$. 
    
For every $(s,a) \in \cS\times \cA$, let $K_t(s,a)$ be the total number of times such that $\qdiff(s,a) > \frac{1}{1-\gamma}$ until the end of $t$-th step. Mathematically, this is 
\begin{equation}
    K_t(s,a) := \max\left\{i ~:~ \tau_i(s,a) \le t\right\}. \label{eq:K-def}
\end{equation}
For notational simplicity, we abbreviate $\tau_i(s,a)$ as $\tau_i$, $K_t(s,a)$ as $K_t$ and $N_t(s,a)$ as $N_t$ when the context is clear.

We can express $\zeta_{t,1}$ with $K_t(s,a)$ after the following manipulation:
    \begin{align}
        \sum_{t=1}^T \zeta_{t,1} &= \sum_{t=1}^T \big(V_{t-1}(s_t) - Q_{t}(s_t,a_t)\big) \notag\\
        &\overset{(\mathrm{i})}{\le} \sum_{t=1}^T \big(Q_{Z_{t-1}}(s_t,a_t) - Q_{t}(s_t,a_t)\big) \notag\\
        &\overset{(\mathrm{ii})}{=} \sum_{s,a}\sum_{i=0}^{K_T(s,a)}\sum_{n=N_{\tau_i}}^{N_{\tau_{i+1}}-1}\left(Q_{Z_{t_n-1}}(s,a) - Q_{t_n}(s,a)\right)\notag\\
        &= \sum_{s,a}\sum_{i=0}^{K_T(s,a)}\sum_{n=N_{\tau_i}}^{N_{\tau_{i+1}}-1}\left(Q_{Z_{t_n-1}}(s,a) - Q_{Z_{t_n+1}}(s,a) + Q_{Z_{t_n+1}}(s,a) - Q_{t_n}(s,a)\right)\notag\\
        &\overset{(\mathrm{iii})}{\le} \sum_{s,a}\sum_{i=0}^{K_T(s,a)}\frac{3}{1-\gamma}\notag\\
        &\lesssim \frac{1}{1-\gamma}\sum_{s,a}K_T(s,a). \label{eq:xi-1-1}
    \end{align}
    In the inequalities above, (i) is because $V_{t-1}(s_t) \le V_{Z_{t-1}}(s_t) = Q_{Z_{t-1}}(s_t,a_t)$ by the monotonicity and $a_t$'s selection rule in Line \ref{alg-t-line:a} of Algorithm \ref{alg:main-t}. (ii) is because we can partition the overall length-$T$ trajectory into $K_T(s,a)$ intervals for each $(s,a)$-pair. (iii) is obtained by combining \eqref{eq:xi-1-1.1} and \eqref{eq:xi-1-1.2}, which are bounds on the two halves of the quantity of interest respectively. Since $Q_{Z_{t_n-1}}(s,a) \neq Q_{Z_{t_n+1}}(s,a)$ at most twice in each interval $N_{\tau_i} \le n \le N_{\tau_{i+1}}-1$, we can use the fact $\norm{Q_{Z_{t_n-1}} - Q_{Z_{t_n+1}}}_\infty \le \frac{1}{1-\gamma}$ at these two points and have
    \begin{align}
        &\quad\ \sum_{s,a}\sum_{i=0}^{K_T(s,a)}\sum_{n=N_{\tau_i}}^{N_{\tau_{i+1}}-1}\left(Q_{Z_{t_n-1}}(s,a) - Q_{Z_{t_n+1}}(s,a)\right) \notag\\
        &\le \sum_{s,a}\sum_{i=0}^{K_T(s,a)}\frac{2}{1-\gamma}. \label{eq:xi-1-1.1}
    \end{align}
    On the other hand, the condition in Line \ref{alg-t-line:q-cond} of Algorithm \ref{alg:main-t} tells us that
    \begin{align}
        &\quad\ \sum_{s,a}\sum_{i=0}^{K_T(s,a)}\sum_{n=N_{\tau_i}}^{N_{\tau_{i+1}}-1}\left(Q_{Z_{t_n+1}}(s,a) - Q_{t_n}(s,a)\right) \notag\\
        &\overset{(\mathrm{a})}{\le}  \sum_{s,a}\sum_{i=0}^{K_T(s,a)}\sum_{n=N_{\tau_i}}^{N_{\tau_{i+1}}-1}\left(Q_{\tau_i+1}(s,a) - Q_{t_n}(s,a)\right) \notag\\
        &\overset{(\mathrm{b})}{\le}  \sum_{s,a}\sum_{i=0}^{K_T(s,a)}\sum_{n=N_{\tau_i}}^{N_{\tau_{i+1}}-1}\left(Q_{\tau_i+1}(s,a) - Q_{t_n+1}(s,a)\right) \notag\\
        &\overset{(\mathrm{c})}{\le} \sum_{s,a}\sum_{i=0}^{K_T(s,a)}\frac{1}{1-\gamma}, \label{eq:xi-1-1.2}
    \end{align}
    in which (a) is because $Z$ is updated when Line \ref{alg-t-line:a} of Algorithm \ref{alg:main-t} is triggered by any $(s,a)$-pair, which is more often than $\tau_i$ that depends on a specific $(s,a)$-pair. (b) is by the monotonicity of $Q_{t}$. (c) is due to $\qdiff_{\tau_{i+1}} \le \frac{1}{1-\gamma}$. This is because the condition in Line \ref{alg-t-line:q-cond} of Algorithm \ref{alg:main-t} is only triggered after the addition of $Q_{M_{t}(s_t, a_t)}(s_t, a_t) - Q_{t+1}(s_t, a_t)$ to $\qdiff_{t}$ when $t = \tau_{i+1}$.

In the following, we switch our focus to bounding $\sum_{s,a}K_T(s,a)$. Let us define a new quantity
\begin{equation}
    \rho_i(s,a) := (1-\gamma)\left(Q_{\tau_{i}(s,a)+1}(s,a) - Q^\star(s,a)\right). \label{eq:rho-def}
\end{equation}
We can observe 
\begin{equation}
    \rho_i(s,a) \ge \rho_{i+1}(s,a) \label{eq:rho-prop1}
\end{equation}
by the monotonicity of $Q_t$. We can obtain another property about $\rho_i(s,a)$, which will be used in later analysis:
    \begin{align}
        &\quad \left(\rho_i(s,a) - \rho_{i+1}(s,a)\right)\left(N_{\tau_{i+1}} - N_{\tau_i}\right)\notag\\
        &= (1-\gamma)\left(Q_{\tau_{i}+1}(s,a) - Q_{\tau_{i+1}+1}(s,a)\right)\left(N_{\tau_{i+1}} - N_{\tau_i}\right)\notag\\
        &\ge 1,\label{eq:rho-prop2}
    \end{align}
    in which the inequality is obtained via the following logic: $\qdiff_{t}(s,a) + Q_{M_{t}(s, a)}(s, a) - Q_{t+1}(s, a)$ in the condition Line \ref{alg-t-line:q-cond} of Algorithm \ref{alg:main-t} tells us $\qdiff_{\tau_{i+1}}(s,a) + Q_{M_{\tau_{i+1}}(s, a)}(s, a) - Q_{\tau_{i+1}+1}(s, a) \ge \frac{1}{1-\gamma}$. This quantity can be viewed as a summation of $(N_{\tau_{i+1}} - N_{\tau_i})$ non-negative summands in the form of $Q_{M_{t}(s, a)}(s, a) - Q_{t+1}(s, a)$, with $t$ being every step index when $(s,a)$ is visited between step $\tau_i+1$ and step $\tau_{i+1}$. $Q_{M_{\tau_{i}}(s,a)}(s,a) - Q_{\tau_{i+1}+1}(s,a) = Q_{\tau_{i}+1}(s,a) - Q_{\tau_{i+1}+1}(s,a)$ is the largest among these summands. The product of the largest summand and the total number of summands is larger than the summation itself. 

To continue our analysis, we introduce a lemma that provided an upper bound on $Q_t(s_{t},a_{t}) - Q^\star(s_{t},a_{t})$, whose proof is deferred to Appendix \ref{sec:appendix-proof-lem:Q_t-upper-bound-crude}. Notice that Lemma \ref{lem:Q_t-upper-bound-crude} is just a looser version of Lemma \ref{lem:Q_t-upper-bound}. The reason that we introduce Lemma \ref{lem:Q_t-upper-bound-crude} instead of invoking Lemma \ref{lem:Q_t-upper-bound} is that a looser bound suffices here, and the proof of Lemma \ref{lem:Q_t-upper-bound} actually relies on Lemma \ref{lem:xi}.

\begin{lemma}\label{lem:Q_t-upper-bound-crude}
    Fix $\delta \in (0,1)$. Suppose that $c_b$ is chosen to be a sufficiently large universal constant in Algorithm \ref{alg:main}. Then with probability at least $1 - \delta$,
    \begin{align}
        &\sum_{t=1}^T Q_t(s_{t},a_{t}) - Q^\star(s_{t},a_{t}) \leq \frac{\gamma(3-\gamma)}{2}\sum_{t=1}^T\left(V_t(s_{t+1}) - V^{\pi_{t+1}}(s_{t+1})\right)
        + 8c_b\sqrt{\frac{SAT}{(1-\gamma)^3}\log\frac{SAT}{\delta}} + \frac{SA}{1-\gamma}.
    \end{align}
\end{lemma}

Thus, we can prove
    \begin{align}
        &\quad \sum_{s,a}\sum_{i=0}^{K_T(s,a)}\frac{\rho_{i+1}(s,a)}{\rho_i(s,a) - \rho_{i+1}(s,a)} \notag\\
        &\overset{(\mathrm{i})}{\le} \sum_{s,a}\sum_{i=0}^{K_T(s,a)}\rho_{i+1}(s,a)\left(N_{\tau_{i+1}} - N_{\tau_i}\right) \notag\\
        &\overset{(\mathrm{ii})}{=} (1-\gamma)\sum_{s,a}\sum_{i=0}^{K_T(s,a)}\sum_{n=N_{\tau_i}}^{N_{\tau_{i+1}}-1}\left(Q_{\tau_{i+1}+1}(s,a) - Q^\star(s,a)\right)\notag\\
        &\overset{(\mathrm{iii})}{\le} (1-\gamma)\sum_{s,a}\sum_{i=0}^{K_T(s,a)}\sum_{n=N_{\tau_i}}^{N_{\tau_{i+1}}-1}\left(Q_{t_n}(s,a) - Q^\star(s,a)\right)\notag\\
        &= (1-\gamma)\sum_{t=1}^T \left(Q_{t}(s_t,a_t) - Q^\star(s_t,a_t)\right)\notag\\
        &\overset{(\mathrm{iv})}{\lesssim} SA + \sqrt{\frac{SAT\log\frac{SAT}{\delta}}{1-\gamma}} + (1-\gamma)\sum_{t=1}^T\left(V_t(s_{t+1}) - V^{\pi_{t+1}}(s_{t+1})\right).\label{eq:rho-diff2}
    \end{align}
    Above, (i) is due to \eqref{eq:rho-prop2}. (ii) is by the definition of $\rho_{i+1}(s,a)$ in \eqref{eq:rho-def}. (iii) is by the monotonicity of $Q_t$. (iv) is by Lemma \ref{lem:Q_t-upper-bound-crude} and the fact $\frac{\gamma(3-\gamma)}{2} \le 1$.

    Finally, we are ready to bound $\sum_{s,a}K_T(s,a)$. Define a sufficiently large index set of $\frac{1}{\rho_i(s,a) - \rho_{i+1}(s,a)}$ that are on the same order. Mathematically, this is 
    \begin{equation}
        \cI(s,a) := \left\{i\le K_T(s,a) ~:~ \frac{1}{\rho_i(s,a) - \rho_{i+1}(s,a)} \asymp \kappa(s,a)\right\}, \label{eq:I-set-def}
    \end{equation}
    in which $\kappa(s,a)$ satisfies $|\cI(s,a)| \gtrsim \frac{K_T(s,a)}{\log T}$. Before continuing, let us explain the reason why there exists such a set $\cI(s,a)$: 
    
    By \eqref{eq:rho-prop2}, we have
    \begin{align}
        \frac{1}{\rho_i(s,a) - \rho_{i+1}(s,a)} \le N_{\tau_{i+1}} - N_{\tau_i} \le T.
    \end{align}
    Moreover, $\norm{Q_{\tau_{i}+1} - Q_{\tau_{i+1}+1}}_\infty \le \frac{1}{1-\gamma}$ , and this implies 
    \begin{equation}
        \rho_i(s,a) - \rho_{i+1}(s,a) \le 1. \label{eq:rho-diff}
    \end{equation}
    Thus, we can conclude
    \begin{equation*}
        1 \le \frac{1}{\rho_i(s,a) - \rho_{i+1}(s,a)} \le T.
    \end{equation*}
    Consider a sequence of sets $\{\cI_m(s,a)\}_{m=1}^{\lceil\log T\rceil}$, each element of which is defined as
    \begin{equation}
        \cI_m(s,a) := \left\{i\le K_T(s,a) ~:~ \frac{T}{2^m} < \frac{1}{\rho_i(s,a) - \rho_{i+1}(s,a)} \le \frac{T}{2^{m-1}}\right\}.
    \end{equation}
    By the pigeonhole principle, there exists at least one set in this sequence with cardinality at least $\frac{K_T(s,a)}{\log T}$. Note that $\frac{1}{\rho_i(s,a) - \rho_{i+1}(s,a)}$'s in the same set are within a factor of $2$.

    With such set $\cI(s,a)$, we can prove
    \begin{align}
        &\quad \sum_{s,a}\sum_{i=0}^{K_T(s,a)}\frac{\rho_{i+1}(s,a)}{\rho_i(s,a) - \rho_{i+1}(s,a)}\notag\\
        &\ge \sum_{s,a}\sum_{i\in \cI(s,a)}\frac{\rho_{i+1}(s,a)}{\rho_i(s,a) - \rho_{i+1}(s,a)}\notag\\
        &\overset{(\mathrm{i})}{\asymp} \sum_{s,a}\sum_{i\in \cI(s,a)} \kappa(s,a)\rho_{i+1}(s,a)\notag\\
        &\overset{(\mathrm{ii})}{=} \sum_{s,a}\sum_{\overset{i_n\in \cI(s,a)}{i_1 > i_2 > \cdots > i_{|\cI(s,a)|}}} \kappa(s,a)\rho_{i_n+1}(s,a)\notag\\
        &\overset{(\mathrm{iii})}{\ge} \sum_{s,a}\sum_{\overset{i_n\in \cI(s,a)}{i_1 > i_2 > \cdots > i_{|\cI(s,a)|}}} \left(\kappa(s,a)\rho_{i_n}(s,a) - 1\right)\notag\\
        &\overset{(\mathrm{iv})}{=} \sum_{s,a}\kappa(s,a)\left(\sum_{j=n}^{|\cI(s,a)|} \rho_{i_j}(s,a) - \rho_{i_{j+1}}(s,a)\right) - \left(\sum_{s,a}\sum_{\overset{i_n\in \cI(s,a)}{i_1 > i_2 > \cdots > i_{|\cI(s,a)|}}}1\right) \notag\\
        &\overset{(\mathrm{v})}{\ge} \sum_{s,a}\kappa(s,a)\left(\sum_{j=n}^{|\cI(s,a)|} \rho_{i_j}(s,a) - \rho_{i_{j}+1}(s,a)\right) - \left(\sum_{s,a}\sum_{\overset{i_n\in \cI(s,a)}{i_1 > i_2 > \cdots > i_{|\cI(s,a)|}}}1\right) \notag\\
        &\overset{(\mathrm{vi})}{\asymp} \sum_{s,a}\sum_{\overset{i_n\in \cI(s,a)}{i_1 > i_2 > \cdots > i_{|\cI(s,a)|}}} \left(\kappa(s,a)\frac{n}{\kappa(s,a)} - 1\right)\notag\\
        &\gtrsim \sum_{s,a}\sum_{i=1}^{|\cI(s,a)|-1} i \notag\\
        &\ge \sum_{s,a}\big(\frac{1}{4}|\cI(s,a)|^2 - \frac{1}{4}\big) \notag\\
        &\asymp \sum_{s,a}\max\left\{|\cI(s,a)|^2 - 1, 0\right\}.\label{eq:rho-set}
    \end{align}
    Above, (i) is by the definition of $\cI(s,a)$ in \eqref{eq:I-set-def}. In (ii), we order the elements in $\cI(s,a)$ descendingly and denote the $n$-th largest element with $i_n$. (iii) is due to \eqref{eq:rho-diff}. In (iv), We can let $\rho_{|\cI(s,a)|+1}(s,a) = 0$. (v) is due to the monotonicity of $Q_t$. (vi) is by the definition of $\cI(s,a)$. 

    By some algebra and \eqref{eq:rho-diff2} and \eqref{eq:rho-set}, we can conclude
    \begin{align}
        &\quad\ \sum_{s,a}\big(|\cI(s,a)|-1\big)\notag\\
        &\overset{(\mathrm{i})}{\le} \sqrt{SA\sum_{s,a}\big(|\cI(s,a)| - 1\big)^2} \notag\\
        &\overset{(\mathrm{ii})}{\lesssim} \sqrt{SA\sum_{s,a}\max\left\{|\cI(s,a)|^2 - 1, 0\right\}} \notag\\
        &\overset{(\mathrm{iii})}{\lesssim} \sqrt{SA}\cdot\sqrt{SA + \sqrt{\frac{SAT\log\frac{SAT}{\delta}}{1-\gamma}} + (1-\gamma)\sum_{t=1}^T\left(V_t(s_{t+1}) - V^{\pi_{t+1}}(s_{t+1})\right)} \notag\\
        &\overset{(\mathrm{iv})}{\le} SA + \frac{(SA)^{3/4}T^{1/4}\log^{1/4}\frac{SAT}{\delta}}{(1-\gamma)^{1/4}} + \sqrt{(1-\gamma)SA\sum_{t=1}^T \big(V_t(s_{t+1}) - V^{\pi_{t+1}}(s_{t+1})\big)}. \label{eq:xi-1-2}
    \end{align}
    Above, (i) and (ii) are by the fact that $(\sum_{i=1}^n a_i)^2 \le n\sum_{i=1}^n a_i^2$. (iii) is obtained by combining \eqref{eq:rho-diff2} and \eqref{eq:rho-set}. (iv) is by the algebraic fact $\sqrt{a+b} \le \sqrt{a}+\sqrt{b}$.

    Finally, by the definition in \eqref{eq:I-set-def}, we have
    \begin{align}
        &\quad\ \sum_{s,a}K_T(s,a) \lesssim \sum_{s,a}|\cI(s,a)|\log T \notag\\
        &\lesssim SA\log T + \frac{(SA)^{3/4}T^{1/4}\log^{5/4}\frac{SAT}{\delta}}{(1-\gamma)^{1/4}} + \log T\sqrt{(1-\gamma)SA\sum_{t=1}^T \big(V_t(s_{t+1}) - V^{\pi_{t+1}}(s_{t+1})\big)}, \label{eq:K-bound}
    \end{align}
    and by \eqref{eq:xi-1-1}, we can obtain the desired final bound
    \begin{align}
        \sum_{t=1}^T \zeta_{t,1} \lesssim \frac{SA}{1-\gamma}\log T + \frac{(SA)^{3/4}T^{1/4}\log^{5/4}\frac{SAT}{\delta}}{(1-\gamma)^{5/4}} + \sqrt{\frac{SA\log^2 T}{1-\gamma}\sum_{t=1}^T \big(V_t(s_{t+1}) - V^{\pi_{t+1}}(s_{t+1})\big)}. \label{eq:xi-1-bound}
    \end{align}

\subsubsection{Proof of Lemma \ref{lem:Q_t-upper-bound-crude}}\label{sec:appendix-proof-lem:Q_t-upper-bound-crude}

\begin{proof}
    To prove this lemma, it suffices to find an upper bound on $Q^{\UCB}_t(s_t, a_t) - Q^\star(s_t, a_t)$. Let us again adopt the abbreviation $N_t := N_t(s,a)$ when $(s,a)$ is clear from the context. By \eqref{eq:Q_UCB-Q*} and \eqref{eq:Jin-Hoeffding}, for any $(s,a)\in\cS\times\cA$,
    \begin{align}
        &\quad\ Q^{\UCB}_t(s, a) - Q^\star(s, a) \notag\\
        &= \eta_0^{N_{t-1}} \left(Q^{\UCB}_1(s, a) - Q^\star(s, a)\right) + \sum_{n=1}^{N_{t-1}} \eta_n^{N_{t-1}} \left(\gamma\left(V^{t_n}(s_{t_n+1}) - V^\star(s_{t_n+1})\right) + \gamma\left(P_{t_n} - P_{s,a}\right)V^\star + b_n\right) \notag\\
        &\le \frac{\eta_0^{N_{t-1}}}{1-\gamma} + \sum_{n=1}^{N_{t-1}} \eta_n^{N_{t-1}} \Big(\gamma\big(V^{t_n}(s_{t_n+1}) - V^\star(s_{t_n+1})\big)+ 2b_n\Big) \notag\\
        &\le \frac{\eta_0^{N_{t-1}}}{1-\gamma} + 4c_b\sqrt{\frac{\log\frac{SAT}{\delta}}{(1-\gamma)^3N_{t-1}}} + \sum_{n=1}^{N_{t-1}} \eta_n^{N_{t-1}} \gamma\left(V^{t_n}(s_{t_n+1}) - V^\star(s_{t_n+1})\right) , \label{eq:crude-Q-1}
    \end{align}
    in which the last inequality is because $\sum_{n=1}^N\frac{\eta_n^N}{\sqrt{N}} \le \frac{2}{\sqrt{N}}$ by \eqref{eq:learning-rate2} in Lemma \ref{lem:learning-rate}.

    Now, we are ready to work towards the desired upper bound:
    \begin{align}
        &\quad\ \sum_{t=1}^T Q_t(s_{t},a_{t}) - Q^\star(s_{t},a_{t}) \notag\\
        &\overset{(\mathrm{i})}{\le} \sum_{t=1}^T Q^{\UCB}_t(s, a) - Q^\star(s, a) \notag\\
        &\overset{(\mathrm{ii})}{\le} \sum_{t=1}^T\frac{\eta_0^{N_{t-1}(s_t,a_t)}}{1-\gamma} + 4c_b\sqrt{\frac{\log\frac{SAT}{\delta}}{(1-\gamma)^3N_{t-1}(s_t,a_t)}} + \sum_{n=1}^{N_{t-1}(s_t,a_t)} \eta_n^{N_{t-1}(s_t,a_t)} \gamma\left(V^{t_n}(s_{t_n+1}) - V^\star(s_{t_n+1})\right) \label{eq:crude-Q-2}\\
        &\overset{(\mathrm{iii})}{\le} \frac{SA}{1-\gamma} + 8c_b\sqrt{\frac{SAT\log\frac{SAT}{\delta}}{(1-\gamma)^3}} + \frac{\gamma(3-\gamma)}{2}\sum_{t=1}^T \big(V_{t}(s_{t+1}) - V^\star(s_{t+1})\big).
    \end{align}
    In the inequalities above, (i) is because for any $(s,a)\in\cS\times\cA$, $Q^{\UCB}_t(s,a) \ge Q_t(s,a)$ by Line \ref{alg-t-line:Q} of Algorithm \ref{alg:main-t}. (ii) is by \eqref{eq:crude-Q-1}. (iii) is obtained by the upper bounds on each of the three terms in \eqref{eq:crude-Q-2}. These three terms can be bounded as follows:

    The first term in \eqref{eq:crude-Q-2} can be bounded with $\frac{SA}{1-\gamma}$ due to \eqref{eq:learning-rate-def2} and the fact that for each $(s,a)\in\cS\times\cA$, $N_{t-1}(s,a) = 0$ for only one instance of $t$ and will become a positive integer thereafter. 
    
    For the second term in \eqref{eq:crude-Q-2}, since $\sum_{(s,a)}N_{T-1}(s,a) \le T$, we have
    \begin{equation*}
        \sum_{t=1}^T\sqrt{\frac{1}{N_{t-1}(s_t, a_t)}} \leq \sum_{(s,a)}\sum_{n=1}^{N_{T-1}(s, a)}\sqrt{\frac{1}{n}} \le \sum_{(s,a)}2\sqrt{N_{T-1}(s, a)} \le 2\sqrt{SAT}.
    \end{equation*}
    Lastly, for the third term in \eqref{eq:crude-Q-2}, we have
    \begin{align*}
        &\quad \sum_{t=1}^T\sum_{n=1}^{N_{t-1}(s_t,a_t)} \eta_n^{N_{t-1}(s_t,a_t)} \gamma\left(V_{t_n(s_t,a_t)}(s_{t_n(s_t,a_t)+1}) - V^\star(s_{t_n(s_t,a_t)+1})\right)\\
        &= \sum_{(s,a)}\sum_{k=1}^{N_{T-1}(s, a)}\sum_{n=1}^k\eta_n^{k}\gamma\left(V_{t_n(s,a)}(s_{t_n(s,a)+1}) - V^\star(s_{t_n(s,a)+1})\right)\\
        &= \sum_{(s,a)}\sum_{k=1}^{N_{T-1}(s, a)}\sum_{n=k}^{N_{T-1}(s, a)}\eta_k^{n}\gamma\left(V_{t_k(s,a)}(s_{t_k(s,a)+1}) - V^\star(s_{t_k(s,a)+1})\right)\\
        &= \sum_{t=1}^T \sum_{n=N_t(s_t,a_t)}^{N_{T-1}(s_t,a_t)}\eta_{N_t(s_t,a_t)}^{n}\gamma\big(V_{t}(s_{t+1}) - V^\star(s_{t+1})\big)\\
        &\le \frac{\gamma(3-\gamma)}{2}\sum_{t=1}^T \big(V_{t}(s_{t+1}) - V^\star(s_{t+1})\big),
    \end{align*}
    in which the penultimate step is obtained because $t = t_k(s,a)$ implies $k = N_t(s,a)$, and the last step is because $\sum_{n=N_t(s_t,a_t)}^{N_{T-1}(s_t,a_t)}\eta_{N_t(s_t,a_t)}^n \le \frac{3-\gamma}{2}$ by Lemma \ref{lem:learning-rate}.
\end{proof}
    
\subsection{Bounding $\zeta_{t,2}$}\label{sec:zeta-2-proof}

We will prove \eqref{eq:xi-2-final} with Lemma \ref{lemma:freedman-application2}. Set
    \begin{align*}
    	W_i \coloneqq V^{\star} - V^{\pi_{i}}
    	\qquad \text{and} \qquad
    	u_i(s_i,a_i) \coloneqq 1. 
    \end{align*}
    Since $V^{\rref}_{i}(s),  V^{\star}(s)\in [0,\frac{1}{1-\gamma}]$, it can be seen that
    \begin{align*}
       \left|u_i(s_i,a_i) \right| = 1 \eqqcolon C_{\mathrm{u}}
    	\qquad \text{and} \qquad
    	  \norm{W_i}_\infty \leq \frac{1}{1-\gamma} \eqqcolon C_{\mathrm{w}}. 
    \end{align*}
    By Lemma \ref{lemma:freedman-application2}, with probability at least $1- \delta/2$, one has
    \begin{align}
     \sum_{t=1}^T \zeta_{t,2} &= \sum_{t = 1}^{T} \left(P_{s_t,a_t} - P_t\right)\left(V^\star - V^{\pi_{t}}\right) \notag\\
     & \lesssim \sqrt{C_{\mathrm{u}}^{2}C_{\mathrm{w}}SA\sum_{i=1}^{T}\EE_{i-1}\left[P_{i}W_{i}\right]\log\frac{T}{\delta}} + C_{\mathrm{u}}C_{\mathrm{w}}SA\log\frac{T}{\delta} \notag \\
     &= \sqrt{\frac{SA}{1-\gamma}\sum_{i=1}^{T}\EE_{i-1}\left[P_{i}\left(V^{\star} - V^{\pi_{i}}\right)\right]\log\frac{T}{\delta}}+\frac{SA}{1-\gamma}\log\frac{T}{\delta} \notag \\
     &= \sqrt{\frac{SA}{1-\gamma} \left\{ \sum_{t=1}^{T}P_{s_{t},a_{t}}\left(V^{\star} - V^{\pi_{t}}\right) \right\} \log\frac{T}{\delta}}+\frac{SA}{1-\gamma}\log\frac{T}{\delta}, \label{eq:xi-2-1}
    \end{align}
    in which the last step is because $V^{\star} - V^{\pi_{t}}$ is deterministic when conditioned on the $(i-1)$-th step. 

    In order to obtain the desired bound, more manipulation is needed on $\sum_{t=1}^{T}P_{s_{t},a_{t}}\big(V^{\star} - V^{\pi_{t}}\big)$ in \eqref{eq:xi-2-1}. Towards this end, we first invoke the Azuma-Hoeffding inequality (Theorem \ref{thm:hoeffding}), which leads to
    \begin{align}
    \left|\sum_{t=1}^{T}\left(P_{t}-P_{s_{t},a_{t}}\right)\left(V^{\star} - V^{\pi_{t}}\right)\right| & \lesssim \sqrt{\frac{T}{(1-\gamma)^2}\log \frac{1}{\delta}}, \label{eq:xi-2-2}
    \end{align}
    with probability at least $1-\delta/2$, because $V^{\star} - V^{\pi_{t}}$ is deterministic when conditioned on the $(i-1)$-th step and $\norm{V^{\star} - V^{\pi_{t}}}_\infty \le \frac{1}{1-\gamma}$. 
    
    With \eqref{eq:xi-2-2} established, we can conclude
    \begin{align}
    &\quad\ \sum_{t=1}^{T}P_{s_{t},a_{t}}\left(V^{\star} - V^{\pi_{t}}\right) \notag\\
    &= \sum_{t=1}^{T}P_{t}\left(V^{\star} - V^{\pi_{t}}\right) + \sum_{t=1}^{T}\left(P_{s_{t},a_{t}} - P_{t}\right)\left(V^{\star} - V^{\pi_{t}}\right) \notag\\
    &\lesssim \sum_{t=1}^{T}\big(V^{\star}(s_{t+1}) - V^{\pi_{t}}(s_{t+1})\big) + \sqrt{\frac{T}{(1-\gamma)^2}\log \frac{1}{\delta}} \notag\\
    &= \sum_{t=1}^{T}\big(V^{\star}(s_{t+1}) - V^{\pi_{t+1}}(s_{t+1})\big) + \sum_{t=1}^{T}\big(V^{\pi_{t+1}}(s_{t+1}) - V^{\pi_{t}}(s_{t+1})\big) + \sqrt{\frac{T}{(1-\gamma)^2}\log \frac{1}{\delta}} \notag\\
    &\lesssim \sum_{t=1}^{T}\big(V^{\star}(s_{t+1}) - V^{\pi_{t+1}}(s_{t+1})\big) + \left(SAT\right)^{1/4}\left(\log\frac{SAT}{\delta}\right)^{5/4}\left(\frac{1}{1-\gamma}\right)^{7/4} + \sqrt{\frac{T}{(1-\gamma)^3}\log \frac{T}{\delta}}, \label{eq:xi-2-3}
    \end{align} 
    with probability at least $1-\delta/2$. The last step is obtained by applying \eqref{eq:xi-1} and the fact $\norm{V_t - V^{\pi_{t+1}}}_\infty \le \frac{1}{1-\gamma}$.
    
    Finally, substituting \eqref{eq:xi-2-3} into \eqref{eq:xi-2-1} gives the desired bound
    \begin{align}
    &\quad \ \sum_{t = 1}^{T}\zeta_{t,2} \notag\\
    &\lesssim \sqrt{\frac{SA}{1-\gamma} \left\{ \sum_{t=1}^{T}P_{s_{t},a_{t}}\left(V^{\star} - V^{\pi_{t}}\right) \right\} \log\frac{T}{\delta}}+\frac{SA}{1-\gamma}\log\frac{T}{\delta} \notag\\
    &\lesssim \sqrt{\frac{SA}{1-\gamma} \left\{ \sum_{t=1}^{T}\big(V^{\star}(s_{t+1}) - V^{\pi_{t+1}}(s_{t+1})\big) + \frac{\left(SAT\right)^{1/4}}{(1-\gamma)^{7/4}}\left(\log\frac{SAT}{\delta}\right)^{5/4} + \sqrt{\frac{T}{(1-\gamma)^3}\log \frac{T}{\delta}}\right\} \log\frac{T}{\delta}}+\frac{SA}{1-\gamma}\log\frac{T}{\delta} \notag\\
    &\le \sqrt{\frac{SA\log\frac{T}{\delta}}{1-\gamma} \sum_{t=1}^{T}\big(V^{\star}(s_{t+1}) - V^{\pi_{t+1}}(s_{t+1})\big)} + \frac{\left(SA\right)^{5/8}T^{1/8}}{(1-\gamma)^{11/8}}\left(\log\frac{SAT}{\delta}\right)^{9/8} \notag\\ & \qquad \qquad + \frac{(SA)^{1/2}T^{1/4}}{(1-\gamma)^{5/4}}\left(\log\frac{T}{\delta}\right)^{3/4} +\frac{SA}{1-\gamma}\log\frac{T}{\delta} \label{eq:xi-2-final}
    \end{align}
    with probability at least $1-\delta$.

\subsection{Bounding $\zeta_{t,3}$}\label{sec:zeta-3-proof}

By the definition of $K_t(s,a)$ in \eqref{eq:K-def}, we can write
    \begin{align}
        \sum_{t=1}^T \zeta_{t,3} &= \gamma\sum_{t=1}^T \big(V^{\pi_{t+1}}(s_{t+1}) - V^{\pi_t}(s_{t+1})\big)\notag\\
        &\leq \gamma\sum_{s,a}\sum_{i=0}^{K_T(s,a)} \frac{1}{1-\gamma}\notag\\
        &\lesssim \frac{1}{1-\gamma}\sum_{s,a}K_T(s,a), \label{eq:xi-3-1}
    \end{align}
    in which the inequality is obtained using the total number of times that policy $\pi_t$ in Algorithm \ref{alg:main-t} switches over $T$ steps and the fact that $\norm{V^{\pi_{t+1}} - V^{\pi_t}}_\infty \le \frac{1}{1-\gamma}$.

    By the upper bound on $\sum_{s,a}K_T(s,a)$ in \eqref{eq:K-bound}, we can obtain the desired final bound
    \begin{align}
        \sum_{t=1}^T \zeta_{t,3} \lesssim \frac{SA}{1-\gamma}\log T + \frac{(SA)^{3/4}T^{1/4}\log^{5/4}\frac{SAT}{\delta}}{(1-\gamma)^{5/4}} + \sqrt{\frac{SA\log^2 T}{1-\gamma}\sum_{t=1}^T \big(V_t(s_{t+1}) - V^{\pi_{t+1}}(s_{t+1})\big)}. \label{eq:xi-3-bound}
    \end{align}

\subsection{Crude Regret Bound as a Consequence of Lemma \ref{lem:xi}}

Note that Lemma \ref{lem:xi} can give rise to a crude bound on $\Regret(T)$, which we will need in the proof of Theorem \ref{thm:main}. It can be summarized in a lemma as follows:

\begin{lemma}\label{lem:crude}
    Choose any $\delta \in (0,1)$. Suppose that $c_b$ is chosen to be a sufficiently large universal constant in Algorithm \ref{alg:main}. Then there exists some absolute constant $C_0 > 0$ such that Algorithm \ref{alg:main} achieves
    \begin{equation}
        \Regret(T) \le \sum_{t=1}^T \big(V_{t-1}(s_t) - V^{\pi_t}(s_t)\big) \lesssim \sqrt{\frac{SAT\log\frac{SAT}{\delta}}{(1-\gamma)^5}} + \frac{SA\log^{2}\frac{SAT}{\delta}}{(1-\gamma)^{7/2}}
    \end{equation}
    with probability at least $1-\delta$.
\end{lemma}

\begin{proof}[Proof of Lemma \ref{lem:crude}]
    By the optimistic property of our value function estimates in Lemma \ref{lem:Q_t-lower-bound}, we can 
\begin{equation*}
    \mathrm{Regret}(T) = \sum_{t=1}^T \big(V^\star(s_t) - V^{\pi_t}(s_t)\big) \le \sum_{t=1}^T \big(V_{t-1}(s_t) - V^{\pi_t}(s_t)\big).
\end{equation*}
In the inequality, since $V_0(s)$ is undefined in our algorithm, we can simply let $V_0(s) = \frac{1}{1-\gamma}$ for all $s\in \cS$. 

We can follow the regret decomposition in \eqref{eq:regret-decomp1} and obtain the following inequality:
\begin{align}
    V_{t-1}(s_t) - V^{\pi_t}(s_t) \le \frac{\gamma(3-\gamma)}{2}\left(V^\star(s_{t+1}) - V^{\pi_{t+1}}(s_{t+1})\right) + Q_{t}(s_t, a_t) - Q^\star(s_t, a_t) + \zeta_t, \label{eq:crude-regret-decomp1}
\end{align}
in which $\zeta_t$ has the same definition as \eqref{eq:zeta-def}:
\begin{equation*}
    \zeta_t := V_{t-1}(s_t) - Q_t(s_t,a_t) + \gamma\left(P_{s_t,a_t} - P_t\right)\left(V^\star - V^{\pi_t}\right) + \gamma\left(V^{\pi_{t+1}}(s_{t+1}) - V^{\pi_t}(s_{t+1})\right).
\end{equation*}

Invoking Lemma \ref{lem:Q_t-upper-bound-crude} and rearranging \eqref{eq:regret-decomp1}, we obtain a recursion as follows:
    \begin{align}
        &\quad \left(1 - \frac{\gamma(3-\gamma)}{2}\right)\sum_{t=1}^T \big(V_{t-1}(s_t) - V^{\pi_t}(s_t)\big) \notag\\
        &= \sum_{t=1}^T\big(V_{t-1}(s_t) - V^{\pi_t}(s_t)\big) - \frac{\gamma(3-\gamma)}{2}\sum_{t=1}^T\big(V_{t-1}(s_t) - V^{\pi_t}(s_t)\big) \notag\\
        &\overset{(\mathrm{i})}{\le} \sum_{t=1}^T\big(V_{t-1}(s_t) - V^{\pi_t}(s_t)\big) - \frac{\gamma(3-\gamma)}{2}\sum_{t=2}^T\big(V_{t-1}(s_t) - V^{\pi_t}(s_t)\big) \notag\\
        &= \sum_{t=1}^T\big(V_{t-1}(s_t) - V^{\pi_t}(s_t)\big) - \frac{\gamma(3-\gamma)}{2}\sum_{t=1}^{T-1}\big(V_{t}(s_{t+1}) - V^{\pi_{t+1}}(s_{t+1})\big) \notag\\
        &\overset{(\mathrm{ii})}{\le} \sum_{t=1}^T\frac{\gamma(3-\gamma)}{2}\big(V^\star(s_{t+1}) - V^{\pi_{t+1}}(s_{t+1})\big) + Q_{t}(s_t, a_t) - Q^\star(s_t, a_t) + \zeta_t - \frac{\gamma(3-\gamma)}{2}\sum_{t=1}^{T-1}\big(V_{t}(s_{t+1}) - V^{\pi_{t+1}}(s_{t+1})\big) \notag\\
        &\overset{(\mathrm{iii})}{\le} 8c_b\sqrt{\frac{SAT}{(1-\gamma)^3}\log\frac{SAT}{\delta}} + \frac{SA}{1-\gamma} + \sum_{t=1}^T\frac{\gamma(3-\gamma)}{2}\big(V_t(s_{t+1}) - V^{\pi_{t+1}}(s_{t+1})\big) + \zeta_t \notag\\ & \qquad - \frac{\gamma(3-\gamma)}{2}\sum_{t=1}^{T-1}\big(V_{t}(s_{t+1}) - V^{\pi_{t+1}}(s_{t+1})\big) \notag\\
        &= 8c_b\sqrt{\frac{SAT}{(1-\gamma)^3}\log\frac{SAT}{\delta}} + \frac{SA}{1-\gamma} + \frac{\gamma(3-\gamma)}{2}\big(V^\star(s_{T+1}) - V^{\pi_{T+1}}(s_{T+1})\big) + \sum_{t=1}^T \zeta_t\notag\\
        &\overset{(\mathrm{iv})}{\le} 8c_b\sqrt{\frac{SAT}{(1-\gamma)^3}\log\frac{SAT}{\delta}} + \frac{SA}{1-\gamma} + \frac{3}{2(1-\gamma)} + \sum_{t=1}^T \zeta_t\notag\\
        &\lesssim \sqrt{\frac{SAT}{(1-\gamma)^3}\log\frac{SAT}{\delta}} + \frac{SA}{1-\gamma} + \sum_{t=1}^T \zeta_t. \label{eq:crude-regret-decomp2}
    \end{align}
    Above, (i) is because by the optimistic property in Lemma \ref{lem:Q_t-upper-bound}, $V_{t-1}(s_t) \ge V^\star(s_t) \ge V^{\pi_t}(s_t) \ge 0$ and thus $V_{t-1}(s_t) - V^{\pi_t}(s_t) \ge 0$. (ii) is obtained by replacing $\left(V_{t-1}(s_t) - V^{\pi_t}(s_t)\right)$ with \eqref{eq:crude-regret-decomp1}. (iii) is obtained by invoking Lemma \ref{lem:Q_t-upper-bound-crude} to replace $\sum_{t=1}^T Q_t(s_{t},a_{t}) - Q^\star(s_{t},a_{t})$. (iv) is due to $V^\star(s_{T+1}) - V^{\pi_{T+1}}(s_{T+1}) \le \frac{1}{1-\gamma}$ and $\gamma\in(0,1)$. 

Note $0 < 1 - \frac{\gamma(3-\gamma)}{2} < 1$ since $\gamma\in(0,1)$. Dividing \eqref{eq:regret-decomp2} by $1 - \frac{\gamma(3-\gamma)}{2}$, we have
\begin{align}
    \sum_{t=1}^T \big(V_{t-1}(s_t) - V^{\pi_t}(s_t)\big) \lesssim \sqrt{\frac{SAT}{(1-\gamma)^5}\log\frac{SAT}{\delta}} + \frac{SA}{(1-\gamma)^2} + \frac{1}{1-\gamma}\sum_{t=1}^T \zeta_t, \label{eq:crude-regret-decomp3}
\end{align}
in which the additional $\frac{1}{1-\gamma}$ factor is due to the fact $1 - \frac{\gamma(3-\gamma)}{2} \ge \frac{1-\gamma}{2}$.

After applying Lemma \ref{lem:xi}, \eqref{eq:crude-regret-decomp3} can be written as
\begin{align}
    & \quad\ \sum_{t=1}^T \big(V_{t-1}(s_t) - V^{\pi_t}(s_t)\big) \notag\\
    &\leq C_5\Bigg(\sqrt{\frac{SAT\log\frac{SAT}{\delta}}{(1-\gamma)^5}} + \frac{SA\log^{2} \frac{SAT}{\delta}}{(1-\gamma)^{7/2}} + \sqrt{\frac{\log^2 T}{(1-\gamma)^3}\sum_{t=1}^T\big(V_{t-1}(s_{t}) - V^{\pi_{t}}(s_{t})\big)} \notag\\ & \qquad + \sqrt{\frac{SA\log\frac{T}{\delta}}{(1-\gamma)^3} \sum_{t=1}^{T}\big(V^{\star}(s_{t}) - V^{\pi_{t}}(s_{t})\big)}\Bigg) \notag\\
    &\leq C_5\Bigg(\sqrt{\frac{SAT\log\frac{SAT}{\delta}}{(1-\gamma)^5}} + \frac{SA\log^{2} \frac{SAT}{\delta}}{(1-\gamma)^{7/2}} + \sqrt{\frac{\log^2 T}{(1-\gamma)^3}\sum_{t=1}^T\big(V_{t-1}(s_{t}) - V^{\pi_{t}}(s_{t})\big)} \notag\\ & \qquad + \sqrt{\frac{SA\log\frac{T}{\delta}}{(1-\gamma)^3} \sum_{t=1}^{T}\big(V_{t-1}(s_{t}) - V^{\pi_{t}}(s_{t})\big)}\Bigg) \label{eq:crude-regret-decomp4}
\end{align}
for some absolute constant $C_5 > 0$. The last inequality is by the optimism of $V_{t-1}$ by Lemma \ref{lem:Q_t-lower-bound}.

Finally, we can invoke Proposition \ref{prop:algebraic} on \eqref{eq:crude-regret-decomp4} by letting 
\begin{align*}
    x &:= C_5\left(\frac{\log T}{(1-\gamma)^{3/2}} + \sqrt{\frac{SA\log\frac{T}{\delta}}{(1-\gamma)^3}}\right),\\
    y &:= C_5\left(\sqrt{\frac{SAT\log\frac{SAT}{\delta}}{(1-\gamma)^5}} + \frac{SA\log^{2} \frac{SAT}{\delta}}{(1-\gamma)^{7/2}}\right),\\
    z &:= \sum_{t=1}^T \big(V_{t-1}(s_t) - V^{\pi_t}(s_t)\big)
\end{align*}
in \eqref{eq:algebraic-lem-cond}. \eqref{eq:algebraic-lem-stat} would give
\begin{align}
    & \quad\ \sum_{t=1}^T \big(V_{t-1}(s_t) - V^{\pi_t}(s_t)\big) \notag\\
    &\leq 2C_5\sqrt{\frac{SAT\log\frac{SAT}{\delta}}{(1-\gamma)^5}} + 2(C_5 + C_5^2)\frac{SA\log^{2} \frac{SAT}{\delta}}{(1-\gamma)^{7/2}} \notag\\
    &\lesssim \sqrt{\frac{SAT\log\frac{SAT}{\delta}}{(1-\gamma)^5}} + \frac{SA\log^{2} \frac{SAT}{\delta}}{(1-\gamma)^{7/2}}, \label{eq:crude-regret-decomp-final}
\end{align}
which is as claimed in Lemma \ref{lem:crude}.
\end{proof}


\section{Proof of Lemma \ref{lem:Q_t-lcb}}\label{sec:proof:Q_t-lcb}

\subsection{Proof of Inequality \eqref{eq:Q-lcb-monotonicity}}

Before proving the pessimism in $V_{t}^{\LCB}$, we focus on proving
\begin{equation}
	\label{eq:lcb}
	Q^{\LCB}_t(s, a) \le Q^{\star}(s, a) \qquad\text{for all } (s, a, t) \in  \cS \times \cA \times [T] . 
\end{equation}
Let us proceed by induction. We start with the base case. By our initialization, we have $$0 = Q^{\LCB}_1(s, a) \le Q^{\star}(s,a);$$
hence, \eqref{eq:lcb} trivially holds.

For the induction step, let us assume the inductive hypothesis that \eqref{eq:lcb} holds all the way up to $t$ for all $(s,a)\in\cS\times\cA$. We now show the statement holds for the $(t+1)$-th step as well. Recall that only $Q_{t+1}^{\LCB}(s_t,a_t)$ is updated at the $t$-th step with all other entries of $Q_{t}^{\LCB}$ fixed, so it suffices to check 
    \begin{equation}
        Q^{\LCB}_{t+1}(s_t,a_t) \le Q^{\star}(s_t,a_t). \label{eq:pessimism-induction-step}
    \end{equation}
Again, we adopt the shorthand notation $N_t := N_t(s_t,a_t)$ and $t_n := t_n(s_t,a_t)$ when it is clear from the context.

Recall $b_n := c_b\sqrt{\frac{\log\frac{SAT}{\delta}}{(1-\gamma)^3 n}}$. We rewrite the subroutine \textbf{update-q-lcb}() in Line \ref{alg-line:QLCB-update} of Algorithm \ref{alg:aux}:
    \begin{equation*}
        Q^{\LCB}_{t+1}(s_t,a_t) = Q^{\LCB}_{t_N+1}(s_t,a_t) = (1 - \eta_N)Q^{\LCB}_{t_N}(s_t,a_t) + \eta_{N}\left(r(s_t,a_t) + \gamma V_{t_N}^\LCB(s_{t_N+1}) - b_N\right).
    \end{equation*}
    This recursive relation can be expanded as
    \begin{equation}
        Q^{\LCB}_{t+1}(s_t,a_t) = \eta_0^N Q^{\LCB}_{1}(s_t,a_t) + \sum_{n=1}^N \eta_n^N \big(r(s_t,a_t) + \gamma V_{t_n}^\LCB(s_{t_n+1}) - b_n\big), \label{eq:QLCB}
    \end{equation}
    in which $\eta_0^N$ and $\eta_n^N$ are defined in Lemma \ref{lem:learning-rate}.
    
    Using \eqref{eq:QLCB}, we have 
    \begin{align}\label{eq:Q_LCB-Q*}
        &\quad\ Q^{\LCB}_{t+1}(s_t,a_t) - Q^\star(s_t,a_t) \nonumber\\
        &= \eta_0^N \left(Q^{\LCB}_{1}(s_t,a_t) - Q^\star(s_t,a_t)\right) + \sum_{n=1}^N \eta_n^N \left(r(s_t,a_t) + \gamma V_{t_n}^\LCB(s_{t_n+1}) - b_n - Q^\star(s_t,a_t)\right) \nonumber\\
        &= \eta_0^N \left(Q^{\LCB}_{1}(s_t,a_t) - Q^\star(s_t,a_t)\right) + \sum_{n=1}^N \eta_n^N \Big(\gamma\left(V_{t_n}^\LCB(s_{t_n+1}) - V^\star(s_{t_n+1})\right) + \gamma\left(P_{t_n} - P_{s_t,a_t}\right)V^\star - b_n\Big),
    \end{align}
    where the last step is by $Q^\star(s,a) = r(s,a) + \gamma P_{s,a} V^\star$ and $V^\star(s_{t_n+1}) = P_{t_n}V^\star$.
    
    Recall in the proof for optimism (lemma \ref{lem:Q_t-lower-bound}), by \eqref{eq:Jin-Hoeffding}, we have with probability at least $1-\delta$,
    \begin{equation*}
        \left|\sum_{n=1}^N\eta_n^N\gamma\left(P_{t_n} - P_{s_t,a_t}\right)V^\star\right| \le c_b\sqrt{\frac{\log\frac{SAT}{\delta}}{(1-\gamma)^3N}}\sum_{n=1}^N\eta_n^N \le \sum_{n=1}^N\eta_n^N b_n,
    \end{equation*}
    for all $(N,s,a) \in [T] \times \cS \times \cA$. We can employ this inequality on \eqref{eq:Q_LCB-Q*} this time and obtain that with probability at least $1-\delta$,
    \begin{align*}
        &\quad\ Q^{\LCB}_{t+1}(s_t,a_t) - Q^\star(s_t,a_t)\\
        &= \eta_0^N \left(Q^{\LCB}_{1}(s_t,a_t) - Q^\star(s_t,a_t)\right) + \sum_{n=1}^N \eta_n^N \Big(\gamma\left(V_{t_n}^\LCB(s_{t_n+1}) - V^\star(s_{t_n+1})\right) + \gamma\left(P_{t_n} - P_{s_t,a_t}\right)V^\star - b_n\Big)\\
        &\le \eta_0^N \left(Q^{\LCB}_{1}(s_t,a_t) - Q^\star(s_t,a_t)\right) + \gamma\sum_{n=1}^N \eta_n^N \left(V_{t_n}^\LCB(s_{t_n+1}) - V^\star(s_{t_n+1})\right)\\
        &\le \gamma\sum_{n=1}^N \eta_n^N \left(V_{t_n}^\LCB(s_{t_n+1}) - V^\star(s_{t_n+1})\right),
    \end{align*}
    where the last step is due to our initialization $Q^{\UCB}_{1}(s, a) = 0 \le Q^\star(s,a)$.
    
    Lastly, the inductive hypothesis gives $V_t^\LCB(s) \le V^\star(s)$ for any $(s,a)\in\cS\times\cA$, so we can conclude
    \begin{equation*}
        Q^{\LCB}_{t+1}(s, a) - Q^\star(s, a) \le \sum_{n=1}^N \eta_n^N \gamma\left(V_{t_n}^\LCB(s_{t_n+1}) - V^\star(s_{t_n+1})\right) \le 0.
    \end{equation*}
    At this point, we have finished the proof for the induction step and thus the proof for \eqref{eq:lcb}.

Now, let us prove the pessimism in $V^{\LCB}$. Given the pessimism in $Q^{\LCB}$ in \eqref{eq:lcb}, we have
\begin{align}
	\max_{a}Q_{t}^{\LCB}(s,a)\leq\max_{a}Q^{\star}(s,a)=V^{\star}(s)\qquad\text{for all }(t,s)\in[T]\times\mathcal{S}.
	\label{eq:lcb-1}
\end{align}
On the other hand, the construction of $V_{t}^{\LCB}$ in Line \ref{alg-t-line:VLCB} of Algorithm \ref{alg:main-t} can be written as
\begin{equation*}
    V_{t}^{\LCB}(s) \leq  \max\left\{ \max_{j:j\leq t}  \max_a Q^{\LCB}_j(s, a) ,\ \max_{j:j < t}V_j^{\LCB}(s) \right\}.
\end{equation*}
Combined with the initialization $V_{1}^{\LCB}(s)=0$ for all $s\in\cS$, we can conclude with the desired pessimism of $V^{\LCB}$ via a simple induction:
\begin{align}
	V_{t}^{\LCB}(s)\leq V^{\star}(s) \ \text{for all }(t,s)\in[T]\times\mathcal{S}.
	\label{eq:lcb-V}
\end{align}

\subsection{Proof of Inequality \eqref{eq:main-lemma}}

The proof of \eqref{eq:main-lemma} essentially follows the result of Lemma \ref{lem:weighted-V}, which provides an upper bound on any weighted sum of our algorithm's value estimates.
\begin{lemma}\label{lem:weighted-V}
Assume there exists a constant $c_b>0$ such that for all $(t,s,a) \in [T]\times\cS\times \cA$, it holds that 
\begin{align} 
	0 &\leq Q_{t+1}(s, a) - Q^{\LCB}_{t+1}(s, a)  \notag\\
	& \le \frac{\eta^{N_t (s, a)}_0}{1-\gamma}
	 + \gamma\sum_{n = 1}^{N_t(s, a)} \eta^{N_t (s, a)}_n \left(V_{t_n}(s_{t_n+1}) - V^{\LCB}_{t_n}(s_{t_n+1}) \right) + 4 c_b \sqrt{\frac{\log \frac{SAT}{\delta}}{(1-\gamma)^3 N_t(s,a)}},\label{eq:weighted-V-cond}
\end{align}
in which we use the abbreviation $t_n := t_n(s,a)$.

For every sequence $\{w_t\}_{t= 1}^{\infty}$ such that $0\le w_t\le w$ for all $t$ and $\sum_{t=1}^{\infty} w_t \le W$, one has 
\begin{align}\label{eq:weighted-V}
	\sum_{t=1}^T &w_t\left(V_{t}(s_{t+1})-V_{t}^{\LCB}(s_{t+1})\right) \le C_{\mathrm{v},1}\frac{ew\left(SA\right)^{3/4}T^{1/4}}{(1-\gamma)^{9/4}}\left(\log\frac{SAT}{\delta}\right)^{5/4} + C_{\mathrm{v},1}\frac{eSAw}{(1-\gamma)^2}\log T \notag\\ & \quad + 8ec_b \sqrt{\frac{SAwW\log \frac{SAT}{\delta}}{(1-\gamma)^5}} + C_{\mathrm{v},1}ew\sqrt{\frac{SA\log^2 T}{(1-\gamma)^3}\sum_{t=1}^T \big(V_t(s_{t+1}) - V^{\pi_{t+1}}(s_{t+1})\big)} + eW
\end{align}
for some absolute constant $C_{\mathrm{v},1} > 0$.
\end{lemma}

    The proof of Lemma \ref{lem:weighted-V} is deferred to Appendix \ref{sec:proof:weighted-V}. We can easily prove \eqref{eq:main-lemma} with Lemma \ref{lem:weighted-V}. Let us specify
    \begin{equation}
        w_t = \mathds{1}\left\{V_{t}(s_{t+1}) - V_{t}^{\LCB}(s_{t+1})>3\right\}. \label{eq:wt-def}
    \end{equation}
    Clearly, this gives us
    \begin{align}
        3W = 3\sum_{t=1}^T w_t &\le \sum_{t=1}^{T}w_t\left(V_{t}(s_{t+1})-V_{t}^{\LCB}(s_{t+1})\right). \label{eq:weighted-V-4}
    \end{align}
    If we assume \eqref{eq:weighted-V-cond} is true, we can bring \eqref{eq:weighted-V} into the right-hand side of \eqref{eq:weighted-V-4}. We can rearrange and rescale the inequality to obtain
    \begin{align}
        \frac{1}{3-e}(3W-eW) &\le \frac{1}{3-e}\Bigg(C_{\mathrm{v},1}\frac{ew\left(SA\right)^{3/4}T^{1/4}}{(1-\gamma)^{9/4}}\left(\log\frac{SAT}{\delta}\right)^{5/4} + C_{\mathrm{v},1}\frac{eSAw}{(1-\gamma)^2}\log T + 8ec_b \sqrt{\frac{SAwW\log \frac{SAT}{\delta}}{(1-\gamma)^5}} \notag\\ &\qquad \qquad \qquad + C_{\mathrm{v},1}ew\sqrt{\frac{SA\log^2 T}{(1-\gamma)^3}\sum_{t=1}^T \big(V_t(s_{t+1}) - V^{\pi_{t+1}}(s_{t+1})\big)}\Bigg). \label{eq:weighted-V-5}
    \end{align}
    Invoking Proposition \ref{prop:algebraic} on \eqref{eq:weighted-V-5} with
    \begin{align*}
        x &:= \frac{8e}{3-e}c_b\sqrt{\frac{SAw\log \frac{SAT}{\delta}}{(1-\gamma)^5}},\\
        y &:= \frac{ew}{3-e}C_{\mathrm{v},1}\Bigg(\frac{\left(SA\right)^{3/4}T^{1/4}}{(1-\gamma)^{9/4}}\left(\log\frac{SAT}{\delta}\right)^{5/4} + \frac{SA}{(1-\gamma)^2}\log T + \sqrt{\frac{SA\log^2 T}{(1-\gamma)^3}\sum_{t=1}^T \big(V_t(s_{t+1}) - V^{\pi_{t+1}}(s_{t+1})\big)}\Bigg),\\
        z &:= \sqrt{W},
    \end{align*}
    we can conclude with the inequality:
    \begin{align}
        &\quad\ \sum_{t=1}^T \mathds{1}\left\{V_{t}(s_{t+1}) - V_{t}^{\LCB}(s_{t+1})>3\right\} \le W \notag\\
        &\lesssim \frac{w\left(SA\right)^{3/4}T^{1/4}}{(1-\gamma)^{9/4}}\left(\log\frac{SAT}{\delta}\right)^{5/4} + \frac{SAw}{(1-\gamma)^5}\log \frac{SAT}{\delta} + w\sqrt{\frac{SA\log^2 T}{(1-\gamma)^3}\sum_{t=1}^T \big(V_t(s_{t+1}) - V^{\pi_{t+1}}(s_{t+1})\big)} \notag\\
        &\overset{(\mathrm{i})}{=} \frac{\left(SA\right)^{3/4}T^{1/4}}{(1-\gamma)^{9/4}}\left(\log\frac{SAT}{\delta}\right)^{5/4} + \frac{SA}{(1-\gamma)^5}\log \frac{SAT}{\delta} + \sqrt{\frac{SA\log^2 T}{(1-\gamma)^3}\sum_{t=1}^T \big(V_t(s_{t+1}) - V^{\pi_{t+1}}(s_{t+1})\big)} \notag\\
        &\overset{(\mathrm{ii})}{\lesssim} \frac{\left(SA\right)^{3/4}T^{1/4}}{(1-\gamma)^{9/4}}\left(\log\frac{SAT}{\delta}\right)^{5/4} + \frac{SA}{(1-\gamma)^5}\log \frac{SAT}{\delta} + \sqrt{\frac{SA\log^2 T}{(1-\gamma)^3}\sum_{t=2}^T \big(V_{t-1}(s_{t}) - V^{\pi_{t}}(s_{t})\big)} \notag\\
        &\overset{(\mathrm{iii})}{\le} \frac{\left(SA\right)^{3/4}T^{1/4}}{(1-\gamma)^{9/4}}\left(\log\frac{SAT}{\delta}\right)^{5/4} + \frac{SA}{(1-\gamma)^5}\log \frac{SAT}{\delta} + \sqrt{\frac{SA\log^2 T}{(1-\gamma)^3}\sum_{t=1}^T \big(V_{t-1}(s_{t}) - V^{\pi_{t}}(s_{t})\big)}, \label{eq:weighted-V-final}
    \end{align}
    in which (i) is because $w=1$ by our definition of $w_t$ in \eqref{eq:wt-def}; (ii) is due to $\norm{V_{T} - V^{\pi_{T+1}}}_\infty \le \frac{1}{1-\gamma}$; (iii) is due to $V_0(s_1) - V^{\pi_{1}}(s_{1}) \ge 0$, which can be obtained from the fact $V_0(s) = \frac{1}{1-\gamma}$ and $V^\star(s) \ge V^{\pi_{1}}(s)$ for all $s\in\cS$. 

    At this point, we can easily prove \eqref{eq:main-lemma}, as \eqref{eq:weighted-V-final} immediately leads to 
        \begin{align}
            &\quad \sum_{t=1}^{T}\left(V_{t}(s_{t+1})-V_{t}^{\LCB}(s_{t+1})\right)\mathds{1}\left\{V_{t}(s_{t+1}) - V_{t}^{\LCB}(s_{t+1})>3\right\} \notag\\
            &\leq \frac{1}{1-\gamma}\sum_{t=1}^{T}\mathds{1}\left\{V_{t}(s_{t+1}) - V_{t}^{\LCB}(s_{t+1})>3\right\} \notag\\
            &\lesssim \frac{\left(SA\right)^{3/4}T^{1/4}}{(1-\gamma)^{13/4}}\left(\log\frac{SAT}{\delta}\right)^{5/4} + \frac{SA}{(1-\gamma)^6}\log \frac{SAT}{\delta} + \sqrt{\frac{SA\log^2 T}{(1-\gamma)^5}\sum_{t=1}^T \big(V_{t-1}(s_{t}) - V^{\pi_{t}}(s_{t})\big)},
        \end{align}
    which is the final result in \eqref{eq:main-lemma} that we desire. 

    As the last step of proving \eqref{eq:main-lemma}, we need to return to the justification of \eqref{eq:weighted-V-cond} that allows for our invocation of Lemma \ref{lem:weighted-V} earlier. The first inequality in \eqref{eq:weighted-V-cond} is easy to prove, for Lemma \ref{lem:Q_t-lower-bound} and Lemma \ref{lem:Q_t-lcb} tells us $Q_t(s,a) \geq Q^{\star}(s, a) \geq Q_t^{\LCB}(s,a)$ for all $(t, s, a) \in  [T] \times \cS \times \cA$. Thus, it only remains to prove the second inequality in \eqref{eq:weighted-V-cond}, which is an upper bound on $Q_{t+1}(s, a) - Q^{\LCB}_{t+1}(s, a)$. 
    
    Let us abbreviate $N_t := N_t(s,a)$ temporarily. Since the update rule of $Q_{t+1}$ in Line \ref{alg-t-line:Q} of Algorithm \ref{alg:main-t} implies $Q_{t+1}(s, a) \leq Q^{\UCB}_{t+1}(s,a)$, we can instead focus on finding an upper bound for the following quantity:
    \begin{align}
        Q_{t+1}(s, a) - Q^{\LCB}_{t+1}(s, a) \le Q^{\UCB}_{t+1}(s, a) - Q^{\LCB}_{t+1}(s, a) = Q^{\UCB}_{t_{N_t}+1}(s, a) - Q^{\LCB}_{t_{N_t}+1}(s, a). \label{eq:Q-QLCB-QUCB-QLCB}
    \end{align}
    Towards this, the update rules of $Q^{\UCB}_t$ and $Q^{\LCB}_t$ in Line \ref{alg-line:QUCB-update} and \ref{alg-line:QLCB-update} of Algorithm \ref{alg:aux} allow us to write
    \begin{align}
     &\quad\ Q^{\UCB}_{t_{N_t}+1}(s, a) - Q^{\LCB}_{t_{N_t}+1}(s, a) \notag\\
     & =(1-\eta_{N_{t}})Q_{t_{N_t}}^{\UCB}(s,a)+\eta_{N_{t}}\left(r(s,a)+\gamma V_{t_{N_t}}(s_{t_{N_t}+1})+c_b\sqrt{\frac{\log\frac{SAT}{\delta}}{(1-\gamma)^3 N_{t}}}\right) \notag\\
     & \qquad-(1-\eta_{N_{t}})Q_{t_{N_t}}^{\LCB}(s,a)-\eta_{N_{t}}\left(r(s,a)+ \gamma V^{\LCB}_{t_{N_t}}(s_{t_{N_t}+1})-c_b\sqrt{\frac{\log\frac{SAT}{\delta}}{(1-\gamma)^3 N_{t}}}\right) \notag\\
     &= (1-\eta_{N_{t}})\left(Q_{t_{N_t}}^{\UCB}(s,a)-Q_{t_{N_t}}^{\LCB}(s,a)\right)
    	+\eta_{N_{t}}\left(\gamma V_{t_{N_t}}(s_{t_{N_t}+1}) - \gamma V^{\LCB}_{t_{N_t}}(s_{t_{N_t}+1}) + 2c_b\sqrt{\frac{\log\frac{SAT}{\delta}}{(1-\gamma)^3 N_{t}}}\right) \notag\\
     &= (1-\eta_{N_{t}})\left(Q_{t_{N_t-1}+1}^{\UCB}(s,a) - Q_{t_{N_t-1}+1}^{\LCB}(s,a)\right)
    	+ \eta_{N_{t}}\left(\gamma V_{t_{N_t}}(s_{t_{N_t}+1}) - \gamma V^{\LCB}_{t_{N_t}}(s_{t_{N_t}+1}) + 2c_b\sqrt{\frac{\log\frac{SAT}{\delta}}{(1-\gamma)^3 N_{t}}}\right) \label{eq:weighted-V-cond-1} \\
    & = \eta^{N_t}_0 \left(Q_{1}^{\UCB}(s,a)-Q_{1}^{\LCB}(s,a)\right) 
    	 + \sum_{n = 1}^{N_t} \eta^{N_t}_n \left(\gamma V_{t_{n}}(s_{t_{n}+1}) - \gamma V^{\LCB}_{t_{n}}(s_{t_{n}+1}) + 2c_b\sqrt{\frac{\log\frac{SAT}{\delta}}{(1-\gamma)^3 n}}\right) \notag\\
    	 & \leq \frac{\eta^{N_t}_0}{1-\gamma}
    	 + \gamma\sum_{n = 1}^{N_t} \eta^{N_t}_n \left(V_{t_{n}}(s_{t_{n}+1}) - V^{\LCB}_{t_{n}}(s_{t_{n}+1})\right) + 4c_b\sqrt{\frac{\log\frac{SAT}{\delta}}{(1-\gamma)^3 N_t}}, \label{eq:weighted-V-cond-2}
    \end{align}
    in which the penultimate line is obtained by expanding $Q_{t_{N_t-1}+1}^{\UCB}(s,a) - Q_{t_{N_t-1}+1}^{\LCB}(s,a)$ recursively using the recursive relation from \eqref{eq:weighted-V-cond-1}, and the last line is due to our initialization $0 = Q_{1}^{\LCB}(s,a) \leq Q_{1}^{\UCB} (s,a) = \frac{1}{1-\gamma}$ and the following fact
    \begin{equation*}
    	\sum_{n = 1}^{N_t} \eta^{N_t}_n c_b\sqrt{\frac{\log\frac{SAT}{\delta}}{(1-\gamma)^3 n}}  \leq 2c_b\sqrt{\frac{\log\frac{SAT}{\delta}}{(1-\gamma)^3 N_{t}}},
    \end{equation*}
    which is an immediate consequence of \eqref{eq:learning-rate2} in Lemma \ref{lem:learning-rate}.
    
    Finally, \eqref{eq:weighted-V-cond-2} combined with \eqref{eq:Q-QLCB-QUCB-QLCB} establishes the second inequality in the condition \eqref{eq:weighted-V-cond}. With \eqref{eq:weighted-V-cond} proven, we have justified our invocation of Lemma \ref{lem:weighted-V} and thus concluded the proof of the inequality \eqref{eq:main-lemma}.

\subsubsection{Proof of Lemma \ref{lem:weighted-V}}\label{sec:proof:weighted-V}

We can decompose the quantity of interest into two terms and bound them separately:
\begin{align}
    &\quad \sum_{t=1}^{T}w_t\left(V_{t}(s_{t+1})-V_{t}^{\LCB}(s_{t+1})\right) \notag\\
    &= \sum_{t=1}^{T}w_t\left(V_{t}(s_{t+1}) - Q_{t+1}(s_{t+1},a_{t+1}) + Q_{t+1}(s_{t+1},a_{t+1}) -V_{t}^{\LCB}(s_{t+1})\right) \notag\\
    &\leq \sum_{t=1}^{T}w_t\Big(\left(V_{t}(s_{t+1}) - Q_{t+1}(s_{t+1},a_{t+1})\right) + \left(Q_{t+1}(s_{t+1},a_{t+1}) - Q_{t+1}^{\LCB}(s_{t+1},a_{t+1})\right)\Big), \label{eq:weighted-V-1}
\end{align}
in which the last step is due to $V_{t}^{\LCB}(s) \le V_{t+1}^{\LCB}(s) \le \max_a Q_{t+1}^{\LCB}(s,a)$ for any $s\in\cS$, which can be seen from the update rule of $V^\LCB$.

In the following, let us bound the two halves of the qunatity in \eqref{eq:weighted-V-1} separately. Notice that the first half coincides with $\zeta_{t,1}$ in Lemma \ref{lem:xi} (defined in \eqref{eq:zeta-decomposition-def}). Thus, there exists some absolute constant $C_{\mathrm{v},1}>0$ such that we have
\begin{align}
    &\quad\ \sum_{t=1}^{T}w_t\big(V_{t}(s_{t+1}) - Q_{t+1}(s_{t+1},a_{t+1})\big) \notag\\
    &= V_{T}(s_{T+1}) - Q_{T+1}(s_{T+1},a_{T+1}) + \sum_{t=1}^{T-1}w_t\big(V_{t}(s_{t+1}) - Q_{t+1}(s_{t+1},a_{t+1})\big) \notag\\
    &\overset{(\mathrm{i})}{\le} \frac{1}{1-\gamma} + \sum_{t=2}^{T}w_{t-1}\big(V_{t-1}(s_{t}) - Q_{t}(s_{t},a_{t})\big) \notag\\
    &\overset{(\mathrm{ii})}{\le} \frac{1}{1-\gamma} + \sum_{t=1}^{T}w_{t-1}\big(V_{t-1}(s_{t}) - Q_{t}(s_{t},a_{t})\big) \notag\\
    &= \frac{1}{1-\gamma} + \sum_{t=1}^{T}w_t\zeta_{t,1} \notag\\
    &\le \frac{1}{1-\gamma} + w\sum_{t=1}^{T}\zeta_{t,1} \notag\\
    &\overset{(\mathrm{iii})}{\lesssim} \frac{wSA}{1-\gamma}\log T + \frac{w\left(SA\right)^{3/4}T^{1/4}}{(1-\gamma)^{5/4}}\left(\log\frac{SAT}{\delta}\right)^{5/4} + w\sqrt{\frac{SA\log^2 T}{1-\gamma}\sum_{t=1}^T \big(V_t(s_{t+1}) - V^{\pi_{t+1}}(s_{t+1})\big)} \notag\\
    &\leq C_{\mathrm{v},1}\left(\frac{wSA}{1-\gamma}\log T + \frac{w\left(SA\right)^{3/4}T^{1/4}}{(1-\gamma)^{5/4}}\left(\log\frac{SAT}{\delta}\right)^{5/4} + w\sqrt{\frac{SA\log^2 T}{1-\gamma}\sum_{t=1}^T \big(V_t(s_{t+1}) - V^{\pi_{t+1}}(s_{t+1})\big)}\right), \label{eq:weighted-V-1-1}
\end{align}
Above, (i) is due to $V_{T}(s_{T+1}) \le \frac{1}{1-\gamma}$ and $Q_{T+1}(s_{T+1},a_{T+1}) \ge Q^\star(s_{T+1},a_{T+1}) \ge 0$. (ii) is due to $V_{0}(s_{1}) - Q_{1}(s_{1},a_{1}) \ge 0$, as we set $V_{0}(s) = \frac{1}{1-\gamma}$ for any $s\in\cS$ in the proof of Lemma \ref{lem:xi}. (iii) is by \eqref{eq:xi-1}. $C_{\mathrm{v},1}>0$ is some absolute constant.

Adopting the abbreviation $N_t := N_t(s_{t+1},a_{t+1})$ and $t_n := t_n(s_{t+1},a_{t+1})$ hereafter, we can expand the second half of \eqref{eq:weighted-V-1} using \eqref{eq:weighted-V-cond} as follows:
\begin{align}
    &\quad\ \sum_{t=1}^{T}w_t\left(Q_{t+1}(s_{t+1},a_{t+1}) - Q_{t+1}^{\LCB}(s_{t+1},a_{t+1})\right) \notag\\
    &\le \sum_{t=1}^{T}\frac{w_t\eta^{N_t}_0}{1-\gamma}
	 + \gamma w_t\sum_{n = 1}^{N_t} \eta^{N_t}_n \left(V_{t_n}(s_{t_n+1}) - V^{\LCB}_{t_n}(s_{t_n+1}) \right) + 4 c_b w_t \sqrt{\frac{\log \frac{SAT}{\delta}}{(1-\gamma)^3 N_t}} \notag\\
    &\le \frac{SAw}{1-\gamma} + \gamma\sum_{t=1}^T \widetilde{w}_t\left(V_{t}(s_{t+1}) - V^{\LCB}_{t}(s_{t+1})\right) + 8c_b \sqrt{\frac{SAwW\log \frac{SAT}{\delta}}{(1-\gamma)^3}}. \label{eq:weighted-V-1-2}
\end{align}
The last step is due to the following two chain of inequalities \eqref{eq:weighted-V-1-2.1} and \eqref{eq:weighted-V-1-2.2}:
\begin{align}
    &\quad\ 4c_b\sum_{t=1}^T w_t\sqrt{\frac{\log \frac{SAT}{\delta}}{(1-\gamma)^3 N_{t}}}\notag\\
    &= 4c_b\sum_{(s,a)\in\cS\times\cA}\sum_{n=1}^{N_{T}(s,a)}w_{t_n}\sqrt{\frac{\log \frac{SAT}{\delta}}{(1-\gamma)^3 n}}\notag\\
    &\leq 4c_b\sum_{(s,a)\in\cS\times\cA}\sum_{n=1}^{\bar{N}_{T}(s,a)/w}w\sqrt{\frac{\log \frac{SAT}{\delta}}{(1-\gamma)^3 n}}\notag\\
    &\leq 4c_b\sum_{(s,a)\in\cS\times\cA}\int_{n=1}^{\bar{N}_{T}(s,a)/w}w\sqrt{\frac{\log \frac{SAT}{\delta}}{(1-\gamma)^3 n}}\notag\\
    &\le 8c_b\sum_{(s,a)\in\cS\times\cA}\sqrt{\frac{w\bar{N}_{T}(s,a)\log \frac{SAT}{\delta}}{(1-\gamma)^3}}\notag\\
    &\leq 8c_b\sqrt{\sum_{(s,a)\in\cS\times\cA}\frac{w\log \frac{SAT}{\delta}}{(1-\gamma)^3}}\sqrt{\sum_{(s,a)\in\cS\times\cA}\bar{N}_{T}(s,a)}\notag\\
    &= 8c_b \sqrt{\frac{SAwW\log \frac{SAT}{\delta}}{(1-\gamma)^3}}, \label{eq:weighted-V-1-2.1}
\end{align}
in which we define $\bar{N}_{T}(s,a) := \sum_{t=1}^T w_t\ind_{(s_{t},a_{t}) = (s,a)}$, and
\begin{align}
    &\quad\ \gamma \sum_{t=1}^T w_t\sum_{n = 1}^{N_{t}} \eta^{N_t}_n \left(V_{t_n}(s_{t_n+1}) - V^{\LCB}_{t_n}(s_{t_n+1}) \right)\notag\\
    &\leq \gamma \sum_{t=1}^T \left(V_{t}(s_{t+1}) - V^{\LCB}_{t}(s_{t+1}) \right) \sum_{n = N_{t+1}}^{N_{T}} w_{t_n}\eta^n_{N_t}, \label{eq:weighted-V-1-2.2}
\end{align}
in which we define $\widetilde{w}_t := \sum_{n = N_{t+1}}^{N_{T}} w_{t_n}\eta^n_{N_t}$. Note that $\widetilde{w}_t \le \frac{(3-\gamma)}{2}w$ by \eqref{eq:learning-rate5} in Lemma \ref{lem:learning-rate}.

Combining \eqref{eq:weighted-V-1-1} and \eqref{eq:weighted-V-1-2}, we have 
\begin{align}
    &\quad\ \sum_{t=1}^{T}w_t\left(V_{t}(s_{t+1})-V_{t}^{\LCB}(s_{t+1})\right) \notag\\
    &\le C_{\mathrm{v},1}\frac{w\left(SA\right)^{3/4}T^{1/4}}{(1-\gamma)^{5/4}}\left(\log\frac{SAT}{\delta}\right)^{5/4} + C_{\mathrm{v},1}\frac{SAw}{1-\gamma}\log T + 8c_b \sqrt{\frac{SAwW\log \frac{SAT}{\delta}}{(1-\gamma)^3}} \notag\\ &\quad + C_{\mathrm{v},1}w\sqrt{\frac{SA\log^2 T}{1-\gamma}\sum_{t=1}^T \big(V_t(s_{t+1}) - V^{\pi_{t+1}}(s_{t+1})\big)} + \gamma\sum_{t=1}^T \widetilde{w}_t\left(V_{t}(s_{t+1}) - V^{\LCB}_{t}(s_{t+1})\right) \label{eq:weighted-V-2}\\
    &\overset{(\mathrm{i})}{\le} C_{\mathrm{v},1}\frac{ew\left(SA\right)^{3/4}T^{1/4}}{(1-\gamma)^{9/4}}\left(\log\frac{SAT}{\delta}\right)^{5/4} + C_{\mathrm{v},1}\frac{eSAw}{(1-\gamma)^2}\log T + 8ec_b \sqrt{\frac{SAwW\log \frac{SAT}{\delta}}{(1-\gamma)^5}} \notag\\ &\quad + C_{\mathrm{v},1}ew\sqrt{\frac{SA\log^2 T}{(1-\gamma)^3}\sum_{t=1}^T \big(V_t(s_{t+1}) - V^{\pi_{t+1}}(s_{t+1})\big)} + \gamma^{H_{\mathrm{v}}}\frac{eW}{1-\gamma}\notag\\
    &\le C_{\mathrm{v},1}\frac{ew\left(SA\right)^{3/4}T^{1/4}}{(1-\gamma)^{9/4}}\left(\log\frac{SAT}{\delta}\right)^{5/4} + C_{\mathrm{v},1}\frac{eSAw}{(1-\gamma)^2}\log T + 8ec_b \sqrt{\frac{SAwW\log \frac{SAT}{\delta}}{(1-\gamma)^5}} \notag\\ &\quad + C_{\mathrm{v},1}ew\sqrt{\frac{SA\log^2 T}{(1-\gamma)^3}\sum_{t=1}^T \big(V_t(s_{t+1}) - V^{\pi_{t+1}}(s_{t+1})\big)} + eW, \label{eq:weighted-V-3}
\end{align}
where (i) is obtained by applying \eqref{eq:weighted-V-2} recursively for $H_{\mathrm{v}} := \lceil\frac{\ln \frac{1}{1-\gamma}}{\ln \frac{1}{\gamma}}\rceil$ times and the fact that $(1+\frac{1-\gamma}{2})^{H_{\mathrm{v}}} < (1+\frac{1-\gamma}{2})^{\lceil\frac{1}{1-\gamma}\rceil} < (1+\frac{1-\gamma}{2})^{\frac{2}{1-\gamma}} < e$, which is true since $\gamma \ge \frac{1}{2}$ and $H_{\mathrm{v}} < \lceil\frac{1}{1-\gamma}\rceil$. We also use the fact $\norm{V_t - V_t^\LCB}_\infty \le \frac{1}{1-\gamma}$.

\section{Proof of Lemma \ref{lem:VR_properties}}\label{sec:proof:VR_properties}

\subsection{Proof of Inequality \eqref{eq:VR-V-proximity}}

In the proof of this lemma, let us adopt the shorthand notation $t_i := t_i(s)$ to denote the index of the time step in which state $s$ is visited for the $i$-th time. We can use this notation for any state $s$ that has been visited at least once during the $T$ steps of our algorithm's execution. Note that for each $s$, both $V^{\rref}_t(s)$ and $V_{t}(s)$ are only updated when $t = t_i + 1$ with indices $i$ from $\{i ~:~ 1\leq t_i\leq T\}$. Thus, to establish \eqref{eq:VR-V-proximity}, it suffices to show that for any $s$ and its corresponding $1\leq t_i \leq T$,
\begin{align}
	\left| V_{t_i+1} (s) - V_{t_i+1}^{\rref} (s) \right| \leq 6.
	\label{eq:VR-V-proximity1}
\end{align}
In the following, let us fix any $s$ and prove \eqref{eq:VR-V-proximity1} by looking at three partition scenarios separately. 

\paragraph{Case 1 of \eqref{eq:VR-V-proximity1}.} Suppose that $t_i$ satisfies either Line \ref{alg-t-line:if-cond-1} or Line \ref{alg-t-line:elif-cond-1} of Algorithm~\ref{alg:main-t}, i.e., either
\begin{align}
	&V_{t_i+1} (s) - V_{t_i+1}^{\LCB} (s) > 3  \label{eq:condition-ti-update-1}
	\end{align}
	or

\begin{align}
	&V_{t_i+1} (s) - V_{t_i+1}^{\LCB} (s) \leq 3 \qquad \text{and} \qquad
	u^{\re}_{t_i}(s) = \mathsf{True}.  \label{eq:condition-ti-update-2}
\end{align}
Either of these two conditions would trigger the update of the reference value function and thus result in
$$V_{t_i+1} (s) = V_{t_i+1}^{\rref} (s),$$
which satisfies \eqref{eq:VR-V-proximity1} trivially. 

\paragraph{Case 2 of \eqref{eq:VR-V-proximity1}.} 
Suppose that $t_{i_0}$ is the first time such that both \eqref{eq:condition-ti-update-1} and \eqref{eq:condition-ti-update-2} are violated, i.e.,
\begin{align}
	{i_0} \coloneqq \min\left\{ j\mid V_{t_{j}+1}(s)-V_{t_j+1}^{\LCB}(s)\leq 3 ~\text{ and }~ u^{\re}_{t_j}(s) = \mathsf{False}\right\}.
	\label{eq:def-first-i-violate-condition}
\end{align}
In the following, let us show \eqref{eq:VR-V-proximity1} is true when $i = i_0$.

First, notice that the definition \eqref{eq:def-first-i-violate-condition} (first time both \eqref{eq:condition-ti-update-1} and \eqref{eq:condition-ti-update-2} are violated) suggests that the reference was updated last time $s$ was visited before the $t_{i_0}$-th step, i.e., $V^{\rref}_{t_{i_0-1}+1}$ was updated in the $t_{i_0-1}$-th step. Notice that since $s$ is not visited during the interval when the step index $t$ satisfies $t_{i_0-1}+1\le t \le t_{i_0}-1$, $V_{t+1}(s)$ is not updated during this period, nor is the block between Line \ref{alg-t-line:if-cond-1}-\ref{alg-t-line:elif-stat-1} of Algorithm \ref{alg:main-t} triggered (which would update the reference $V_{t+1}^\rref(s)$). Thus, we can conclude
\begin{align}
	V_{t_{i_0}}^{\rref}(s)=V_{t_{i_0-1}+1}^{\rref}(s)=V_{t_{i_0-1}+1}(s)=V_{t_{i_0}}(s). 
	\label{eq:lcb-final-value0}
\end{align}

Furthermore, since the reference $V^{\rref}_{t_{i_0-1}+1}$ was updated in the $t_{i_0-1}$-th step, this indicates either \eqref{eq:condition-ti-update-1} or \eqref{eq:condition-ti-update-2} is satisfied in this step. If \eqref{eq:condition-ti-update-1} is satisfied, then Line \ref{alg-t-line:if-cond-1} and \ref{alg-t-line:if-stat-1} of Algorithm \ref{alg:main-t} tell us that $\mathsf{True} = u^{\re}_{t_{i_0-1}+1}(s)$ and Line \ref{alg-t-line:elif-cond-1} is never triggered between Step $t_{i_0 -1}+1$ and Step $t_{i_0}-1$, which implies $u^{\re}_{t_{i_0}}(s) = u^{\re}_{t_{i_0-1}+1}(s) = \mathsf{True}$. However, this contradicts with the assumption $u^{\re}_{t_{i_0}}(s) = \mathsf{False}$ in \eqref{eq:def-first-i-violate-condition}. Hence, only \eqref{eq:condition-ti-update-2} is true at the $t_{i_0-1}$-th step, which allows us to conclude
\begin{align}\label{eq:lcb-final-value2}
	V_{t_{i_0}} (s) - V_{t_{i_0}}^{\LCB} (s)  = V_{t_{i_0-1}+1} (s) - V_{t_{i_0-1}+1}^{\LCB} (s) \leq 3.
\end{align}

On the other hand, the definition \eqref{eq:def-first-i-violate-condition} suggests that $V^{\rref}_{t_{i_0} + 1}(s)$ is not updated by the update block between Line \ref{alg-t-line:if-cond-1}-\ref{alg-t-line:elif-stat-1} at Step $t_{i_0}$, so we have
\begin{align}\label{eq:lcb-final-value}
	V_{t_{i_0}+1}^{\rref} (s) = V_{t_{i_0}}^{\rref} (s).
\end{align}

Finally, combining \eqref{eq:lcb-final-value0}, \eqref{eq:lcb-final-value2} and \eqref{eq:lcb-final-value}, we can derive that 
\begin{align}
	V^{\rref}_{t_{i_0}+1}(s) -V_{t_{i_0} + 1 }(s) 
		&\overset{\mathrm{(i)}}{=}  V^{\rref}_{t_{i_0}}(s) -V_{t_{i_0} + 1 }(s) 
		\overset{\mathrm{(ii)}}{=}  V_{t_{i_0}}(s) - V_{t_{i_0} + 1 }(s) \label{eq:ref-final-value0}\\
		& \overset{\mathrm{(iii)}}{\leq} V_{t_{i_0}}(s) - V^{\LCB}_{t_{i_0}}(s)\overset{\mathrm{(iv)}}{\leq} 3, 
		\label{eq:ref-final-value1}
\end{align}
where (i) is due to \eqref{eq:lcb-final-value}; (ii) is due to \eqref{eq:lcb-final-value0}; (iii) holds since $V_{t_{i_0} + 1}(s) \geq V^\star(s) \geq V^{\LCB}_{t_{i_0}}(s)$; (iv) follows from \eqref{eq:lcb-final-value2}.
In addition, let us point out that \eqref{eq:ref-final-value0}, together with the monotonicity of $V_t(s)$, implies
\begin{align}
	V^{\rref}_{t_{i_0}+1}(s) -V_{t_{i_0} + 1 }(s) =  V_{t_{i_0}}(s) - V_{t_{i_0} + 1 }(s) \geq 0. \label{eq:ref-final-value2}
\end{align}
By \eqref{eq:ref-final-value1} and \eqref{eq:ref-final-value2}, we can arrive at the desired bound \eqref{eq:VR-V-proximity1} with $i=i_0$.

\paragraph{Case 3 of \eqref{eq:VR-V-proximity1}.} For the last case, consider any $i>i_0$ (that does not satisfy Case 1). Then it must be true that $V_{t_{i}+1}(s)-V_{t_i+1}^{\LCB}(s) \leq 3$ and $u^{\re}_{t_i}(s) = \mathsf{False}$, due to the monotonicity of $V$ and $V^{\LCB}$ and our algorithmic design that $u^{\re}(s)$ stays False once set to False. 

It then follows that
\begin{align}
V_{t_i+1}^{\rref}(s) & \overset{(\mathrm{i})}{\leq}V_{t_{i_0}+1}^{\rref}(s)\overset{(\mathrm{ii})}{\leq}V_{t_{i_0}+1}(s)+3\overset{(\mathrm{iii})}{\leq}V_{t_{i_0}+1}^{\LCB}(s)+6 \overset{(\mathrm{iv})}{\leq}V^{\star}(s)+6\overset{(\mathrm{v})}{\leq}V_{t_{i}+1}(s)+6. \label{eq:VR-V-proximity1-1}
\end{align}
Here, (i) holds due to the monotonicity of $V_{t}$ and thus the monotonicity of $V^{\rref}_t$ ($V^{\rref}_t$ can only be updated from $V_t$); (ii) is a consequence of \eqref{eq:ref-final-value1}; (iii) comes from the definition of $i_0$ in \eqref{eq:def-first-i-violate-condition}; (iv) arises from the pessimism of $V^{\LCB}$ to $V^{\star}$ (see Lemma \ref{lem:Q_t-lcb}); (v) is valid since the optimism gives $V_{t_i + 1}(s) \geq V^{\star}(s)$ (see Lemma \ref{lem:Q_t-lower-bound}).  

On the other hand, since $V_{t}$ is non-increasing in $t$ and $V^{\rref}_t$ can only be updated from $V_t$, we have that $V_{t_i+1}(s) \le V_{t_i+1}^{\rref}(s)$, which, combined with \eqref{eq:VR-V-proximity1-1}, gives
$$0\leq V_{t_i+1}^{\rref}(s)- V_{t_i+1}(s) \leq 6$$
and thus justifies \eqref{eq:VR-V-proximity1} for this case. 

\bigskip
Overall, taking the three aforementioned cases in conjunction, we can establishe \eqref{eq:VR-V-proximity1} and thus \eqref{eq:VR-V-proximity} completely.

\subsection{Proof of Inequality \eqref{eq:VR_lazy}}

First, let us suppose that 
\begin{align}
	V^{\rref}_{t}(s_{t+1}) - V^{\rref}_{T}(s_{t+1}) \neq 0
	\label{eq:claim-Vt-VT}
\end{align}
holds for some $t < T$. Let $t$ be the largest number such that \eqref{eq:claim-Vt-VT} holds. Then it implies two possible scenarios: 
\paragraph{Case 1.} Suppose the following event is true
	\begin{equation}
	V_{t+1}(s_{t+1}) - V_{t+1}^{\LCB}(s_{t+1}) \leq 3 \quad \text{ and } \quad u^{\re}_{t}(s_{t+1}) = \mathsf{False}.
			\label{eq:claim-Vt-VT-cond1}
	\end{equation}
	Then for any $t' \ge t + 1$, one necessarily has
    \begin{align}
        V_{t'}(s_{t+1}) - V_{t'}^{\LCB}(s_{t+1})
		\leq V_{t+1}(s_{t+1}) - V_{t+1}^{\LCB}(s_{t+1}) \leq 3 \label{eq:claim-Vt-VT-conclusion1}
    \end{align}
    by the monotonicity of $V_{t}$ and $V_{t}^{\LCB}$. In addition, \eqref{eq:claim-Vt-VT-conclusion1} and our assumption $u^{\re}_t(s_{t+1}) = \mathsf{False}$ together can imply 
    \begin{align}
        u^{\re}_{t'}(s_{t+1}) = u^{\re}_t(s_{t+1}) = \mathsf{False}, \label{eq:claim-Vt-VT-conclusion2}
    \end{align}
    because both conditions in Line \ref{alg-t-line:if-cond-1} and Line \ref{alg-t-line:elif-cond-1} of Algorithm \ref{alg:main-t} would be violated at the $(t+1)$-th step and stay so thereafter. 
    
    \eqref{eq:claim-Vt-VT-conclusion1} and \eqref{eq:claim-Vt-VT-conclusion2} imply that the reference $V^{\rref}$ will no longer be updated after the $t$-th step, thus indicating that 
\begin{align}
	V^{\rref}_{t} = V^{\rref}_{t+1} = \cdots = V^{\rref}_{T}. 
\end{align}
However, this contradicts with the assumption \eqref{eq:claim-Vt-VT}, so by the contrapositive, \eqref{eq:claim-Vt-VT} would imply this case is impossible. 

\paragraph{Case 2.} As the complement event of Case 1, if the condition in \eqref{eq:claim-Vt-VT-cond1} is false, then it implies that either
\begin{equation}\label{eq:claim-Vt-VT-cond2-1}
	V_{t+1}(s_{t+1}) - V_{t+1}^{\LCB}(s_{t+1}) > 3,
\end{equation}
or
\begin{align}\label{eq:claim-Vt-VT-cond2-2}
	V_{t+1}(s_{t+1}) - V_{t+1}^{\LCB}(s_{t+1}) \leq 3 \quad \text{and} \quad
	u^{\re}_t(s_{t+1}) = \mathsf{True}.
\end{align}

In summary, we have just shown that \eqref{eq:claim-Vt-VT} can only occur if either \eqref{eq:claim-Vt-VT-cond2-1} or \eqref{eq:claim-Vt-VT-cond2-2} holds. In light of this, we can decompose the quantity of interest with respect to these two cases as follows: 
\begin{align} 
 &\quad\ \sum_{t=1}^{T}\left(V_{t}^{\rref}(s_{t+1})-V_{T}^{\rref}(s_{t+1})\right) \notag\\
 &= \sum_{t=1}^{T}\left(V_{t}^{\rref}(s_{t+1})-V_{T}^{\rref}(s_{t+1})\right)\ind\left\{V_{t}^{\rref}(s_{t+1})-V_{T}^{\rref}(s_{t+1})\neq 0\right\}\notag\\
 &\le \sum_{t=1}^{T}\left(V_{t}^{\rref}(s_{t+1})-V_{T}^{\rref}(s_{t+1})\right)\ind\left\{V_{t+1}(s_{t+1}) - V_{t+1}^{\LCB}(s_{t+1}) \leq 3 \text{ and } u^{\re}_t(s_{t+1}) = \mathsf{True}\right\}\notag\\
 & \qquad+ \sum_{t=1}^{T}\left(V_{t}(s_{t+1})-V_{t}^{\LCB}(s_{t+1})\right)\ind\left\{V_{t+1}(s_{t+1}) - V_{t+1}^{\LCB}(s_{t+1})>3\right\},
	\label{eq:Vt-VT-1}
\end{align}
where the inequality is because $V_{t}(s_{t+1}) \ge V_{t}^{\rref}(s_{t+1})$, which is due to the monotonicity of $V_t$ and the fact that $V^{\rref}$ can only be updated from $V$, as well as the fact that
\begin{equation*}
V^{\rref}_{T}(s_{t})  \geq V^\star(s_{t}) \geq V^{\LCB}_{t}(s_{t}), 
\end{equation*}
which is due to the optimism of $V$ and thus $V^{\rref}$ and the pessimism of $V^{\LCB}_t$ in Lemma \ref{lem:Q_t-lower-bound} and Lemma \ref{lem:Q_t-lcb} respectively. 

The first term in \eqref{eq:Vt-VT-1} can be easily bounded because for all $s\in \cS$,
\begin{align}
	\sum_{t=1}^{T}\ind\left\{V_{t+1}(s) - V_{t+1}^{\LCB}(s) \leq 3 \text{ and } 
	u^{\re}_t(s) = \mathsf{True}\right\} \leq 1. \label{eq:Vt-VT-1.1}
\end{align}
This is because for each $s\in \cS$, the above condition is satisfied at most once, in view of the monotonicity property of $V_t$ and $V_t^{\LCB}$ (their gap only grows smaller over time) as well as the update rule of $u_{\re}$ in Algorithm \ref{alg:main-t}. Consequently, one has
\begin{align*}
 &\quad\ \sum_{t=1}^{T}\left(V_{t}^{\rref}(s_{t+1})-V_{T}^{\rref}(s_{t+1})\right)\ind\left\{V_{t+1}(s_{t+1}) - V_{t+1}^{\LCB}(s_{t+1}) \leq 3 \text{ and } u^{\re}_t(s_{t+1}) = \mathsf{True}\right\}\notag\\
 & \leq \frac{1}{1-\gamma}\sum_{t=1}^{T}\ind\left\{V_{t+1}(s_{t+1}) - V_{t+1}^{\LCB}(s_{t+1}) \leq 3 \text{ and } u^{\re}_t(s_{t+1}) = \mathsf{True}\right\}\notag\\
 & \le \frac{1}{1-\gamma}\sum_{s\in\mathcal{S}}\sum_{t=1}^{T}\ind\left\{V_{t+1}(s) - V_{t+1}^{\LCB}(s) \leq 3 \text{ and } u^{\re}_t(s) = \mathsf{True}\right\} \notag\\
 & \leq \frac{1}{1-\gamma}\sum_{s\in\mathcal{S}}1=\frac{S}{1-\gamma}, 
\end{align*}
where the first inequality holds since $\norm{V_{t}^{\rref}-V_{T}^{\rref}}_\infty\leq \frac{1}{1-\gamma}$. 

To complete this proof, it remains to bound the second term in \eqref{eq:Vt-VT-1}. By the monotonicity in $V_t$ and $V_t^\LCB$, 
\begin{align}
&\quad\ \sum_{t=1}^{T}\left(V_{t}(s_{t+1})-V_{t}^{\LCB}(s_{t+1})\right)\mathds{1}\left\{V_{t+1}(s_{t+1}) - V_{t+1}^{\LCB}(s_{t+1})>3\right\} \nonumber \\
&\le \sum_{t=1}^{T}\left(V_{t}(s_{t+1})-V_{t}^{\LCB}(s_{t+1})\right)\mathds{1}\left\{V_{t}(s_{t+1}) - V_{t}^{\LCB}(s_{t+1})>3\right\} \nonumber \\
&\lesssim \frac{\left(SA\right)^{3/4}T^{1/4}}{(1-\gamma)^{13/4}}\left(\log\frac{SAT}{\delta}\right)^{5/4} + \frac{SA}{(1-\gamma)^6}\log \frac{SAT}{\delta} + \sqrt{\frac{SA\log^2 T}{(1-\gamma)^5}\sum_{t=1}^T \big(V_{t-1}(s_{t}) - V^{\pi_{t}}(s_{t})\big)},\label{eq:ref-error-upper-final}
\end{align}
where the last inequality is by \eqref{eq:main-lemma} in Lemma \ref{lem:Q_t-lcb} and $\norm{V_{t}-V_{t}^{\LCB}}_\infty\leq \frac{1}{1-\gamma}$. 

Finally, combining \eqref{eq:Vt-VT-1.1} and \eqref{eq:ref-error-upper-final}, we can conclude
\begin{align} 
&\quad\ \sum_{t=1}^{T}\left(V_{t}^{\rref}(s_{t+1})-V_{T}^{\rref}(s_{t+1})\right) \notag\\
    & \lesssim \frac{\left(SA\right)^{3/4}T^{1/4}}{(1-\gamma)^{13/4}}\left(\log\frac{SAT}{\delta}\right)^{5/4} + \frac{SA}{(1-\gamma)^6}\log \frac{SAT}{\delta} + \sqrt{\frac{SA\log^2 T}{(1-\gamma)^5}\sum_{t=1}^T \big(V_{t-1}(s_{t}) - V^{\pi_{t}}(s_{t})\big)}
\end{align}
as claimed.

\section{Proof of Lemma \ref{lem:Q_t-upper-bound}}\label{sec:proof:Q_t-upper-bound}

    For notational simplicity, let us adopt the shorthand notation $t_n := t_n(s_t, a_t)$ throughout this section. Recall the update rule of $Q^{\rref}_{t+1}(s_t, a_t)$ in Line \ref{alg-line:QR-update} of Algorithm \ref{alg:aux}. We can derive
	\begin{align}
		&\quad\ Q^{\rref}_t(s_t,a_t) - Q^{\star}(s_t,a_t) \notag\\
        &= Q^{\rref}_{t_{N_{t-1}(s_t,a_t)}+1}(s_t,a_t) - Q^{\star}(s_t,a_t) \notag\\
		&= \eta^{N_{t-1}(s_t,a_t)}_0 \left(Q^{\rref}_1(s_t,a_t) - Q^{\star}(s_t,a_t)\right) + \sum_{n = 1}^{N_{t-1}(s_t,a_t)} \eta^{N_{t-1}(s_t,a_t)}_n b^{\rref}_{t_{n}+1} \notag\\
		&\qquad+  \gamma\sum_{n = 1}^{N_{t-1}(s_t,a_t)} \eta^{N_{t-1}(s_t,a_t)}_n  \left(V_{t_n}(s_{t_n+1}) - V^{\rref}_{t_n}(s_{t_n+1})+\frac{1}{n}\sum_{i=1}^{n}\Big(V^{\rref}_{t_i}(s_{t_i+1})-P_{s_t,a_t}V^{\star}\Big)\right)\notag\\
		&\le \frac{\eta^{N_{t-1}(s_t,a_t)}_0}{1-\gamma} + \sum_{n = 1}^{N_{t-1}(s_t,a_t)} \eta^{N_{t-1}(s_t,a_t)}_n b^{\rref}_{t_{n}+1} \notag\\
		&\qquad+  \gamma\sum_{n = 1}^{N_{t-1}(s_t,a_t)} \eta^{N_{t-1}(s_t,a_t)}_n  \left(V_{t_n}(s_{t_n+1}) - V^{\rref}_{t_n}(s_{t_n+1})+\frac{1}{n}\sum_{i=1}^{n}\Big(V^{\rref}_{t_i}(s_{t_i+1})-P_{s_t,a_t}V^{\star}\Big)\right)\notag\\
		&\le \frac{\eta^{N_{t-1}(s_t,a_t)}_0}{1-\gamma} + B^{\rref}_{t}(s_t,a_t) + 2c_b\frac{\log^2\frac{SAT}{\delta}}{\left(N_{t-1}(s_t,a_t)\right)^{3/4}(1-\gamma)^2}\notag\\
		& \qquad + \gamma\sum_{n = 1}^{N_{t-1}(s_t,a_t)} \eta^{N_{t-1}(s_t,a_t)}_n  \left(V_{t_n}(s_{t_n+1}) - V^{\rref}_{t_n}(s_{t_n+1})+\frac{1}{n}\sum_{i=1}^{n}\Big(V^{\rref}_{t_i}(s_{t_i+1})-P_{s_t,a_t}V^{\star}\Big)\right),\notag 
	\end{align} 
    where the penultimate line follows from our initialization $Q^{\rref}_1(s_t, a_t) = \frac{1}{1-\gamma}$ and the last line follows from \eqref{eq:eta-b-sum-squeeze} with $B^{\rref}_{N_{t-1}+1} = B^{\rref}_t$. With the inequality above, we can sum over $1 \le t \le T$ and have
    \begin{align}
		&\quad\ \sum_{t = 1}^{T} Q^{\rref}_t(s_t, a_t) - Q^{\star}(s_t, a_t) \nonumber \\
		&\le \sum_{t = 1}^{T} \left(\frac{\eta^{N_{t-1}(s_t,a_t)}_0}{1-\gamma} + B^{\rref}_{t}(s_t,a_t) + 2c_b\frac{\log^2\frac{SAT}{\delta}}{\left(N_{t-1}(s_t,a_t)\right)^{3/4}(1-\gamma)^2}\right)\nonumber\\
		& \qquad + \gamma\sum_{t = 1}^{T} \sum_{n = 1}^{N_{t-1}(s_t,a_t)} \eta^{N_{t-1}(s_t,a_t)}_n \left(V_{t_n}(s_{t_n+1}) - V^{\rref}_{t_n}(s_{t_n+1})+\frac{1}{n}\sum_{i=1}^{n}V^{\rref}_{t_i}(s_{t_i+1})-P_{s_t,a_t}V^{\star}\right)\nonumber\\
		&\le \underbrace{\sum_{t = 1}^{T} \left(\frac{\eta^{N_{t-1}(s_t,a_t)}_0}{1-\gamma} + B^{\rref}_{t}(s_t,a_t) + 2c_b\frac{\log^2\frac{SAT}{\delta}}{\left(N_{t-1}(s_t,a_t)\right)^{3/4}(1-\gamma)^2}\right)}_{(\mathrm{I})} \nonumber\\
		& \qquad + \underbrace{\gamma\sum_{t = 1}^{T} \sum_{n = 1}^{N_{t-1}(s_t,a_t)} \eta^{N_{t-1}(s_t,a_t)}_n \Big(V_{t_n}(s_{t_n+1}) - V^\star(s_{t_n+1})\Big)}_{(\mathrm{II})} \nonumber\\
		& \qquad + \underbrace{\gamma\sum_{t = 1}^{T} \sum_{n = 1}^{N_{t-1}(s_t,a_t)} \eta^{N_{t-1}(s_t,a_t)}_n \left(V^\star(s_{t_n+1}) - V^{\rref}_{t_n}(s_{t_n+1})+\frac{1}{n}\sum_{i=1}^{n}V^{\rref}_{t_i}(s_{t_i+1})-P_{s_t,a_t}V^{\star}\right)}_{(\mathrm{III})}. \label{eq:QR-Q*-1} 
	\end{align}
	To finish the proof of this lemma, we control the three terms in \eqref{eq:QR-Q*-1} separately.

    \textbf{(I) in \eqref{eq:QR-Q*-1}}: We can bound (I) by making two observations. First, we have 
		\begin{align}\label{eq:QR-Q*-1.1}
			\sum_{t = 1}^{T}\eta^{N_{t-1}(s_t,a_t)}_0
			\leq \sum_{(s,a) \in \mathcal{S}\times \mathcal{A}} \sum_{n = 1}^{N_{T - 1}(s,a)}\eta^{n}_0 \leq SA ,
		\end{align}
		where the last inequality follows since $\eta_0^n = 0$ for all $n>0$. Then, we can observe that
		\begin{align}\label{eq:QR-Q*-1.2}
			\sum_{t = 1}^{T}  \frac{1}{\left(N_{t-1}(s_t, a_t) \right)^{3/4}} &= \sum_{(s,a) \in \mathcal{S}\times \mathcal{A}} \sum_{n=1}^{N_{T-1}(s,a) } \frac{1}{n^{3/4}} \leq \sum_{(s,a) \in \mathcal{S}\times \mathcal{A}} 4 \left(N_{T - 1}(s,a)\right)^{1/4} \notag\\
            &\leq 4\left(\sum_{(s,a) \in \mathcal{S}\times \mathcal{A}} 1\right)^{3/4}  \left(\sum_{(s,a) \in \mathcal{S}\times \mathcal{A}}  N_{T - 1}(s,a) \right)^{1/4} \leq 4(SA)^{3/4}T^{1/4}, 
		\end{align}
		where the penultimate inequality is due to H\"{o}lder's inequality.
		
		Applying the above two bounds to (I) yields
		\begin{align}
			&\quad\ \sum_{t = 1}^{T} \left(\frac{\eta^{N_{t-1}(s_t,a_t)}_0}{1-\gamma} + B^{\rref}_{t}(s_t,a_t) + 2c_b\frac{\log^2\frac{SAT}{\delta}}{(N_{t-1}(s_t,a_t))^{3/4}(1-\gamma)^2}\right) \notag\\
			&\leq \frac{SA}{1-\gamma}+8c_b(SA)^{3/4}T^{1/4}\frac{1}{(1-\gamma)^2}\log^2\frac{SAT}{\delta}+\sum_{t=1}^{T} B^{\rref}_{t}(s_t,a_t) .
		\end{align}
  
    \textbf{(II) in \eqref{eq:QR-Q*-1}}: It can be obtained via some algebra that
		\begin{align}
			&\quad\ \gamma\sum_{t = 1}^{T} \sum_{n = 1}^{N_{t-1}(s_t,a_t)} \eta^{N_{t-1}(s_t,a_t)}_n \big(V_{t_n}(s_{t_n+1}) - V^\star(s_{t_n+1})\big)\notag\\
			&= \gamma\sum_{l = 1}^{T} \sum_{N = N_{l}(s_l,a_l)}^{N_{T-1}(s_l,a_l)} \eta_{N_{l}(s_l,a_l)}^N \big(V_{l}(s_{l+1}) - V^\star(s_{l+1})\big)\notag\\
			&\le \gamma\left(1 + \frac{1-\gamma}{2}\right)\sum_{t = 1}^{T} \big(V_{t}(s_{t+1}) - V^\star(s_{t+1})\big).
		\end{align}
		Above, the second line replaces $t_n$ (resp. $n$) with $l$ (resp. $N_l(s_l, a_l)$), and the last line is due to \eqref{eq:learning-rate5} in Lemma \ref{lem:learning-rate}) and replaces $l$ with $t$ again.
		
		\textbf{(III) in \eqref{eq:QR-Q*-1}}: (III) can derived in the following manner:
		\begin{align*}
			&\quad\ \gamma\sum_{t = 1}^{T} \sum_{n = 1}^{N_{t-1}(s_t,a_t)} \eta^{N_{t-1}(s_t,a_t)}_n  \left(V^\star(s_{t_n+1}) - V^{\rref}_{t_n}(s_{t_n+1})+\frac{1}{n}\sum_{i=1}^{n}\Big(V^{\rref}_{t_i}(s_{t_i+1})-P_{s_t,a_t}V^{\star}\Big)\right)\nonumber \\
			&= \gamma\sum_{t = 1}^{T} \sum_{n = 1}^{N_{t-1}(s_t,a_t)} \eta^{N_{t-1}(s_t,a_t)}_n  \left(\left(P_{t_n} - P_{s_t,a_t}\right)\left(V^\star - V^{\rref}_{t_n}\right)+\frac{1}{n}\sum_{i=1}^{n}\Big(V^{\rref}_{t_i}(s_{t_i+1})-P_{s_t,a_t}V^{\rref}_{t_n}\Big)\right)\nonumber \\
			&= \gamma\sum_{t = 1}^{T} \sum_{N = N_{t}(s_t,a_t)}^{N_{T-1}(s_t,a_t)} \eta_{N_{t}(s_t,a_t)}^N \left(\left(P_{t} - P_{s_t,a_t}\right)\left(V^\star - V^{\rref}_{t}\right)+ \frac{\sum_{i=1}^{N_t(s_t,a_t)}\left(V^{\rref}_{t_i}(s_{t_i+1})-P_{s_t,a_t}V^{\rref}_t\right)}{N_t(s_t,a_t)}\right).
		\end{align*}
		Here, the second line is obtained by $V^\star(s_{t_n+1}) - V^{\rref}_{t_n}(s_{t_n+1}) = P_{t_n} \left(V^{\star} - V^{\rref}_{t_n} \right)$, and the third line is by a simple rearrangement of terms and replacement of $t_n$ (resp. $n$) with $t$ (resp. $N_t(s_t, a_t)$).  

    Furthermore, we can rearrange the terms in \eqref{eq:QR-Q*-1} and write
	\begin{align}
		&\quad\ \sum_{t = 1}^{T} \big(Q^{\rref}_t(s_t, a_t) - Q^{\star}(s_t, a_t)\big) \notag\\
		&\le  \gamma\left(1 + \frac{1-\gamma}{2}\right)\sum_{t = 1}^{T} \big(V_{t}(s_{t+1}) - V^\star(s_{t+1})\big) + \frac{SA}{1-\gamma} +8c_b(SA)^{3/4}T^{1/4}\frac{1}{(1-\gamma)^2}\log\frac{SAT}{\delta} + \underbrace{\sum_{t=1}^{T} B^{\rref}_{t}(s_t,a_t)}_{\cR_1}\notag\\
		&\quad + \underbrace{\gamma\sum_{t = 1}^{T} \sum_{N = N_{t}(s_t,a_t)}^{N_{T-1}(s_t,a_t)} \eta_{N_{t}(s_t,a_t)}^N \left[\left(P_{t} - P_{s_t,a_t}\right)\left(V^\star - V^{\rref}_{t}\right)+ \frac{\sum_{i=1}^{N_t(s_t,a_t)}\left(V^{\rref}_{t_i}(s_{t_i+1})-P_{s_t,a_t}V^{\rref}_t\right)}{N_t(s_t,a_t)}\right]}_{\cR_2},
		\label{eq:QR-Q*-2}
	\end{align}
	where the upper bounds of $\cR_1$ and $\cR_2$ are derived in Appendix \ref{sec:proof:R1-bound} and \ref{sec:proof:R2-bound} respectively. 

    Overall, taking the obtained upper bounds on $\cR_1$ and $\cR_2$, we can conclude with probability $1-\delta$ that
    \begin{align}
        &\quad\ \sum_{t = 1}^{T} \big(Q^{\rref}_t(s_t, a_t) - Q^{\star}(s_t, a_t)\big) \notag\\
        &\leq \gamma\left(1 + \frac{1-\gamma}{2}\right)\sum_{t = 1}^{T} \big(V_{t}(s_{t+1}) - V^\star(s_{t+1})\big) + C_{\mathrm{r},1}\Bigg(\sqrt{\frac{SAT}{1-\gamma}\log\frac{SAT}{\delta}} + \frac{SA}{(1-\gamma)^{15/4}}\log^{7/2}\frac{SAT}{\delta}\Bigg) \notag\\
        &\qquad + \frac{SA}{1-\gamma} +8c_b(SA)^{3/4}T^{1/4}\frac{1}{(1-\gamma)^2}\log\frac{SAT}{\delta}\notag\\
        &\qquad + C_{\mathrm{r},2}\Bigg(\sqrt{\frac{SAT}{1-\gamma}\log^{3}\frac{SAT}{\delta}} + \frac{SA}{(1-\gamma)^7}\log^{5/2} \frac{SAT}{\delta} + \sqrt{\frac{SA\log^2 T}{(1-\gamma)^5}\sum_{t=1}^T \big(V_{t-1}(s_{t}) - V^{\pi_{t}}(s_{t})\big)}\Bigg) \notag\\
        &\leq \gamma\left(1 + \frac{1-\gamma}{2}\right)\sum_{t = 1}^{T} \left(V_{t}(s_{t+1}) - V^\star(s_{t+1})\right) + C_{\mathrm{r},1}\Bigg(\sqrt{\frac{SAT}{1-\gamma}\log\frac{SAT}{\delta}} + \frac{SA}{(1-\gamma)^{15/4}}\log^{7/2}\frac{SAT}{\delta}\Bigg) \notag\\
        &\qquad + 8c_b\sqrt{\frac{SAT}{1-\gamma}\log\frac{SAT}{\delta}} + 8c_b\frac{SA}{(1-\gamma)^{7/2}}\log^{3/2}\frac{SAT}{\delta}\notag\\
        &\qquad + C_{\mathrm{r},2}\Bigg(\sqrt{\frac{SAT}{1-\gamma}\log^{3}\frac{SAT}{\delta}} + \frac{SA}{(1-\gamma)^7}\log^{5/2} \frac{SAT}{\delta} + \sqrt{\frac{SA\log^2 T}{(1-\gamma)^5}\sum_{t=1}^T \big(V_{t-1}(s_{t}) - V^{\pi_{t}}(s_{t})\big)}\Bigg) \notag\\
        &\leq \gamma\left(1 + \frac{1-\gamma}{2}\right)\sum_{t = 1}^{T} \left(V_{t}(s_{t+1}) - V^\star(s_{t+1})\right) + C_1\Bigg(\sqrt{\frac{SAT}{1-\gamma}\log^{3}\frac{SAT}{\delta}} + \frac{SA}{(1-\gamma)^7}\log^{7/2} \frac{SAT}{\delta} \notag\\
        &\qquad + \sqrt{\frac{SA\log^2 T}{(1-\gamma)^5}\sum_{t=1}^T \big(V_{t-1}(s_{t}) - V^{\pi_{t}}(s_{t})\big)}\Bigg), \label{eq:QR-Q*-UB-final}
    \end{align}
    in which the penultimate step uses $2ab \le a^2+b^2$, which follows from Cauchy-Schwarz. $C_1 > 0$ is a sufficiently large absolute constant. 

    At this point, we have proven the desired result.

 \subsection{Bounding $\cR_1$}\label{sec:proof:R1-bound}

By the definition of $B^{\rref}_t(s_t, a_t)$ in Line \ref{alg-line:BR} of Algorithm \ref{alg:aux}, we can decompose $\mathcal{R}_1$ as follows: 
\begin{align}
	\cR_1 &= \underbrace{c_b\sqrt{\frac{1}{1-\gamma}\log\frac{SAT}{\delta}}\sum_{t=1}^{T}\sqrt{\frac{\sigma^{\mathrm{adv}}_{t}(s_t, a_t) - \left(\mu^{\mathrm{adv}}_{t}(s_t, a_t)\right)^2}{N_t(s_t,a_t)}}}_{\cR_{1,1}} \nonumber\\
	&\qquad\qquad\qquad\qquad\qquad + \underbrace{c_b\sqrt{\log\frac{SAT}{\delta}}\sum_{t=1}^{T}\sqrt{\frac{\sigma^{\re}_{t}(s_t, a_t) - \left(\mu^{\re}_{t}(s_t, a_t)\right)^2}{N_t(s_t,a_t)}}}_{\cR_{1,2}}. \label{eq:R1-def1}
\end{align}
We proceed by bounding the two terms in \eqref{eq:R1-def1} separately.

\paragraph{Step 1: bounding $\cR_{1,1}$.}
Some manipulation on the summation part of $\cR_{1,1}$ gives
\begin{align}
	&\quad\ \sum_{t=1}^{T}\sqrt{\frac{\sigma^{\mathrm{adv}}_{t}(s_t, a_t) - \left(\mu^{\mathrm{adv}}_{t}(s_t, a_t)\right)^2}{N_t(s_t,a_t)}} \nonumber \\
	&\le \sum_{t=1}^{T}\sqrt{\frac{\sigma^{\mathrm{adv}}_{t}(s_t, a_t)}{N_t(s_t,a_t)}} \nonumber \\
	&= \sum_{t=1}^{T}\sqrt{\frac{\sum_{n=1}^{N_{t}(s_t,a_t)}\eta_{n}^{N_{t}(s_t,a_t)}\left(V_{t_n}(s_{t_n+1}) - V^{\rref}_{t_n}(s_{t_n+1})\right)^{2}}{N_t(s_t, a_t)}} \nonumber\\
    &\le \sum_{t = 1}^T\sqrt{\frac{81}{N_t(s_t,a_t)}} \nonumber\\
    &\le 18\sqrt{SAT}, \label{eq:R11-1}
\end{align}
in which the first inequality follows from the update rule of $\sigma_t^{\adv}$ in Line \ref{alg-line:sigma-adv} of Algorithm \ref{alg:aux}, and the second inequality uses the relation 
$ |V_t(s_t) - V^{\rref}_{t}(s_t) | \le 9$ (cf. \eqref{eq:VR-V-proximity}) and the property $\sum_{n = 1}^{N_t(s_t, a_t)} \eta^{N_t(s_t, a_t)}_n \leq 1$ (cf. \eqref{eq:learning-rate1} in Lemma \ref{lem:learning-rate}). The last inequality holds due to the following series of inequalities:
\begin{align}
	\sum_{t = 1}^T\sqrt{\frac{1}{N_t(s_t,a_t)}} &= \sum_{(s, a)\in \cS\times \cA} \sum_{n = 1}^{N_t(s, a)} \sqrt{\frac{1}{n}} \notag\\
    &\le 2 \sum_{(s, a)\in \cS\times \cA} \sqrt{N_t(s, a)} \notag \\
	& \le  2  \sqrt{ \sum_{(s, a)\in \cS\times \cA} 1 } \cdot \sqrt{\sum_{(s,a)\in \cS\times \cA}N_t(s, a)}  =  2\sqrt{SAT}, \label{eq:sqrt-N-sum}
\end{align}
where the last line is obtained by Cauchy-Schwarz and the fact that $\sum_{(s,a)}N_t(s, a) = T$.

Overall, we have
\begin{equation}
    \cR_{1,1} \le 9c_b\sqrt{\frac{SAT}{1-\gamma}\log\frac{SAT}{\delta}}. \label{eq:R11-final}
\end{equation}

\paragraph{Step 2: bounding $\cR_{1,2}$.}

By the update rules of $\mu_t^{\re}$ and $\sigma_t^{\re}$ in \eqref{eq:I4-1.1} and \eqref{eq:I4-1.2}, we can write the summation part of $\cR_{1,2}$ as
\begin{align}
	&\quad\ \sum_{t=1}^{T}\sqrt{\frac{\sigma^{\re}_{t}(s_t, a_t) - \left(\mu^{\re}_{t}(s_t, a_t)\right)^2}{N_t(s_t,a_t)}} \nonumber \\
	&= \sum_{t = 1}^T\sqrt{\frac{1}{N_t(s_t,a_t)}} \sqrt{ \frac{\sum_{n=1}^{N_{t}(s_t,a_t)}\left(V^{\rref}_{t_n}(s_{t_n+1}) \right)^2}{N_t(s_t,a_t)} - \left(\frac{\sum_{n=1}^{N_{t}(s_t,a_t)}V^{\rref}_{t_n}(s_{t_n+1})}{N_t(s_t,a_t)}\right)^2 }. \label{eq:R12-1}
\end{align}
To further bound the summation in \eqref{eq:R12-1}, we focus on a part of each summand, which can be decomposed into two terms as follows: 
\begin{align}
	&\quad\ \frac{\sum_{n=1}^{N_{t}(s_t,a_t)}\left(V^{\rref}_{t_n}(s_{t_n+1}) \right)^2}{N_t(s_t,a_t)} - \left(\frac{\sum_{n=1}^{N_{t}(s_t,a_t)}V^{\rref}_{t_n}(s_{t_n+1})}{N_t(s_t,a_t)}\right)^2 \nonumber\\
    &\leq \frac{\sum_{n=1}^{N_{t}(s_t,a_t)}\left(V^{\rref}_{t_n}(s_{t_n+1}) \right)^2 - \left(V^\star(s_{t_n+1}) \right)^2}{N_t(s_t,a_t)} + \frac{\sum_{n=1}^{N_{t}(s_t,a_t)}\left(V^\star(s_{t_n+1}) \right)^2}{N_t(s_t,a_t)} - \left(\frac{\sum_{n=1}^{N_{t}(s_t,a_t)}V^\star(s_{t_n+1})}{N_t(s_t,a_t)}\right)^2\nonumber\\
    &= \frac{\sum_{n=1}^{N_{t}(s_t,a_t)}\left(V^{\rref}_{t_n}(s_{t_n+1})+V^\star(s_{t_n+1})\right)\left(V^{\rref}_{t_n}(s_{t_n+1})-V^\star(s_{t_n+1})\right)}{N_t(s_t,a_t)} + \frac{\sum_{n=1}^{N_{t}(s_t,a_t)}\left(V^\star(s_{t_n+1}) \right)^2}{N_t(s_t,a_t)} - \left(\frac{\sum_{n=1}^{N_{t}(s_t,a_t)}V^\star(s_{t_n+1})}{N_t(s_t,a_t)}\right)^2\nonumber\\
	&\le \underbrace{ \frac{\sum_{n=1}^{N_{t}(s_t,a_t)}\frac{2}{1-\gamma}\left(V^{\rref}_{t_n}(s_{t_n+1})-V^\star(s_{t_n+1})\right)}{N_t(s_t,a_t)}}_{=:\cR_{1,2,1}^{(t)}}+ \underbrace{ \frac{\sum_{n=1}^{N_{t}(s_t,a_t)}\left(V^\star(s_{t_n+1}) \right)^2}{N_t(s_t,a_t)} - \left(\frac{\sum_{n=1}^{N_{t}(s_t,a_t)}V^\star(s_{t_n+1})}{N_t(s_t,a_t)}\right)^2}_{=: \cR_{1,2,2}^{(t)}}. \label{eq:R12-J}
\end{align}
Above, the last line is obtained due to the fact $\frac{1}{1-\gamma}\geq \norm{V^{\rref}_{t_n}}_\infty \ge \norm{V^{\star}}_\infty \geq 0$ for all $t_n \leq T$.

In order to arrive at a bound on \eqref{eq:R12-1}, we shall bound $\cR_{1,2,1}^{(t)}$ and $\cR_{1,2,2}^{(t)}$ separately. 

Define $J_t(s_t, a_t)$ as
\begin{equation} \label{eq:Phi-def}
		J_t(s_t, a_t) := \sum_{n=1}^{N_{t}(s_t,a_t)}\left(V^{\rref}_{t_n}(s_{t_n+1})-V^\star(s_{t_n+1})\right)\ind\left(V^{\rref}_{t_n}(s_{t_n+1})-V^\star(s_{t_n+1}) > 12\right).
	\end{equation}
$\cR_{1,2,1}^{(t)}$ can be simply decomposed as follows
	\begin{align}
		\cR_{1,2,1}^{(t)} &= \frac{2}{(1-\gamma)N_{t}(s_t,a_t)} \left( \sum_{n=1}^{N_{t}(s_t,a_t)}\left(V^{\rref}_{t_n}(s_{t_n+1})-V^\star(s_{t_n+1})\right)\ind\left(V^{\rref}_{t_n}(s_{t_n+1})-V^\star(s_{t_n+1}) \le 9\right) + J_t(s_t, a_t) \right) \nonumber \\
		&\le \frac{24}{1-\gamma} + \frac{2}{(1-\gamma)N_{t}(s_t,a_t)} J_t(s_t, a_t). \label{eq:J1}
	\end{align}
	
	On the other hand, we claim
	\begin{align} \label{eq:J2}
		\cR_{1,2,2}^{(t)} \lesssim \Var_{s_t, a_t}(V^{\star}) + \frac{1}{(1-\gamma)^2}\sqrt{\frac{\log\frac{SAT}{\delta}}{N_t(s_t,a_t)}},
	\end{align}
	which will be justified in Appendix \ref{sec:proof:eq:J2}. 

Substituting \eqref{eq:J1} and \eqref{eq:J2} back into \eqref{eq:R12-J} and \eqref{eq:R12-1} allows us to have
\begin{align}
	&\quad\ \sum_{t=1}^{T}\sqrt{\frac{\sigma^{\re}_{t}(s_t, a_t) - \left(\mu^{\re}_{t}(s_t, a_t)\right)^2}{N_t(s_t,a_t)}}  \nonumber \\
	& \lesssim \sum_{t = 1}^T\sqrt{\frac{1}{N_t(s_t,a_t)}} \sqrt{\frac{1}{1-\gamma}+\frac{J_t(s_t, a_t)}{(1-\gamma)N_{t}(s_t,a_t)} + \Var_{s_t, a_t}(V^{\star}) + \frac{1}{(1-\gamma)^2}\sqrt{\frac{\log\frac{SAT}{\delta}}{N_t(s_t,a_t)}}} \nonumber\\
	&\leq \sum_{t = 1}^T \Bigg( \sqrt{\frac{1}{(1-\gamma)N_{t}(s_t,a_t)}}  + \frac{\sqrt{J_t(s_t, a_t)}}{N_{t}(s_t,a_t)\sqrt{1-\gamma}} + \sqrt{ \frac{ \Var_{s_t, a_t}(V^{\star})}{N_{t}(s_t,a_t)}}  + \frac{\log^{1/4}\frac{SAT}{\delta}}{(1-\gamma) \big(N_{t}(s_t,a_t) \big)^{3/4}} \Bigg) \nonumber\\
	& \lesssim \sqrt{\frac{SAT}{1-\gamma}} +\sum_{t=1}^T \frac{\sqrt{J_t(s_t, a_t)}}{N_{t}(s_t,a_t)\sqrt{1-\gamma}} + \sum_{t = 1}^T\sqrt{ \frac{ \Var_{s_t, a_t}(V^{\star})}{N_{t}(s_t,a_t)}} + \frac{(SA)^{3/4} }{1-\gamma}\left(T\log\frac{SAT}{\delta} \right)^{1/4} ,
	\label{eq:R12-final}
\end{align}
in which the last line follows from \eqref{eq:sqrt-N-sum} and \eqref{eq:QR-Q*-1.2}.

\paragraph{Step 3: combining the preceding bounds.}

Finally, substituting \eqref{eq:R11-final} and \eqref{eq:R12-final} back into \eqref{eq:R1-def1} leads to
\begin{align*}
	\cR_1 &\lesssim \sqrt{\frac{SAT}{1-\gamma}\log\frac{SAT}{\delta}} + \sqrt{\log\frac{SAT}{\delta}}\sum_{t=1}^{T}\sqrt{\frac{\sigma^{\mathrm{ref}}_{t}(s_t, a_t) - \left(\mu^{\mathrm{ref}}_{t}(s_t, a_t)\right)^2}{N_t(s_t,a_t)}}\\
	&\lesssim \sqrt{\frac{SAT}{1-\gamma}\log\frac{SAT}{\delta}} + \frac{(SA)^{3/4} }{1-\gamma}T^{1/4}\log^{3/4}\frac{SAT}{\delta}  + \sqrt{\log\frac{SAT}{\delta} }\sum_{t = 1}^T\sqrt{ \frac{ \Var_{s_t, a_t}(V^{\star})}{N_{t}(s_t,a_t)}} +\sqrt{\frac{\log\frac{SAT}{\delta}}{1-\gamma}}\sum_{t=1}^T \frac{\sqrt{J_t(s_t, a_t)}}{N_{t}(s_t,a_t)}\\
	&\overset{\mathrm{(i)}}{\lesssim} \sqrt{\frac{SAT}{1-\gamma}\log\frac{SAT}{\delta}}  + \frac{(SA)^{3/4} }{1-\gamma}T^{1/4}\log^{3/4}\frac{SAT}{\delta}  + \frac{SA}{(1-\gamma)^{15/4}}\log^{7/2}\frac{SAT}{\delta}\\
	&\overset{\mathrm{(ii)}}{\lesssim} \sqrt{\frac{SAT}{1-\gamma}\log\frac{SAT}{\delta}} + \frac{SA}{(1-\gamma)^{15/4}}\log^{7/2}\frac{SAT}{\delta}.
\end{align*} 
Here, (i) is true due to the following two inequalities
\begin{align} 
	\sum_{t = 1}^T\sqrt{ \frac{ \Var_{s_t, a_t}(V^{\star})}{N_{t}(s_t,a_t)}} & \lesssim \sqrt{\frac{SAT}{1-\gamma}} + \frac{SA}{(1-\gamma)^{3}}\log\frac{SAT}{\delta},\label{eq:R1-1}  \\
	\sum_{t=1}^T \frac{\sqrt{J_t(s_t, a_t)}}{N_{t}(s_t,a_t)} & \lesssim \sqrt{SAT} + \frac{SA}{(1-\gamma)^{13/4}}\log^3\frac{SAT}{\delta}, \label{eq:R1-2}
\end{align}
whose proofs are deferred to Appendix \ref{sec:proof:eq:R1-1} and Appendix \ref{sec:proof:eq:R1-2}, respectively. In addition, (ii) uses $2ab \le a^2 + b^2$ and is obtained via the following logic:
\begin{align*}
	\frac{(SA)^{3/4} }{1-\gamma}T^{1/4}\log^{3/4}\frac{SAT}{\delta} &=\left(\frac{SA}{(1-\gamma)^2}\log\frac{SAT}{\delta}\right)^{1/2}\left(SAT\log\frac{SAT}{\delta}\right)^{1/4}\\
	&\lesssim \frac{SA}{(1-\gamma)^2}\log\frac{SAT}{\delta} + \sqrt{\frac{SAT}{1-\gamma}\log\frac{SAT}{\delta}}.
\end{align*}
This concludes the proof of the advertised upper bound on $\mathcal{R}_1$ that leads to \eqref{eq:QR-Q*-UB-final}. 

\subsubsection{Proof of Inequality \eqref{eq:J2}}\label{sec:proof:eq:J2}

We begin by constructing the necessary notation and verify the conditions for Lemma \ref{lemma:freedman-application}. Consider any $N \in [T]$. Let
	\begin{align}
		W_i := ({V}^{\star})^{2} \qquad \text{and} \qquad u_i(s,a, N) := \frac{1}{N}. \label{J2-Wu}
	\end{align}
	Clearly, $W_i$ and $u_i(s,a,N)$ are deterministic at the end of the $(i-1)$-th step. Moreover, define
	\begin{align}
		C_{\mathrm{u}} := \frac{1}{N} \qquad \text{and} \qquad \norm{W_i}_\infty \leq \frac{1}{(1-\gamma)^2} =: C_{\mathrm{w}}. \label{J2-cw-cu}
	\end{align}
    Lastly, it is obvious that
    \begin{align}
		0 \leq \sum_{n=1}^N u_{t_n(s,a)}(s,a, N) = 1 \label{eq:J2-cond-u}
	\end{align}
	holds for all $(N,s,a) \in [T] \times \cS\times \cA$.

    By the definition \eqref{eq:var-def}
$$\Var_{s_t, a_t}(V^{\star}) = P_{s_t, a_t}(V^{\star})^2 -  \left(P_{s_t, a_t} V^{\star} \right)^2,$$
we have that with probability at least $1-\delta$,
\begin{align*} 
	&\quad\  \left|\frac{1}{{N_t}}\sum_{n = 1}^{{N_t}}  V^{\star}(s_{t_n+1})^2 - \left(\frac{1}{{N_t}}\sum_{n = 1}^{{N_t}} V^{\star}(s_{t_n+1}) \right)^2 - \Var_{s_t, a_t}(V^{\star})\right|\\
	&\leq \left|\frac{1}{{N_t}}\sum_{n = 1}^{{N_t}}  V^{\star}(s_{t_n+1})^2 - P_{s_t, a_t}(V^{\star})^2 \right| 
	+ \left| \left(\frac{1}{{N_t}}\sum_{n = 1}^{{N_t}} V^{\star}(s_{t_n+1}) \right)^2 - (P_{s_t, a_t}V^{\star})^2 \right| \\
	&\lesssim \frac{1}{(1-\gamma)^2}\sqrt{\frac{\log\frac{SAT}{\delta}}{{N_t}}} + \left| \left(\frac{1}{{N_t}}\sum_{n = 1}^{{N_t}} V^{\star}(s_{t_n+1}) \right)^2 - (P_{s_t, a_t}V^{\star})^2 \right| \\
    &= \frac{1}{(1-\gamma)^2}\sqrt{\frac{\log\frac{SAT}{\delta}}{{N_t}}} + \left| \frac{1}{{N_t}}\sum_{n = 1}^{{N_t}} V^{\star}(s_{t_n+1}) - P_{s_t, a_t}V^{\star}  \right| \cdot  \left| \frac{1}{{N_t}}\sum_{n = 1}^{{N_t}} V^{\star}(s_{t_n+1}) + P_{s_t, a_t}V^{\star} \right| \\
	&\lesssim \frac{1}{(1-\gamma)^2}\sqrt{\frac{\log\frac{SAT}{\delta}}{{N_t}}},
\end{align*}
where the first inequality is by the definition \eqref{eq:var-def}:
\begin{equation*}
    \Var_{s_t, a_t}(V^{\star}) = P_{s_t, a_t}(V^{\star})^2 -  \left(P_{s_t, a_t} V^{\star} \right)^2,
\end{equation*}
and the second inequality is obtained through the application of Lemma \ref{lemma:freedman-application} with \eqref{J2-Wu}, \eqref{J2-cw-cu} and $(N,s,a)=(N_t,s_t, a_t)$, which allows us to conclude that with probability at least $1-\delta$, 
    \begin{align*}
    	\left| \frac{1}{{N_t}}\sum_{n = 1}^{{N_t}} \big( V^{\star}(s_{t_n+1}) \big)^2 - P_{s_t, a_t}(V^{\star})^2 \right| = \left| \frac{1}{{N_t}}\sum_{n = 1}^{{N_t}} \left(P_{t_n} - P_{s_t, a_t} \right) \left(V^{\star} \right)^2  \right|
    	\lesssim \frac{1}{(1-\gamma)^2}\sqrt{\frac{\log\frac{SAT}{\delta}}{{N_t}}},
    \end{align*}
and similarly the last line is obtained with the fact $\norm{V^{\star}}_{\infty} \leq \frac{1}{1-\gamma}$ and an application of Lemma \ref{lemma:freedman-application}, which allows us to conclude that with probability at least $1-\delta$,
\begin{align*}
	\left| \frac{1}{{N_t}}\sum_{n = 1}^{{N_t}} V^{\star}(s_{t_n+1}) - P_{s_t, a_t}V^{\star} \right|
	= \left| \frac{1}{{N_t}}\sum_{n = 1}^{{N_t}} \left(P_{t_n} - P_{s_t, a_t} \right) V^{\star} \right|
	\lesssim \frac{1}{1-\gamma}\sqrt{\frac{\log\frac{SAT}{\delta}}{{N_t}}}.
\end{align*}

\subsubsection{Proof of Inequality \eqref{eq:R1-1}}\label{sec:proof:eq:R1-1}

Let us begin with the following manipulation
\begin{align}
	\sum_{t = 1}^T\sqrt{ \frac{ \Var_{s_t, a_t}(V^{\star})}{N_{t}(s_t,a_t)}} &= \sum_{(s, a)\in\cS\times\cA} \sum_{n = 1}^{N_t(s, a)} \sqrt{\frac{\Var_{s, a}(V^{\star})}{n}} \notag\\
    &\overset{\mathrm{(i)}}{\le} 2\sum_{(s, a)\in\cS\times\cA}   \sqrt{N_t(s, a)\Var_{s, a}(V^{\star})} \notag\\
    &\overset{\mathrm{(ii)}}{\le} 2\sqrt{\sum_{(s, a)\in\cS\times\cA} 1}  \cdot \sqrt{ \sum_{(s, a)\in\cS\times\cA}   N_t(s, a)\Var_{s, a}(V^{\star})} \notag \\
	&= 2 \sqrt{SA}\sqrt{\sum_{t = 1}^T \Var_{s_t, a_t}(V^{\star})}, \label{eq:R1-1.1}
\end{align}
in which (i) is due to $\sum_{n = 1}^{N} 1/{\sqrt{n}} \leq 2\sqrt{N}$; (ii) is by the Cauchy-Schwarz inequality. 

It only remains to bound the variance term in \eqref{eq:R1-1.1}. To this end, let us decompose it further into
\begin{align}
	\sum_{t = 1}^T \Var_{s_t, a_t}(V^{\star}) &\leq 
	\underbrace{\sum_{t = 1}^T \Var_{s_t, a_t} \big( V^{\pi_t}\big)}_{\alpha_1} + \underbrace{\sum_{t = 1}^T \left|\Var_{s_t, a_t}(V^{\star}) - \Var_{s_t, a_t}(V^{\pi_t})\right|}_{\alpha_2} \notag \\
	&\lesssim \frac{T}{1-\gamma} + \sqrt{\frac{SAT\log\frac{SAT}{\delta}}{(1-\gamma)^7}} + \frac{SA\log^2\frac{SAT}{\delta}}{(1-\gamma)^{9/2}} + \frac{1}{(1-\gamma)^2}\sqrt{T\log\frac{SAT}{\delta}} + \frac{\log T}{(1-\gamma)^2}\sqrt{SAT}\notag\\
    &\lesssim \frac{T}{1-\gamma} + \sqrt{\frac{SAT\log\frac{SAT}{\delta}}{(1-\gamma)^7}} + \frac{SA\log^2\frac{SAT}{\delta}}{(1-\gamma)^{9/2}} + \frac{\log T}{(1-\gamma)^2}\sqrt{SAT}.\label{eq:R1-1.2}
\end{align}
The above inequality can be obtained with an upper bound on $\alpha_1$ and $\alpha_2$ respectively. Let us use $e_{s}$ to denote the $S$-dimensional canonical basis vector whose entry is $1$ in the $s$-th dimension. To bound $\alpha_1$ in \eqref{eq:R1-1.2}, we can first write $\alpha_1$ as follows:
\begin{align}
	\Var_{s_t, a_t} \big( V^{\pi_t}\big) & =P_{s_t,a_t}\left(V^{\pi_t}\right)^2-\left(P_{s_t,a_t}V^{\pi_t}\right)^2\nonumber \\
	& =P_{s_t,a_t}\left(V^{\pi_t}\right)^2 -\frac{1}{\gamma^{2}}\left(r(s_t,a_t)-V^{\pi_t}(s_t)\right)^2\nonumber \\
	& = P_{s_t,a_t}\left(V^{\pi_t}\right)^2-\frac{1}{\gamma^{2}}r^2(s_t,a_t)-\frac{1}{\gamma^{2}}\left(V^{\pi_t}(s_t)\right)^2 +\frac{2}{\gamma^{2}}V^{\pi_t}(s_t)\cdot r(s_t,a_t)\nonumber \\
	& \leq\frac{1}{\gamma^{2}}\left(\gamma^{2}P_{s_t,a_t}-e_{s_t}^\top\right)\left(V^{\pi_t}\right)^2 +\frac{2}{\gamma^{2}}V^{\pi_t}(s_t)\cdot r(s_t,a_t).
	\label{eq:v-bound}
\end{align}
Here, the second line follows from the Bellman optimality equation
$V^{\pi_t}(s_t) = Q^{\pi_t}(s_t,a_t) =r(s_t, a_t)+\gamma P_{s_t,a_t}V^{\pi_t}$. Now, we are ready to bound $\alpha_1$ under two partition cases. If $\gamma < \frac{1}{2}$, then
\begin{align*}
    \alpha_1 &= \sum_{t = 1}^T \Var_{s_t, a_t} \big( V^{\pi_t}\big)\\
    &= \sum_{t = 1}^T \left(P_{s_t,a_t}\left(V^{\pi_t}\right)^2-\left(P_{s_t,a_t}V^{\pi_t}\right)^2\right)\\
    &\le \frac{T}{(1-\gamma)^2}\\
    &\le \frac{2T}{1-\gamma}.
\end{align*}
If $\gamma \ge \frac{1}{2}$, by \eqref{eq:v-bound}, we can write
\begin{align*}
	\alpha_1 &= \sum_{t = 1}^T \Var_{s_t, a_t} \left( V^{\pi_t}\right)\\
	&\le \sum_{t = 1}^T\left(\frac{1}{\gamma^{2}}\left(\gamma^{2}P_{s_t,a_t}-e_{s_t}^\top\right)\left(V^{\pi_t}\right)^2 +\frac{2}{\gamma^{2}}V^{\pi_t}(s_t)\cdot r(s_t,a_t)\right)\\
	&\le \sum_{t = 1}^T\left(\frac{1}{\gamma^{2}}\left|\left(\gamma^{2}P_{s_t,a_t}-e_{s_t}^\top\right)\left(V^{\pi_t}\right)^2\right| +\frac{2}{\gamma^{2}}V^{\pi_t}(s_t)\cdot r(s_t,a_t)\right)\\
	&= \sum_{t = 1}^T\left(\frac{1}{\gamma^{2}}\left|\left[(1-\gamma)e_{s_t}^\top + \gamma (e_{s_t}^\top - \gamma P_{s_t,a_t})\right]\left(V^{\pi_t}\right)^2\right| +\frac{2}{\gamma^{2}}V^{\pi_t}(s_t)\cdot r(s_t,a_t)\right)\\
	&\overset{(\mathrm{i})}{\leq} \sum_{t = 1}^T\left(\frac{1-\gamma}{\gamma^2}\norm{e_{s_t}^\top}_1\norm{V^{\pi_t}}^2_\infty + \frac{1}{\gamma}\left|(e_{s_t}^\top - \gamma P_{s_t,a_t})V^{\pi_t}\right|\norm{V^{\pi_t}}_\infty + \frac{2}{\gamma^{2}}V^{\pi_t}(s_t)\right)\\
	&\overset{(\mathrm{ii})}{=} \sum_{t = 1}^T\left(\frac{1-\gamma}{\gamma^2}\norm{e_{s_t}^\top}_1\norm{V^{\pi_t}}^2_\infty + \frac{1}{\gamma}\left|r(s_t,a_t)\right|\norm{V^{\pi_t}}_\infty + \frac{2}{\gamma^{2}}V^{\pi_t}(s_t)\right)\\
	&\overset{(\mathrm{iii})}{\leq} \sum_{t = 1}^T\left(\frac{1}{2\gamma^2(1-\gamma)} + \frac{1}{\gamma(1-\gamma)} + \frac{2}{\gamma^2(1-\gamma)}\right)\\
	&\overset{(\mathrm{iv})}{\leq} \frac{12T}{1-\gamma},
\end{align*}
where (i) is due to the H\"{o}lder inequality and $\norm{r}_\infty \le 1$; (ii) is due to the Bellman optimality equation $V^{\pi_t}(s_t) = Q^{\pi_t}(s_t,a_t) =r(s_t, a_t)+\gamma P_{s_t,a_t}V^{\pi_t}$ and thus $(e_{s_t}^\top - \gamma P_{s_t,a_t})V^{\pi_t} = r(s_t,a_t)$; (iii) is due to $\norm{V^{\pi_t}}_\infty \le \frac{1}{1-\gamma}$ and $\norm{e_{s_t}}_1 = 1$ and $\norm{r}_\infty \le 1$; (iv) is due to $\gamma \ge \frac{1}{2}$.

On the other hand, $\alpha_2$ in \eqref{eq:R1-1.2} can be bounded as follows 
\begin{align}
	\alpha_2 &= \sum_{t = 1}^T \left|\Var_{s_t, a_t}(V^{\star}) - \Var_{s_t, a_t}(V^{\pi_t})\right|\nonumber\\
	& =\sum_{t=1}^{T}\left|P_{s_{t},a_{t}}(V^{\star})^{2}-\left(P_{s_{t},a_{t}}V^{\star}\right)^{2}-P_{s_{t},a_{t}}(V^{\pi_{t}})^{2}+\left(P_{s_{t},a_{t}}V^{\pi_{t}}\right)^{2}\right| \notag\\
	& \leq\sum_{t=1}^{T}\left(\left|P_{s_{t},a_{t}}\left(\left(V^{\star}-V^{\pi_{t}}\right) \circ\left(V^{\star}+V^{\pi_{t}}\right)\right)\right|+\left|\left(P_{s_{t},a_{t}}V^{\star}\right)^{2}-\left(P_{s_{t},a_{t}}V^{\pi_{t}}\right)^{2}\right|\right) \notag\\
	&\overset{(\mathrm{i})}{\leq} \frac{4}{1-\gamma}\sum_{t=1}^{T}P_{s_{t},a_{t}}\left(V^{\star}-V^{\pi_{t}}\right)  \notag\\
	& =\frac{4}{1-\gamma}\sum_{t=1}^{T}\left\{ V^{\star}(s_{t+1})-V^{\pi_{t}}(s_{t+1})+\left(P_{s_{t},a_{t}}-P_{t}\right)\left(V^{\star}-V^{\pi_{t}}\right)\right\}  \nonumber\\
	& \overset{(\mathrm{ii})}{=} \frac{4}{1-\gamma} \sum_{t=1}^{T}\left(\Delta_{t}+\phi_{t}\right) \nonumber\\
	& \overset{(\mathrm{iii})}{\lesssim} \sqrt{\frac{SAT\log\frac{SAT}{\delta}}{(1-\gamma)^7}} + \frac{SA\log^2\frac{SAT}{\delta}}{(1-\gamma)^{9/2}} + \frac{1}{(1-\gamma)^2}\sqrt{T\log\frac{SAT}{\delta}} + \frac{\log T}{(1-\gamma)^2}\sqrt{SAT},
	\label{eq:R1-1.3}
\end{align}
Above, (i) is obtained by observing that $V^{\star}-V^{\pi_{t}}\geq 0$ and 
\begin{align*}
	\left|P_{s_{t},a_{t}}\left(\left(V^{\star}-V^{\pi_{t}}\right)^\top\left(V^{\star}+V^{\pi_{t}}\right)\right)\right| & \leq P_{s_{t},a_{t}}\left(V^{\star}-V^{\pi_{t}}\right)\left(\norm{ V^{\star}}_{\infty}+\norm{V^{\pi_{t}}}_{\infty}\right)\\
	& \leq \frac{2}{1-\gamma}P_{s_{t},a_{t}}\left(V^{\star}-V^{\pi_{t}}\right),\\
	\left|\left(P_{s_{t},a_{t}}V^{\star}\right)^{2}-\left(P_{s_{t},a_{t}}V^{\pi_{t}}\right)^{2}\right| & \leq\left|P_{s_{t},a_{t}}\left(V^{\star}-V^{\pi_{t}}\right)\right| \cdot \left|P_{s_{t},a_{t}}\left(V^{\star}+V^{\pi_{t}}\right)\right|\\
	& \leq \frac{2}{1-\gamma}P_{s_{t},a_{t}}\left(V^{\star}-V^{\pi_{t}}\right).
\end{align*}
In (ii) of \eqref{eq:R1-1.3}, we define
\begin{align}
	\Delta_t := V^{\star}(s_{t+1})-V^{\pi_{t}}(s_{t+1}),  \qquad \phi_t := \left(P_{s_{t},a_{t}}-P_{t}\right)\left(V^{\star}-V^{\pi_{t}}\right).
\end{align}
(iii) in \eqref{eq:R1-1.3} results from the following two bounds:
\begin{subequations}
	\begin{align}
		\sum_{t=1}^{T}\Delta_{t} & \lesssim \frac{SA\log^2\frac{SAT}{\delta}}{(1-\gamma)^{7/2}} + \sqrt{\frac{SAT\log\frac{SAT}{\delta}}{(1-\gamma)^5}} + \frac{\log T}{1-\gamma}\sqrt{SAT}, \label{eq:R1-1.3-delta}\\
		\sum_{t=1}^{T}\phi_{t} & \lesssim \frac{1}{1-\gamma}\sqrt{T\log\frac{SAT}{\delta}}\label{eq:R1-1.3-phi}.
	\end{align}
\end{subequations}
Now, we show how \eqref{eq:R1-1.3-delta} and \eqref{eq:R1-1.3-phi} are obtained. For \eqref{eq:R1-1.3-delta}, we decompose it into two terms as follows:
\begin{align}
    \sum_{t=1}^T \Delta_{t} &= \sum_{t=1}^T \big(V^{\star}(s_{t+1})-V^{\pi_{t}}(s_{t+1})\big) = \sum_{t=1}^T \big(\underbrace{V^{\star}(s_{t+1})-V^{\pi_{t+1}}(s_{t+1})}_{=: \Delta_{t,1}} + \underbrace{V^{\pi_{t+1}}(s_{t+1}) - V^{\pi_{t}}(s_{t+1})}_{=: \Delta_{t,2}}\big).
\end{align}
Notice $\sum_{t=1}^T\Delta_{t,1}$ is the regret, so we can bound it with our crude regret bound in Lemma \ref{lem:crude}, which gives
\begin{align}
    \sum_{t=1}^T \Delta_{t,1} &\lesssim \frac{SA\log^2\frac{SAT}{\delta}}{(1-\gamma)^{7/2}} + \sqrt{\frac{SAT\log\frac{SAT}{\delta}}{(1-\gamma)^5}}.\label{eq:delta-t1-bound}
\end{align}
On the other hand, notice $\Delta_{t,2}$ coincides with $\xi_{t,3}$ (cf. \eqref{eq:zeta-decomposition-def}) in the proof of Lemma \ref{lem:xi}. Thus, we can invoke \eqref{eq:xi-3-bound} and have
\begin{align}
    \sum_{t=1}^T \Delta_{t,2} &= \sum_{t=1}^T \big(V^{\pi_{t+1}}(s_{t+1}) - V^{\pi_t}(s_{t+1})\big)\notag\\
    &\lesssim \frac{SA}{1-\gamma}\log T + \frac{(SA)^{3/4}T^{1/4}\log^{5/4}\frac{SAT}{\delta}}{(1-\gamma)^{5/4}} + \sqrt{\frac{SA\log^2 T}{1-\gamma}\sum_{t=1}^T \big(V_t(s_{t+1}) - V^{\pi_{t+1}}(s_{t+1})\big)} \notag\\
    &\lesssim \frac{SA}{1-\gamma}\log T + \frac{(SA)^{3/4}T^{1/4}\log^{5/4}\frac{SAT}{\delta}}{(1-\gamma)^{5/4}} + \frac{\log T}{1-\gamma}\sqrt{SAT} \notag\\
    &\lesssim SA\log^2\frac{SAT}{\delta} + \sqrt{\frac{SAT\log\frac{SAT}{\delta}}{(1-\gamma)^5}} + \frac{\log T}{1-\gamma}\sqrt{SAT},
\end{align}
in which the penultimate step is due to $\norm{V_t - V^{\pi_{t+1}}}_\infty \le \frac{1}{1-\gamma}$ and the last step is due to $2ab \le a^2 + b^2$.

For \eqref{eq:R1-1.3-phi}, we can directly invoke the Azuma-Hoeffding inequality (Theorem \ref{thm:hoeffding}) and obtain that with probability $1-\delta$,
\begin{align*}
    \sum_{t=1}^{T}\phi_{t} = \sum_{t=1}^{T}\big(P_{s_{t},a_{t}}-P_{t}\big)\big(V^{\star}-V^{\pi_{t}}\big) \lesssim \frac{1}{1-\gamma}\sqrt{T\log\frac{SAT}{\delta}}.
\end{align*}

Consequently, substituting \eqref{eq:R1-1.2} into \eqref{eq:R1-1.1} gives us
\begin{align*}
	&\quad \sum_{t=1}^{T}\sqrt{\frac{\Var_{s_{t},a_{t}}(V^{\star})}{N_{t}(s_{t},a_{t})}}\\
    & \lesssim \sqrt{SA}\sqrt{\frac{T}{1-\gamma} + \sqrt{\frac{SAT\log\frac{SAT}{\delta}}{(1-\gamma)^7}} + \frac{SA\log^2\frac{SAT}{\delta}}{(1-\gamma)^{9/2}} + \frac{\log T}{(1-\gamma)^2}\sqrt{SAT}}\\
	& \lesssim \sqrt{\frac{SAT}{1-\gamma}} + \frac{(SA)^{3/4}}{(1-\gamma)^{7/4}}\left(T\log\frac{SAT}{\delta}\right)^{1/4} + \frac{SA}{(1-\gamma)^{9/4}}\log\frac{SAT}{\delta} + \frac{(SA)^{3/4}}{1-\gamma}T^{1/4}\sqrt{\log T}\\
	&= \sqrt{\frac{SAT}{1-\gamma}} + \left(\frac{SAT}{1-\gamma}\right)^{1/4}\left(\frac{SA}{(1-\gamma)^3}\right)^{1/2}\log^{1/4}\frac{SAT}{\delta} + \frac{SA}{(1-\gamma)^{9/4}}\log\frac{SAT}{\delta} + \left(\frac{SAT}{1-\gamma}\right)^{1/4}\left(\frac{SA}{(1-\gamma)^{3/2}}\right)^{1/2}\sqrt{\log T}\\
	&\lesssim \sqrt{\frac{SAT}{1-\gamma}} + \frac{SA}{(1-\gamma)^{3}}\log\frac{SAT}{\delta},
\end{align*}
where we have applied the basic inequality $2ab\leq a^2 + b^2$ for any $a,b\geq 0$.

\subsubsection{Proof of Inequality \eqref{eq:R1-2}}\label{sec:proof:eq:R1-2}

Let us begin with the following manipulation
\begin{align}
	\sum_{t = 1}^T\frac{\sqrt{J_t(s_t, a_t)}}{N_t(s_t, a_t)} & =\sum_{(s, a)\in\cS\times\cA} \sum_{n = 1}^{N_T(s, a)} \frac{\sqrt{J_{t_n(s, a)}(s, a)}}{n} \nonumber \\
	&\le\sum_{(s, a)\in\cS\times\cA}  \sqrt{J_{N_T(s,a)}(s, a)}\log T \nonumber \\
	&\le\sqrt{SA \sum_{(s, a)\in\cS\times\cA} J_{N_T(s,a)}(s, a)}\log T .\label{eq:R1-2.1}
\end{align}
Here, the first inequality holds using the harmonic number and the monotonicity of $J_{t}(s_t, a_t)$ with respect to $t$ (cf. \eqref{eq:Phi-def}) due to the monotonicity of $V_{t}^{\rref}$, while the second inequality comes from Cauchy-Schwarz.

We can continue by writing out $J_{N_t(s,a)}(s, a)$ with its definition explicitly:
\begin{align}
	\sqrt{\sum_{(s,a)\in\cS\times\cA}J_{N_{T}(s,a)}(s,a)} &=\sqrt{\sum_{t=1}^{T}\Big(V^{\rref}_{t}(s_{t+1})-V^\star(s_{t+1})\Big)\ind\Big(V^{\rref}_{t}(s_{t+1})-V^\star(s_{t+1}) > 9\Big)}\nonumber\\
	& \leq\sqrt{\sum_{t=1}^{T}\Big(V_t(s_{t+1})+6-V_{t}^{\mathrm{LCB}}(s_{t+1})\Big)\ind\Big(V_t(s_{t+1})+6-V_{t}^{\mathrm{LCB}}(s_{t+1})>9\Big)}\nonumber\\
	& =\sqrt{\sum_{t=1}^{T}\Big(V_t(s_{t+1})+6-V_{t}^{\mathrm{LCB}}(s_{t+1})\Big)\ind\Big(V_t(s_{t+1})-V_{t}^{\mathrm{LCB}}(s_{t+1})>3\Big)}\nonumber\\
	& \leq\sqrt{\sum_{t=1}^{T}3\Big(V_t(s_{t+1})-V_{t}^{\mathrm{LCB}}(s_{t+1})\Big)\ind\Big(V_t(s_{t+1})-V_{t}^{\mathrm{LCB}}(s_{t+1})>3\Big)}, \label{eq:R1-2.2}
\end{align}
where the first inequality follows from Lemma \ref{lem:VR_properties} (cf. \eqref{eq:VR-V-proximity}) and the pessimism of $V^\LCB$ in Lemma \ref{lem:Q_t-lcb}, which implies
$$V^{\rref}_{t}(s_{t+1})-V^\star(s_{t+1}) \leq V_t(s_{t+1})+6-V^\star(s_{t+1})\leq V_t(s_{t+1})+6-V_{t}^{\mathrm{LCB}}(s_{t+1}),$$
and the last inequality holds since $3 < V_t(s_{t+1})-V_{t}^{\mathrm{LCB}}(s_{t+1})$ when $\ind\Big(V_t(s_{t+1})-V_{t}^{\mathrm{LCB}}(s_{t+1})>3\Big) \neq 0$.

Furthermore, invoking \eqref{eq:VR_lazy} from Lemma \ref{lem:VR_properties} in \eqref{eq:R1-2.2} gives
\begin{align*}
	&\quad\ \sqrt{\sum_{(s,a)\in\cS\times\cA}J_{N_{T}(s,a)}(s,a)}\\
    &\lesssim \sqrt{\frac{\left(SAT\right)^{1/4}}{(1-\gamma)^{13/4}}\left(\log\frac{SAT}{\delta}\right)^{5/4} + \frac{SA}{(1-\gamma)^6}\log \frac{SAT}{\delta} + \sqrt{\frac{\log^2 T}{(1-\gamma)^5}\sum_{t=1}^T \big(V_{t-1}(s_{t}) - V^{\pi_{t}}(s_{t})\big)}} \notag\\
    &\lesssim \left(SAT\right)^{1/8}\left(\log\frac{SAT}{\delta}\right)^{5/8}\left(\frac{1}{1-\gamma}\right)^{13/8} + \frac{\sqrt{SA\log \frac{SAT}{\delta}}}{(1-\gamma)^3} + \frac{T^{1/4}\sqrt{\log T}}{(1-\gamma)^{3/2}}\\
    &= T^{1/8}\left(SA\right)^{1/8}\left(\log\frac{SAT}{\delta}\right)^{5/8}\left(\frac{1}{1-\gamma}\right)^{13/8} + \frac{\sqrt{SA\log \frac{SAT}{\delta}}}{(1-\gamma)^3} + \frac{T^{1/4}}{\sqrt{\log T}}\frac{\log T}{(1-\gamma)^{3/2}}\\
    &\lesssim \frac{\sqrt{T}}{\log T} + \log T + \left(SA\right)^{1/4}\left(\log\frac{SAT}{\delta}\right)^{5/4}\left(\frac{1}{1-\gamma}\right)^{13/4} + \frac{\sqrt{SA\log \frac{SAT}{\delta}}}{(1-\gamma)^3} + \frac{\sqrt{T}}{\log T} + \frac{\log^2 T}{(1-\gamma)^{3}}\\
    &\lesssim \frac{\sqrt{T}}{\log T} + \frac{\sqrt{SA}}{(1-\gamma)^{13/4}}\log^2\frac{SAT}{\delta},
\end{align*}
where the second inequality is by $V_{t-1}(s_{t}) - V^{\pi_{t}}(s_{t}) \le \frac{1}{1-\gamma}$, and the penultimate inequality is by $2ab \le a^2 + b^2$.

Finally, substitution into \eqref{eq:R1-2.1} gives
\begin{align*}
	\sum_{t = 1}^T\frac{\sqrt{J_t(s_t, a_t)}}{N_t(s_t, a_t)} &\lesssim \left(\sqrt{SA}\log T\right)\left(\frac{\sqrt{T}}{\log T} + \frac{\sqrt{SA}}{(1-\gamma)^{13/4}}\log^2\frac{SAT}{\delta}\right)\\
    &\le \sqrt{SAT} + \frac{SA}{(1-\gamma)^{13/4}}\log^3\frac{SAT}{\delta}, 
\end{align*}
which concludes this proof.

\subsection{Bounding $\cR_2$}\label{sec:proof:R2-bound}

In this section, let us again adopt the shorthand notation $t_i := t_i(s_t, a_t)$. Also define 
\begin{equation} 
	\label{eq:lambda_t_bound}
	\lambda_t := \sum_{N = N_{t}(s_t,a_t)}^{N_{T-1}(s_t,a_t)} \eta_{N_{t}(s_t,a_t)}^N.
\end{equation} 
We can decompose $\mathcal{R}_2$ in \eqref{eq:QR-Q*-2} as follows:
\begin{equation*}
    \mathcal{R}_2:= \underbrace{\gamma\sum_{t = 1}^{T} \lambda_t \left(P_{t} - P_{s_t,a_t}\right)\left(V^* - V^{\rref}_{t}\right)}_{\eqqcolon \mathcal{R}_{2,1}} + \underbrace{\gamma\sum_{t = 1}^{T} \lambda_t\frac{\sum_{i=1}^{N_t(s_t,a_t)}\left(V^{\rref}_{t_i}(s_{t_i+1})-P_{s_t,a_t}V^{\rref}_t\right)}{N_t(s_t,a_t)}}_{\eqqcolon \mathcal{R}_{2,2}}.
\end{equation*}
Note that $\lambda_t \le 1 + \frac{1-\gamma}{2}$ because of the property $\sum_{N = n}^{\infty} \eta_n^N \le 1+ \frac{1-\gamma}{2}$ in \eqref{eq:learning-rate5} of Lemma \ref{lem:learning-rate}. In the following, we shall control these two terms separately. 

\paragraph{Step 1: bounding $\mathcal{R}_{2,1}$.} 

We begin by constructing the necessary notation and verify the conditions for Lemma \ref{lemma:freedman-application2}. Let
	\begin{align}
		W_i \coloneqq V^{\rref}_{i} - V^{\star}
	\qquad \text{and} \qquad
	u_i(s_i,a_i) \coloneqq \lambda_i, \label{R21-Wu}
	\end{align}
	where $\lambda_i$ is as defined in \eqref{eq:lambda_t_bound}. Clearly, $W_i$ and $u_i(s_i,a_i)$ are deterministic at the end of the $(i-1)$-th step, given $N(s_i,a_i)$. Moreover, define
	\begin{align}
		\left|u_i(s_i,a_i) \right| \leq 1+\frac{1-\gamma}{2} \eqqcolon C_{\mathrm{u}}
	\qquad \text{and} \qquad
	   \norm{W_i}_\infty \leq \frac{1}{1-\gamma} \eqqcolon C_{\mathrm{w}}. \label{R21-cw-cu}
	\end{align}
 
	Hence, we can apply Lemma \ref{lemma:freedman-application2} with \eqref{R21-Wu} and \eqref{R21-cw-cu} and conclude that with probability at least $1- \delta/2$, 
\begin{align}
 \cR_{2,1} &\le \gamma\left|\sum_{t = 1}^{T} \lambda_t \left(P_{t} - P_{s_t,a_t}\right)\left(V^* - V^{\rref}_{t}\right)\right|\nonumber\nonumber \\
 & \lesssim \gamma\sqrt{C_{\mathrm{u}}^{2}C_{\mathrm{w}}SA\sum_{i=1}^{T}\mathbb{E}_{i-1}\left[P_{i}W_{i}\right]\log\frac{T}{\delta}}+C_{\mathrm{u}}C_{\mathrm{w}}SA\log\frac{T}{\delta}\nonumber\nonumber \\
 & \lesssim \gamma\sqrt{\frac{SA}{1-\gamma}\sum_{i=1}^{T}\mathbb{E}_{i-1}\left[P_{i}\big(V^{\rref}_i-V^{\star}\big)\right]\log\frac{T}{\delta}}+\frac{SA}{1-\gamma}\log\frac{T}{\delta}\nonumber\nonumber \\
 &= \gamma\sqrt{\frac{SA}{1-\gamma} \bigg\{ \sum_{t=1}^{T}P_{s_{t},a_{t}}\big(V^{\rref}_t-V^{\star}\big) \bigg\} \log\frac{T}{\delta}}+\frac{SA}{1-\gamma}\log\frac{T}{\delta}.
	\label{eq:R21-1}
\end{align}

To finish the upper bound on $\cR_{2,1}$, it remains to control 
\begin{equation}\label{eq:R21-1.5}
    \sum_{t=1}^{T}P_{s_{t},a_{t}}\left(V^{\rref}_t-V^{\star}\right) = \sum_{t=1}^{T}P_{t}\left(V^{\rref}_t-V^{\star}\right) + \sum_{t=1}^{T}\left(P_{s_{t},a_{t}} - P_t\right)\left(V^{\rref}_t-V^{\star}\right),
\end{equation}
which can be further decomposed into two terms. Let us first examine the first term: 
\begin{align}
	\sum_{t=1}^{T}P_{t}\big(V^{\rref}_t-V^{\star}\big) &\overset{\mathrm{(i)}}{\leq} \sum_{t=1}^{T}P_{t}\big(V_t + 9 -V^{\star}\big) \nonumber \\
	&\leq 9T + \sum_{t=1}^{T}\big(V_t(s_{t+1}) -V^{\star}(s_{t+1})\big) \notag\\
    &\overset{\mathrm{(ii)}}{\lesssim} T + \frac{SA\log^2\frac{SAT}{\delta}}{(1-\gamma)^{7/2}} + \sqrt{\frac{SAT\log\frac{SAT}{\delta}}{(1-\gamma)^5}} \label{eq:R21-2}
 \end{align}
with probability at least $1-{\delta}/{4}$. Above, (i) holds according to \eqref{eq:VR-V-proximity}, and (ii) is valid since
 \begin{align*}
	 \sum_{t=1}^{T}\big(V_t(s_{t+1}) -V^{\star}(s_{t+1})\big) &\leq \sum_{t=1}^{T}\big(V_t(s_{t+1})-V^{\pi_{t+1}}(s_{t+1})\big) \lesssim \frac{SA\log^2\frac{SAT}{\delta}}{(1-\gamma)^{7/2}} + \sqrt{\frac{SAT\log\frac{SAT}{\delta}}{(1-\gamma)^5}},
 \end{align*}
where the first inequality follows since $V^{\star}\geq V^{\pi_{t+1}}$, and the second inequality comes from \eqref{eq:delta-t1-bound} in Lemma \ref{lem:crude}. 

Additionally, for the second term in \eqref{eq:R21-1.5}, the Azuma-Hoeffding inequality (Theorem \ref{thm:hoeffding}) directly leads to
\begin{align*}
\left|\sum_{t=1}^{T}\left(P_{t}-P_{s_{t},a_{t}}\right)\left(V^{\rref}_t-V^{\star}\right)\right| & \lesssim \sqrt{\frac{T}{(1-\gamma)^2}\log \frac{1}{\delta}}
\end{align*}
with probability at least $1-\delta/4$. 

Together, we have that
\begin{align}
\sum_{t=1}^{T}P_{s_{t},a_{t}}\left(V^{\rref}_t-V^{\star}\right) & \leq \sum_{t=1}^{T}P_{t}\left(V^{\rref}_t-V^{\star}\right)+\left|\sum_{t=1}^{T}\left(P_{t}-P_{s_{t},a_{t}}\right)\left(V^{\rref}_t-V^{\star}\right)\right| \notag\\
 & \lesssim T + \frac{SA\log^2\frac{SAT}{\delta}}{(1-\gamma)^{7/2}} + \sqrt{\frac{SAT\log\frac{SAT}{\delta}}{(1-\gamma)^5}} \label{eq:R21-3}
\end{align} 
with probability at least $1-\delta/2$. 

Finally, substituting this inequality into \eqref{eq:R21-1} allows us to conclude with
\begin{align}
&\quad\ \gamma\left|\sum_{t = 1}^{T} \lambda_t \left(P_{t} - P_{s_t,a_t}\right)\left(V^* - V^{\rref}_{t}\right)\right| \nonumber \\
&  \lesssim \gamma\sqrt{\frac{SA}{1-\gamma} \bigg\{ \sum_{t=1}^{T}P_{s_{t},a_{t}}\big(V^{\rref}_t-V^{\star}\big) \bigg\} \log\frac{T}{\delta}}+\frac{SA}{1-\gamma}\log\frac{T}{\delta}\nonumber \\
 & \lesssim \gamma\sqrt{\frac{SA}{1-\gamma} \left(T + \frac{SA\log^2\frac{SAT}{\delta}}{(1-\gamma)^{7/2}} + \sqrt{\frac{SAT\log\frac{SAT}{\delta}}{(1-\gamma)^5}}\right) \log\frac{T}{\delta}}+\frac{SA}{1-\gamma}\log\frac{T}{\delta}\nonumber\\
 &\lesssim \sqrt{\frac{SA}{1-\gamma}\left(T + \frac{SA\log^2\frac{SAT}{\delta}}{(1-\gamma)^{7/2}} + \frac{SA}{(1-\gamma)^5}\log\frac{SAT}{\delta}\right)\log\frac{T}{\delta}}+\frac{SA}{1-\gamma}\log\frac{T}{\delta}\nonumber\\
 & \lesssim\sqrt{\frac{SAT}{1-\gamma}\log\frac{SAT}{\delta}} + \frac{SA}{(1-\gamma)^3}\log^{3/2}\frac{SAT}{\delta} \label{eq:R21-final}
\end{align}
with probability at least $1-\delta$. Above, the penultimate line holds due to $2ab \le a^2 + b^2$.

\paragraph{Step 2: bounding $\mathcal{R}_{2,2}$.} 

We can start by decomposing $\mathcal{R}_{2,2}$ into three terms as follows: 
\begin{align}
\mathcal{R}_{2,2} &\leq \gamma\sum_{t = 1}^{T} \frac{\lambda_t}{N_t(s_t,a_t)}\sum_{i=1}^{N_t(s_t,a_t)}\left(V^{\rref}_{t_i}(s_{t_i+1})-P_{s_t,a_t}V^{\rref}_T\right)\nonumber\\
& = \gamma\sum_{t = 1}^T \sum_{n = N_t(s_t, a_t)}^{N_{T}(s_t, a_t)} \frac{\lambda_t}{n} \left(V^{\rref}_{t}(s_{t+1}) - V^{\rref}_{T}(s_{t+1}) + \left(P_{t} - P_{s_t, a_t} \right)V^{\rref}_{T} \right) \nonumber \\
& \leq  \left(1+\frac{1-\gamma}{2}\right)\log T \sum_{t = 1}^T \big(V^{\rref}_{t}(s_{t+1}) - V^{\rref}_{T}(s_{t+1})\big) + \sum_{t = 1}^T \sum_{n = N_t(s_t, a_t)}^{N_{T}(s_t, a_t)} \frac{\lambda_t}{n}\big(P_{t} - P_{s_t, a_t} \big)V^{\star} \nonumber\\
&\qquad + \sum_{t = 1}^T \sum_{n = N_t(s_t, a_t)}^{N_{T}(s_t, a_t)} \frac{\lambda_t}{n}\big(P_{t} - P_{s_t, a_t} \big)\big( V^{\rref}_{T} -V^\star \big), \label{eq:R22-1}
\end{align}
in which the first inequality is due to the monotonicity $V^{\rref}_{t} \ge V^{\rref}_{t+1} \ge \cdots \geq V^{\rref}_{T}$, and the last inequality is due to the facts that $\sum_{n = N_t(s_t, a_t)}^{N_{T}(s_t, a_t)} \frac{1}{n} \leq \log T$ (harmonic number) and $\lambda_t \leq 1+\frac{1-\gamma}{2}$ (cf. \eqref{eq:lambda_t_bound}). To proceed, we shall control the three terms in \eqref{eq:R22-1} separately.

The first term in \eqref{eq:R22-1} can be controlled with Lemma \ref{lem:VR_properties} (cf. \eqref{eq:VR_lazy}) as follows: 
\begin{align}
    &\quad\ \sum_{t = 1}^T \big(V^{\rref}_{t}(s_{t+1}) - V^{\rref}_{T}(s_{t+1})\big) \notag\\
	&\lesssim \frac{\left(SA\right)^{3/4}T^{1/4}}{(1-\gamma)^{13/4}}\left(\log\frac{SAT}{\delta}\right)^{5/4} + \frac{SA}{(1-\gamma)^6}\log \frac{SAT}{\delta} + \sqrt{\frac{SA\log^2 T}{(1-\gamma)^5}\sum_{t=1}^T \big(V_{t-1}(s_{t}) - V^{\pi_{t}}(s_{t})\big)} \notag\\
    &\lesssim \sqrt{\frac{SAT}{1-\gamma}\log^{2}\frac{SAT}{\delta}} + \frac{SA}{(1-\gamma)^6}\log^{3/2} \frac{SAT}{\delta} + \sqrt{\frac{SA\log^2 T}{(1-\gamma)^5}\sum_{t=1}^T \big(V_{t-1}(s_{t}) - V^{\pi_{t}}(s_{t})\big)} \label{eq:R22-1-1}
\end{align}
with probability at least $1- {\delta}/{3}$. Above, the last inequality is by $2ab \le a^2 + b^2$.

For the second term in \eqref{eq:R22-1}, let us first construct the necessary notation and verify the conditions for Lemma \ref{lemma:freedman-application2}. We can set
\begin{align}
	W_i \coloneqq V^{\star},
	\qquad \text{and} \qquad
	u_i(s_i,a_i) \coloneqq \sum_{n = N_i(s_i, a_i)}^{N_T(s_i,a_i)} \frac{\lambda_{i}}{n}. \label{R22-Wu} 
\end{align}
Clearly, $W_i$ and $u_i(s_i,a_i)$ are deterministic at the end of the $(i-1)$-th step, given $N_T(s_i,a_i)$. Furthermore, let
\begin{align}
	\left| u_i(s_i,a_i) \right| 
	\leq \left(1+\frac{1-\gamma}{2}\right)\sum_{n = N_i(s_i, a_i)}^{N_T(s_i,a_i)} \frac{1}{n}\leq \left(1+\frac{1-\gamma}{2}\right)\log T \eqqcolon C_{\mathrm{u}} \label{R22-cu} 
\end{align}
due to the facts that $\sum_{n = N_t(s_t, a_t)}^{N_{T}(s_t, a_t)} \frac{1}{n} \leq \log T$ (harmonic number) and $\lambda_t \leq 1+\frac{1-\gamma}{2}$ (cf. \eqref{eq:lambda_t_bound}), and
\begin{align}
	  \norm{W_i}_\infty \leq \frac{1}{1-\gamma} \eqqcolon C_{\mathrm{w}}. \label{R22-cw} 
\end{align}
Hence, we can apply Lemma \ref{lemma:freedman-application2} with \eqref{R22-Wu}, \eqref{R22-cu} and \eqref{R22-cw} and conclude that with probability at least $1-{\delta}/{3}$,
\begin{align}
&\quad\ \left|\sum_{t = 1}^T \sum_{n = N_t(s_t, a_t)}^{N_T(s_t, a_t)} \frac{\lambda_t}{n}\left(P_{t} - P_{s_t, a_t} \right)V^{\star}\right| = \left|\sum_{t = 1}^T X_{t} \right| \nonumber\\
& \lesssim \sqrt{ C_{\mathrm{u}}^{2} SA \sum_{i=1}^{T}\mathbb{E}_{i-1}\left[\left|(P_{i}-P_{s_{i},a_{i}})W_{i}\right|^{2}\right] \log\frac{T}{\delta} }  + C_{\mathrm{u}} C_{\mathrm{w}} SA  \log\frac{T}{\delta}  \notag\\
& \overset{\mathrm{(i)}}{\asymp} \sqrt{\sum_{t=1}^{T} \Var_{s_t, a_t}(V^{\star}) \cdot SA\log^3\frac{T}{\delta} }  + \frac{SA}{1-\gamma} \log^2\frac{T}{\delta}  \notag\\
& \overset{\mathrm{(ii)}}{\lesssim} \sqrt{SA\left(\frac{T}{1-\gamma} + \sqrt{\frac{SAT\log\frac{SAT}{\delta}}{(1-\gamma)^7}} + \frac{SA\log^2\frac{SAT}{\delta}}{(1-\gamma)^{9/2}}\right)\log^3\frac{SAT}{\delta}} + \frac{SA}{1-\gamma}\log^2\frac{T}{\delta} \notag\\ 
&  \lesssim \sqrt{\frac{SAT}{1-\gamma}\log^{3}\frac{SAT}{\delta} + \sqrt{\frac{SAT}{(1-\gamma)^7}}\log^{7/2}\frac{SAT}{\delta}} + \frac{SA}{(1-\gamma)^{9/4}}\log^{5/2}\frac{SAT}{\delta} \notag\\
& \overset{\mathrm{(iii)}}{\lesssim} \sqrt{\frac{SAT}{1-\gamma}\log^{3}\frac{SAT}{\delta} + \frac{\log^{4}\frac{SAT}{\delta}}{(1-\gamma)^6}} + \frac{SA}{(1-\gamma)^{9/4}}\log^{5/2}\frac{SAT}{\delta} \notag\\
&\lesssim \sqrt{\frac{SAT}{1-\gamma}\log^3\frac{SAT}{\delta}} + \frac{SA}{(1-\gamma)^{3}}\log^{2}\frac{SAT}{\delta}, \label{eq:R22-2-1}
\end{align}
in which (i) comes from the definition of variance in \eqref{eq:var-def}; (ii) holds due to \eqref{eq:R1-1.2}; (iii) is valid since
\begin{equation*}
    \sqrt{\frac{SAT}{(1-\gamma)^7}}\log^{7/2}\frac{SAT}{\delta} = \sqrt{\frac{SAT}{1-\gamma}}\log^{3/2}\frac{SAT}{\delta}\cdot\frac{\log^{2}\frac{SAT}{\delta}}{(1-\gamma)^3} \lesssim \frac{SAT}{1-\gamma}\log^{3}\frac{SAT}{\delta} + \frac{\log^{4}\frac{SAT}{\delta}}{(1-\gamma)^6}
\end{equation*}
due to $2ab \le a^2 + b^2$ (by Cauchy-Schwarz).

Finally, for the third term in \eqref{eq:R22-1}, in order to get a tight bound in the presence of the dependency between $P_t$ and $V_{T}^{\mathrm{R}}$, we use the standard epsilon-net argument (see, e.g., \citet{taotopics}), whose details are deferred to Appendix \ref{appendix:covering}. The final bound on this term is
 \begin{align}
 & \left|\sum_{t = 1}^T \sum_{n = N_t(s_t, a_t)}^{N_{T}(s_t, a_t)} \frac{\lambda_t}{n}\big(P_{t} - P_{s_t, a_t} \big)\big( V^{\rref}_{T} -V^\star \big)\right| \lesssim \frac{SA}{(1-\gamma)^{3}} \log^{5/2}\frac{SAT}{\delta} + \sqrt{\frac{SAT}{1-\gamma}\log^{3}\frac{SAT}{\delta}}.
	 \label{eq:R22-3}
\end{align}

Bringing \eqref{eq:R22-1-1}, \eqref{eq:R22-2-1} and \eqref{eq:R22-3} into \eqref{eq:R22-1}, with a union bound argument, we can demonstrate
\begin{align}
	\mathcal{R}_{2,2} &\leq  C_{2,2} \Bigg\{ \sqrt{\frac{SAT}{1-\gamma}\log^{3}\frac{SAT}{\delta}} + \frac{SA}{(1-\gamma)^6}\log^{5/2} \frac{SAT}{\delta} + \sqrt{\frac{\log^2 T}{(1-\gamma)^5}\sum_{t=1}^T \big(V_{t-1}(s_{t}) - V^{\pi_{t}}(s_{t})\big)}\Bigg\}
	\label{eq:R22-final}
\end{align}
with probability at least $1-\delta$, where $C_{3,2}>0$ is some constant.

\paragraph{Step 3: combining the preceding bounds.} 

Taking \eqref{eq:R21-final} and \eqref{eq:R22-final} together, we immediately arrive at
\begin{align}
	\mathcal{R}_2 \leq \left| \mathcal{R}_{2,1} \right|  +  \mathcal{R}_{2,2} &\leq  C_{\mathrm{r},2} \left\{ \sqrt{\frac{SAT}{1-\gamma}\log^{3}\frac{SAT}{\delta}} + \frac{SA}{(1-\gamma)^7}\log^{5/2} \frac{SAT}{\delta} + \sqrt{\frac{SA\log^2 T}{(1-\gamma)^5}\sum_{t=1}^T \big(V_{t-1}(s_{t}) - V^{\pi_{t}}(s_{t})\big)}\right\}
\end{align}
with probability at least $1-2\delta$, where $C_{\mathrm{r},2}>0$ is some constant.  
This immediately concludes the proof.

\subsubsection{Proof of \eqref{eq:R22-3}: An Epsilon-Net Argument}\label{appendix:covering}

\paragraph{Step 1: concentration on a fixed vector.} 

Let us establish a concentration bound over a fixed vector, the choice of which we will later specify when we introduce our epsilon net. Consider a fixed vector $V^{\mathrm{d}} \in \mathbb{R}^{S}$ satisfying: 
\begin{align}
	V^{\star} \leq V^{\mathrm{d}} &\leq \frac{1}{1-\gamma}.
	\label{eq:Vd-assumption}
\end{align}
We intend to derive a bound on the following quantity as in \eqref{eq:R22-3}:
\begin{equation*}
    \sum_{t = 1}^T \sum_{n = N_t(s_t, a_t)}^{N_{T}(s_t, a_t)} \frac{\lambda_t}{n}\big(P_{t} - P_{s_t, a_t} \big)\big( V^{\mathrm{d}} -V^\star \big).
\end{equation*}
Let us first construct the necessary notation and verify the conditions for Lemma \ref{lemma:freedman-application2}. We can set
\begin{align}
	W_i \coloneqq V^{\mathrm{d}} - V^{\star}
	\qquad \text{and} \qquad
	u_i(s_i,a_i) \coloneqq \sum_{n = N_i(s_i, a_i)}^{N_T(s_i,a_i)} \frac{\lambda^{i}_h}{n}. \label{Vd-Wu} 
\end{align}
Clearly, $W_i$ and $u_i(s_i,a_i)$ are deterministic at the end of the $(i-1)$-th step, given $N_T(s_i,a_i)$. Furthermore, let
\begin{align}
	\left| u_i(s_i,a_i) \right| 
	\leq \left(1 + \frac{1-\gamma}{2}\right)\sum_{n = N_i(s_i,a_i)}^{N_T(s_i,a_i)} \frac{1}{n} &\leq \left(1 + \frac{1-\gamma}{2}\right)\log T \eqqcolon C_{\mathrm{u}} \label{Vd-cu} 
\end{align}
due to the facts that $\sum_{n = N_t(s_t, a_t)}^{N_{T}(s_t, a_t)} \frac{1}{n} \leq \log T$ (harmonic number) and $\lambda_t \leq 1+\frac{1-\gamma}{2}$ (cf. \eqref{eq:lambda_t_bound}), and
\begin{align}
	  \norm{W_i}_\infty \leq \frac{1}{1-\gamma} \eqqcolon C_{\mathrm{w}}. \label{Vd-cw} 
\end{align}
Hence, we can apply Lemma \ref{lemma:freedman-application2} with \eqref{Vd-Wu}, \eqref{Vd-cu} and \eqref{Vd-cw} and conclude that with probability at least $1-\delta_0$,
\begin{align}
 &\quad\ \left|\sum_{t = 1}^T \sum_{n = N_t(s_t, a_t)}^{N_{T}(s_t, a_t)} \frac{\lambda_t}{n}\left(P_{t} - P_{s_t, a_t} \right)\left( V^{\mathrm{d}} -V^\star \right)\right|=\left|\sum_{t=1}^{T}X_{t}\right|\nonumber\\
 & \lesssim\sqrt{C_{\mathrm{u}}^{2}C_{\mathrm{w}}\left\{\sum_{i=1}^{T}\mathbb{E}_{i-1}\left[P_{i}W_{i}\right]\right\}\log\frac{T^{SA}}{\delta_{0}}}+C_{\mathrm{u}}C_{\mathrm{w}}\log\frac{T^{SA}}{\delta_{0}}\nonumber\\
 & \lesssim \sqrt{\frac{\log^2 T}{1-\gamma}\left\{\sum_{t=1}^{T}P_{s_{t},a_{t}}\left(V^{\mathrm{d}}-V^{\star}\right)\right\}\log\frac{T^{SA}}{\delta_{0}}} + \frac{\log T}{1-\gamma} \log\frac{T^{SA}}{\delta_{0}}.\label{eq:Vd-concentration} 
\end{align}
The choice of $\delta_0$ will be specified momentarily. 

\paragraph{Step 2: constructing an epsilon net.} 

In this step, we construct a properly-sized epsilon net that can cover the space of $V^\rref_T$. Specifically, we will construct an epsilon net $\cN_{\alpha}$ (the parameter $\alpha$ will be specified shortly) such that:
\begin{itemize}
	\item For any $V \in [0,\frac{1}{1-\gamma}]^S$, one can find a point $V^{\mathrm{net}} \in \cN_{\alpha}$ such that $$0\leq V (s) - V^{\mathrm{net}} (s) \leq \alpha \qquad \text{for all }s\in \cS;$$
 
    \item The cardinality of $\cN_{\alpha}$ satisfies the following bound:
    \begin{equation}
    	\left| \mathcal{N}_{\alpha} \right| \leq \left( \frac{1}{(1-\gamma)\alpha} \right)^S.  
    	\label{eq:N-cardinality}
    \end{equation}
\end{itemize}

Clearly, such epsilon net exists by letting it be all intersection points in a grid on $[0,\frac{1}{1-\gamma}]^S$ with interval length $\alpha$. Now, we extend the concentration bound \eqref{eq:Vd-concentration} in Step 1 to all members of this epsilon net $\cN_\alpha$. Set $\delta_0 = \frac{\delta}{6} / \left( \frac{1}{(1-\gamma)\alpha} \right)^S$. Taking a union bound over all vectors in $\cN_\alpha$, \eqref{eq:Vd-concentration} gives that with probability at least $1- \delta_0 \left( \frac{1}{(1-\gamma)\alpha} \right)^S = 1-\delta/6$, 
\begin{align}
 &\quad\ \left|\sum_{t = 1}^T \sum_{n = N_t(s_t, a_t)}^{N_{T}(s_t, a_t)} \frac{\lambda_t}{n}\left(P_{t} - P_{s_t, a_t} \right)\left(V^{\mathrm{net}}-V^{\star}\right)\right|\nonumber\\
 & \lesssim \sqrt{\frac{\log^2 T}{1-\gamma}\sum_{t=1}^{T}P_{s_{t},a_{t}}\left(V^{\mathrm{net}}-V^{\star}\right)\log\frac{T^{SA}}{\delta_{0}}}+ \frac{\log T}{1-\gamma} \log\frac{T^{SA}}{\delta_{0}} \notag\\
 & \le \sqrt{\frac{SA\log^2 T}{1-\gamma}\sum_{t=1}^{T}P_{s_{t},a_{t}}\left(V^{\mathrm{net}}-V^{\star}\right)\log\frac{T}{(1-\gamma)^S\alpha^S\delta}}+ \frac{SA\log T}{1-\gamma} \log\frac{T}{(1-\gamma)^S\alpha^S\delta}
	\label{eq:N-concentration}
\end{align}
simultaneously for all elements $V^{\mathrm{net}} \in \mathcal{N}_{\alpha} $.

\paragraph{Step 3: obtaining a uniform bound.}
Now, we are ready to establish a uniform bound over the entire space of $V^\rref_T$, which is just all optimistic vectors in $[0,\frac{1}{1-\gamma}]^S$. Consider an arbitrary vector
$V^{\mathrm{u}}\in \RR^S$ satisfying $V^{\star} \leq V^{\mathrm{u}} \leq \frac{1}{1-\gamma}$. By the construction of $\cN_\alpha$, one can find a point $V^{\mathrm{net}} \in \mathcal{N}_{\alpha}$
such that
\begin{align}
	0\leq  V^{\mathrm{u}}(s) - V^{\mathrm{net}}(s) \leq \alpha \qquad \text{for all }s\in \cS. 
	\label{eq:Vd-Vnet-proximity}
\end{align}

Setting $\alpha = \frac{1}{(1-\gamma)T^{1/S}}$, we can deduce that
\begin{align}
 &\quad\ \left|\sum_{t = 1}^T \sum_{n = N_t(s_t, a_t)}^{N_{T}(s_t, a_t)} \frac{\lambda_t}{n}\left(P_{t} - P_{s_t, a_t} \right)\left(V^{\mathrm{u}}-V^{\star}\right)\right|\nonumber\\
 &\leq \left|\sum_{t = 1}^T \sum_{n = N_t(s_t, a_t)}^{N_{T}(s_t, a_t)} \frac{\lambda_t}{n}\left(P_{t} - P_{s_t, a_t} \right)\left(V^{\mathrm{net}}-V^{\star}\right)\right|+\left|\sum_{t = 1}^T \sum_{n = N_t(s_t, a_t)}^{N_{T}(s_t, a_t)} \frac{\lambda_t}{n}\left(P_{t} - P_{s_t, a_t} \right)\left(V^{\mathrm{u}} - V^{\mathrm{net}}\right)\right|\nonumber\\
 &\overset{(\mathrm{i})}{\lesssim} \left|\sum_{t = 1}^T \sum_{n = N_t(s_t, a_t)}^{N_{T}(s_t, a_t)} \frac{\lambda_t}{n}\left(P_{t} - P_{s_t, a_t} \right)\left(V^{\mathrm{net}}-V^{\star}\right)\right|+T\alpha\log T\nonumber\\
 & \lesssim \sqrt{\frac{SA\log^2 T}{1-\gamma}\sum_{t=1}^{T}P_{s_{i},a_{i}}\left(V^{\mathrm{net}}-V^{\star}\right)\log\frac{T}{(1-\gamma)^S\alpha^S\delta}}+ \frac{SA\log T}{1-\gamma} \log\frac{T}{(1-\gamma)^S\alpha^S\delta}+T\alpha\log T \notag\\
 & \overset{(\mathrm{ii})}{\asymp} \sqrt{\frac{SA\log^2 T}{1-\gamma}\sum_{t=1}^{T}P_{s_{i},a_{i}}\left(V^{\mathrm{net}}-V^{\star}\right)\log\frac{T}{\delta}}+ \frac{SA\log T}{1-\gamma} \log\frac{T}{\delta}, 
	\label{eq:uniform-bound-all-Vd}
\end{align}
in which (i) is due to 
\begin{align}
 &\quad\ \left|\sum_{t = 1}^T \sum_{n = N_t(s_t, a_t)}^{N_{T}(s_t, a_t)} \frac{\lambda_t}{n}\left(P_{t} - P_{s_t, a_t} \right)\left(V^{\mathrm{u}} - V^{\mathrm{net}}\right)\right|\nonumber\\
 &\leq \left|\sum_{t = 1}^T \sum_{n = N_t(s_t, a_t)}^{N_{T}(s_t, a_t)} \frac{\lambda_t}{n}\left(\norm{P_{t}}_1 + \norm{P_{s_t, a_t}}_1 \right)\norm{V^{\mathrm{u}} - V^{\mathrm{net}}}_\infty\right|\nonumber\\
 & \leq 2\left(1 + \frac{1-\gamma}{2}\right)T \alpha\log T 
\end{align}
because $\sum_{n = N_t(s_t, a_t)}^{N_{T}(s_t, a_t)} \frac{1}{n} \leq \log T$ (harmonic number) and $\lambda_t \leq 1+\frac{1-\gamma}{2}$ (cf. \eqref{eq:lambda_t_bound}). 

(ii) holds due to the condition \eqref{eq:Vd-Vnet-proximity} and our choice of $\alpha$.  

In summary with probability at least $1-\delta/6$, \eqref{eq:uniform-bound-all-Vd} holds simultaneously for all $V^{\mathrm{u}} \in \mathbb{R}^{S}$ obeying $V^{\star} \leq V^{\mathrm{u}} \leq \frac{1}{1-\gamma}$. 

\paragraph{Step 4: controlling the original quantity of interest.} 
 With the uniform bound over the set of all optimistic vectors in $\mathbb{R}^{S}$ from Step 3, we can finally control the original quantity of interest
 \begin{align}
 \sum_{t = 1}^T \sum_{n = N_t(s_t, a_t)}^{N_{T}(s_t, a_t)} \frac{\lambda_t}{n}\big(P_{t} - P_{s_t, a_t} \big)\big(V^{\rref}_{T} -V^\star \big).
 \end{align}
Substitution into \eqref{eq:uniform-bound-all-Vd} with $V^{\rref}_{T}$ as $V^{\mathrm{u}}$ yields
 \begin{align}
 &\quad\ \left|\sum_{t = 1}^T \sum_{n = N_t(s_t, a_t)}^{N_{T}(s_t, a_t)} \frac{\lambda_t}{n}\big(P_{t} - P_{s_t, a_t} \big)\big(V^{\rref}_{T} -V^\star \big)\right|\nonumber\\
 &\lesssim \sqrt{\frac{SA\log^2 T}{1-\gamma}\sum_{t=1}^{T}P_{s_{i},a_{i}}\left(V^{\rref}_{T}-V^{\star}\right)\log\frac{T}{\delta}}+ \frac{SA}{1-\gamma} \log^2\frac{T}{\delta} \notag\\
 &\overset{(\mathrm{i})}{\lesssim} \sqrt{\frac{SA\log^2 T}{1-\gamma}\left\{ T + \frac{SA\log^2\frac{SAT}{\delta}}{(1-\gamma)^{7/2}} + \sqrt{\frac{SAT\log\frac{SAT}{\delta}}{(1-\gamma)^5}}\right\}\log\frac{T}{\delta}}+ \frac{SA}{1-\gamma} \log^2\frac{T}{\delta} \notag\\
 &\lesssim \sqrt{\frac{SA\log^2 T}{1-\gamma}\left\{ \sqrt{\frac{SAT}{(1-\gamma)^5}\log\frac{SAT}{\delta}} + T\right\}\log\frac{T}{\delta}}+ \frac{SA}{(1-\gamma)^{9/4}} \log^{5/2}\frac{T}{\delta} \notag\\
 &\overset{(\mathrm{ii})}{\lesssim} \sqrt{\frac{SA}{1-\gamma}\left\{ \frac{SA}{(1-\gamma)^5}\log\frac{SAT}{\delta}+T\right\} \log^{3}\frac{SAT}{\delta}} + \frac{SA}{(1-\gamma)^{9/4}} \log^{5/2}\frac{T}{\delta} \notag\\
 &\lesssim \frac{SA}{(1-\gamma)^{3}} \log^{5/2}\frac{SAT}{\delta} + \sqrt{\frac{SAT}{1-\gamma}\log^{3}\frac{SAT}{\delta}},
	 \label{eq:covering-final}
\end{align}
in which (i) is obtained because with probability exceeding $1-\delta/6$, 
\begin{align}
  \sum_{t = 1}^T P_{s_t, a_t}\big(V^{\rref}_{T} -V^\star \big)
	&\le\sum_{t = 1}^T P_{s_t, a_t}\big(V^{\rref}_{t} -V^\star \big) \le T + \frac{SA\log^2\frac{SAT}{\delta}}{(1-\gamma)^{7/2}} + \sqrt{\frac{SAT\log\frac{SAT}{\delta}}{(1-\gamma)^5}},
\end{align}
where the first inequality holds because $V_{t}$ is non-increasing in $t$ and $V^{\rref}_t$ can only be updated from $V_t$ (thus $V^{\rref}_{t}$ is monotonically non-increasing), and the second inequality follows from \eqref{eq:R21-3}, which we are allowed to use because Lemma \ref{lem:Q_t-lower-bound} suggests 
 \begin{align}
 	V^{\star}(s) \leq V^{\rref}_{T}(s) &\leq \frac{1}{1-\gamma} \quad \text{for any } s\in\cS.
 \end{align}

(ii) holds since 
\begin{equation*}
    \sqrt{\frac{SAT}{(1-\gamma)^5}\log\frac{SAT}{\delta}} = \sqrt{\frac{SA}{(1-\gamma)^5}\log\frac{SAT}{\delta}}\sqrt{T}\lesssim \frac{SA}{(1-\gamma)^5}\log\frac{SAT}{\delta}+T.
\end{equation*}

\end{document}